\documentclass[a4paper,11pt]{article}

\usepackage{amsmath,amsfonts,bm}

\def\eqref#1{equation~\ref{#1}}

\def\floor#1{\lfloor #1 \rfloor}
\def\1{\bm{1}}

\DeclareMathAlphabet{\mathsfit}{\encodingdefault}{\sfdefault}{m}{sl}
\SetMathAlphabet{\mathsfit}{bold}{\encodingdefault}{\sfdefault}{bx}{n}

\newcommand{\E}{\mathbb{E}}

\newcommand{\R}{\mathbb{R}}

\newcommand{\Cov}{\mathrm{Cov}}

\DeclareMathOperator{\sign}{sign}

\usepackage{geometry}
\usepackage[utf8]{inputenc} 
\usepackage[T1]{fontenc}    
\usepackage{url}            
\usepackage{amsfonts}       
\usepackage{nicefrac}       
\usepackage{microtype}      

\usepackage[parfill]{parskip}

\usepackage{graphicx}
\usepackage{subfig}
\usepackage{booktabs} 
\usepackage{tcolorbox}
\usepackage{wrapfig}

\usepackage{amsmath}
\usepackage{amssymb}
\usepackage{mathtools}
\usepackage{amsthm}
\usepackage{rotating}
\usepackage{comment}
\usepackage{enumerate}
\usepackage{enumitem}
\usepackage{threeparttable}
\usepackage{makecell}
\usepackage{pifont}
\usepackage{multirow}

\usepackage[dvipsnames]{xcolor}

\definecolor{emphcoral}{HTML}{FF6F61}
\definecolor{emphverm}{HTML}{E34234}
\definecolor{emphrasp}{HTML}{D63A5C}
\definecolor{burgundy}{HTML}{B38A5F}

\definecolor{emphgreen}{HTML}{00AA00}
\definecolor{emphcyan}{HTML}{00CCCC}
\definecolor{emphorange}{HTML}{FF6600}

\definecolor{emphblue}{HTML}{1F4E79}
\definecolor{emphviolet}{HTML}{6F42C1}
\definecolor{darkgreen}{rgb}{0.0,0.5,0.0}

\usepackage[pagebackref]{hyperref}
\hypersetup{
	colorlinks=true,        
	linkcolor=blue,         
	filecolor=magenta,
	citecolor=darkgreen,         
	urlcolor=blue           
}
\renewcommand*{\backrefalt}[4]{%
    \ifcase #1 \footnotesize{(Not cited.)}%
    \or        \footnotesize{(Cited on page~#2)}%
    \else      \footnotesize{(Cited on pages~#2)}%
    \fi}

\usepackage[capitalize,noabbrev]{cleveref}

\usepackage{natbib}

\theoremstyle{plain}
\newtheorem{theorem}{Theorem}[section]

\newtheorem{lemma}[theorem]{Lemma}
\newtheorem{corollary}[theorem]{Corollary}
\theoremstyle{definition}
\newtheorem{definition}[theorem]{Definition}
\theoremstyle{remark}
\newtheorem{remark}[theorem]{Remark}
\usepackage{enumitem}

\newcommand{\norm}[1]{{}\left\| #1 \right\|}
 
\newcommand{\tr}{\text{Tr}}

\DeclareMathOperator{\diag}{diag}

\newcommand{\eqdef}{\coloneqq}

\newtcolorbox{mybox}[2][]{
  colframe = white, 
  colback  = gray!7,
  #1
}

\newcommand{\cmark}{\textcolor{green!50!black}{\ding{51}}} 
\newcommand{\xmark}{\textcolor{red!70!black}{\ding{55}}}   

\def\toptitlebar{\hrule height1pt \vskip .25in} 
\def\bottomtitlebar{\vskip .22in \hrule height1pt \vskip .3in}

\title{
\toptitlebar
{{\center\baselineskip 18pt
{\Large\bf On the Interaction of Batch Noise, Adaptivity, and Compression, under \texorpdfstring{$(L_0,L_1)$}{(L0,L1)}-Smoothness: An SDE Approach}}
} 
\bottomtitlebar}
\date{}

\author{%
\parbox{0.95\textwidth}{\centering
Enea Monzio Compagnoni$^{1}$ \quad
Rustem Islamov$^{1}$ \quad
Frank Norbert Proske$^{2}$\\[0.35em]
 \quad Aurelien Lucchi$^{1,\dagger}$
 \quad Antonio Orvieto$^{4,5,6,\dagger}$
Eduard Gorbunov$^{3,\dagger}$\\[1.4em]
{\small
$^{1}$University of Basel, Switzerland\\
$^{2}$University of Oslo, Norway\\
$^{3}$MBZUAI, Abu Dhabi, United Arab Emirates\\
$^{4}$Max Planck Institute for Intelligent Systems, Germany\\
$^{5}$ELLIS Institute Tübingen, Germany\\
$^{6}$Tübingen AI Center, Germany
}
}%
}

\begin{document}

\maketitle

\begingroup
\renewcommand{\thefootnote}{\fnsymbol{footnote}}
\footnotetext[2]{Eduard Gorbunov, Aurelien Lucchi, and Antonio Orvieto share senior supervision.}
\endgroup

\begin{abstract}
\noindent Distributed stochastic optimization intertwines (i) stochastic gradient noise, (ii) communication compression, and (iii) adaptive/normalized updates. While each factor has been studied in isolation, their joint effect under realistic assumptions remains poorly understood. In this work, we develop a unified theoretical framework for Distributed Compressed SGD (DCSGD) and its sign variant Distributed SignSGD (DSignSGD) under the recently introduced $(L_0, L_1)$-smoothness condition. From a conceptual perspective, we show that the first- and second-order modified equations from the literature do not accurately model the discrete-time stepsize/stability restrictions, especially under $(L_0,L_1)$-smoothness. From a technical perspective, we propose new \textit{first}-order SDEs by carefully incorporating curvature-dependent terms into their drift: This helps capture the fine-grained relationship between learning rate restrictions, gradient noise, compression, and the geometry of the loss landscape. Importantly, we do so under general gradient noise assumptions, including heavy-tailed and affine-variance regimes, which extend beyond the classical bounded-variance setting. Our results suggest that normalizing the updates of DCSGD emerges as a natural condition for stability, with the degree of normalization precisely determined by the gradient noise structure, the landscape’s regularity, and the compression rate. In contrast, DSignSGD converges even under heavy-tailed noise with standard learning rate schedules. Together, these findings offer both new theoretical insights and perspectives, and practical guidance.
\end{abstract}

\section{Introduction}\label{sec:Introduction}

Distributed stochastic gradient methods are the workhorse of modern large-scale learning.
In practice, their behavior is shaped by three effects that interact nontrivially:


\begin{enumerate}[leftmargin=*, labelindent=0em]
    \item \textbf{Batch noise.} Stochastic gradient methods rely on mini-batches to reduce computational cost, but this introduces uncertainty in the gradient estimates. In practice, this noise may not only be non-vanishing but can exhibit complex, heavy-tailed behavior~\citep{simsekli2019tailindex}. Such noise has a profound impact on convergence rates, stability, and generalization, especially in nonconvex landscapes.

\vspace{0.2cm}
    
    \item \textbf{Communication compression.} In distributed systems, communicating full-precision gradients is often prohibitively expensive. To alleviate this, gradient compression techniques such as sparsification, quantization, and sign-based schemes are commonly used. While these methods reduce communication overhead, they alter the optimization dynamics by introducing bias and additional variance~\citep{alistarh2017qsgd}. Understanding the trade-off between efficiency and convergence guarantees under compression remains a central question.

\item \textbf{Adaptivity.} Many successful optimizers in deep learning, such as Adam, AdaGrad, or SignSGD, incorporate some form of normalization or adaptivity in their update rules. Adaptivity has been empirically shown to mitigate the detrimental effects of noise and ill-conditioning~\citep{safaryan2021signsgd,staib2019escaping}, yet a rigorous understanding of \emph{why} adaptivity helps in distributed and noisy scenarios is still incomplete. In particular, the interaction between adaptivity and the statistical properties of the gradient noise is far from fully understood.
\end{enumerate}

Despite substantial work studying each component separately, their \emph{joint} interplay remains underexplored, especially under realistic assumptions on the loss landscape. Most theoretical results rely on \emph{$L$-smoothness}, i.e., the assumption that the gradient of the objective function is globally Lipschitz continuous~\citep{bubeck2015convex}. While this simplifies analysis, it fails to capture the complexities of practical problems, including those encountered in nonconvex optimization for DL. In contrast, \emph{$(L_0,L_1)$-smoothness} allows the norm of the Hessian of the loss to grow at most affinely with its gradient norm: This relaxes the aforementioned regularity condition~\citep{zhang2019gradient} and is a much more realistic alternative. Similarly, while most of the literature relies on the assumption that the gradient noise is bounded or has bounded variance, more realistic models, such as affine variance and heavy-tailed noise, are increasingly adopted in recent works.

In this paper, SDEs are \emph{models} of discrete optimizers: they are not intended to reproduce the iterate-by-iterate trajectory, but rather to capture selected properties of the dynamics. Since there is no ``universally correct'' SDE, we focus on one key property, the learning-rate restrictions that ensure stability for distributed stochastic optimizers under $(L_0,L_1)$-smoothness, when there may exist \emph{no single constant stepsize} that is stable uniformly over all initializations.

\vspace{-0.1cm}

\textbf{On the limitations of classic models from the literature.} Following the standard approach, we first considered \emph{first-order} SDEs, which are naturally well-suited to handle a broad range of noise models, e.g., Gaussian, affine, or even heavy-tailed, and historically led to new conceptual \citep{Su2014nesterov} and practical \citep{jastrzkebski2017three} insights. However, these models can be misleading: As we prove in Sec. \ref{sec:fail1}, they do not prescribe any learning rate constraints and incorrectly suggest that constant-stepsize SGD converges unconditionally. A natural next step is to move to the \emph{second-order} model from the literature \citep{li2017stochastic}. As we show in Sec. \ref{sec:fail2}, this is even more problematic: it additionally predicts \emph{accelerated convergence} at large stepsizes, where SGD, in fact, diverges. Beyond this, neither of these models captures the more subtle mechanisms that arise under $(L_0,L_1)$-smoothness, where \textit{no universal stepsize can guarantee stability uniformly across initializations}.

\vspace{-0.1cm}

\textbf{Our approach: stability-corrected SDEs.}
To address this, we derive \emph{stability-faithful} SDE models by modifying the drift to carefully include curvature-dependent information. The resulting models are \emph{first-order weak approximations} which correctly encode the stability thresholds and align closely with the dynamics of their respective optimizers.

We develop a comprehensive analysis of DCSGD and DSignSGD under $(L_0,L_1)$-smoothness with flexible gradient noise assumptions encompassing affine variance and heavy-tailed noise. In settings already examined in the literature, such as $(L_0,L_1)$-smoothness with affine variance,\footnote{This setting is studied for Normalized SGD and AdaGrad \citep{faw2023beyond,chen2023generalized}, \textbf{and not} for DCSGD or DSignSGD.} our results are consistent with established findings. In previously unexplored regimes, where $(L_0,L_1)$-smoothness, affine variance, heavy-tailed noise, and gradient compression are brought together under a unified framework, our analysis provides novel insights that advance the understanding of the interaction of these factors.

\paragraph{Contributions.}

\begin{enumerate}[leftmargin=*, labelindent=0em]

\vspace{0.1cm}

    \item \textbf{Conceptual: Limitations of models from the literature.} We identify concrete ways in which classical first- and second-order SDE models can fail to capture the correct discrete-time stepsize/stability behavior;

    \vspace{0.1cm}
    
    \item \textbf{Technical: Stability-faithful continuous-time models under $(L_0,L_1)$-smoothness.}
    We \textit{formally} derive first-order SDE models that correctly capture the learning rate restrictions and stability thresholds \textit{also} under $(L_0,L_1)$-smoothness, which cannot be done by classic SDEs;

    \vspace{0.1cm}
    
    \item \textbf{Technical: Unified SDE analysis of compression and nonstandard noise.} We prove convergence bounds for the models of DCSGD and DSignSGD under $(L_0,L_1)$-smoothness and batch noise assumptions more general than those commonly used in the literature, namely, affine-variance noise for DCSGD and heavy-tailed noise for DSignSGD;

    \vspace{0.1cm}
    
    \item \textbf{Practical: Normalization strength for DCSGD.}
    We demonstrate that the degree of normalization required for DCSGD to converge is precisely determined by the interplay between the compression rate, the structure of gradient noise, and the smoothness constants of the loss;

    \vspace{0.1cm}
    
    \item \textbf{Practical: Robustness of DSignSGD under heavy tails.}
    We show that an \textit{adaptive} method such as DSignSGD converges even under heavy-tailed noise with standard assumptions on the learning rate scheduler.
\end{enumerate}

\begin{table*}[t]
\centering
\caption{Comparison of existing convergence results for stochastic methods applied to $(L_0,L_1)$-smooth problems. All results are derived for non-convex problems, and the bounds are given in expectation unless stated otherwise. All works assume bounded noise or bounded variance unless stated otherwise. Abbreviations: ``HT'' = heavy-tailed noise, ``Affine var.'' = affine variance.}
\label{tab:comparison}

\setlength{\tabcolsep}{4pt}
\renewcommand{\arraystretch}{0.9}

\begin{threeparttable}
\begin{tabular*}{\textwidth}{@{\extracolsep{\fill}}lccccc@{}}
\toprule[1pt]
\multirow{2}{*}{\textbf{Reference}} &
\multirow{2}{*}{\textbf{Dynamics}} &
\multicolumn{2}{c}{\textbf{Noise}} &
\multirow{2}{*}{\textbf{$(L_0,L_1)$-smooth}} &
\multirow{2}{*}{\textbf{Compression}} \\[-0.3em]
\cmidrule(lr){3-4}
& & \textbf{HT} & \textbf{Affine var.} & & \\
\midrule[1pt]
\makecell[l]{\citep{zhang2019gradient, zhang2020improved}\\
\citep{zhao2021convergence}\\
\citep{crawshaw2022robustness}\\
\citep{koloskova2023revisiting}\\
\citep{li2023convex}\tnote{\color{blue}(1)}~~~~\tnote{\color{blue}(2)}\\
\citep{hubler2024parameter}\\
\citep{li2024convergence}\tnote{\color{blue}(1)}~~~~\tnote{\color{blue}(3)}\\
\citep{gaash2025convergence}\tnote{\color{blue}(1)}~~~~\tnote{\color{blue}(3)}}
& Discrete & \xmark & \xmark & \cmark & \xmark \\
\midrule[1pt]
\makecell[l]{\citep{faw2023beyond}\tnote{\color{blue}(1)}~~~~\tnote{\color{blue}(2)}\\
\citep{wang2023convergence}\tnote{\color{blue}(1)}~~~~\tnote{\color{blue}(2)}\\
\citep{chen2023generalized}}
& Discrete & \xmark & \cmark & \cmark & \xmark \\
\midrule[1pt]
\citep{khirirat2024error} & Discrete & \xmark & \xmark & \cmark & \cmark \\
\midrule[1pt]
\citep{chezhegov2025convergence}\tnote{\color{blue}(1)}~~~~\tnote{\color{blue}(3)}
& Discrete & \cmark & \xmark & \cmark & \xmark \\
\midrule[1pt]\midrule[1pt]
\citep{compagnoni2025unbiased} & Continuous & \cmark & \xmark & \xmark & \cmark \\
\textbf{This work} & Continuous & \cmark & \cmark & \cmark & \cmark \\
\bottomrule[1pt]
\end{tabular*}

\begin{tablenotes}[flushleft]
\footnotesize
\item[{\color{blue}(1)}] High-probability convergence analysis.
\item[{\color{blue}(2)}] Convergence bounds have inverse-power dependence on the failure probability.
\item[{\color{blue}(3)}] Derived for convex problems.
\end{tablenotes}
\end{threeparttable}

\end{table*}

\section{Related work}
\label{sec:RelatedWorks}

\vspace{-0.1cm}

\paragraph{SDE Approximations in Optimization.} Continuous-time models in the form of differential equations are a well-established tool to study discrete-time optimizers, e.g.~\citep{helmke1994optimization,kushner2003stochastic}. It was~\citep{li2017stochastic} that first introduced a \textit{rigorous} theoretical framework to derive SDEs that faithfully model the stochastic behavior intrinsic to optimization algorithms widely employed in machine learning. Since then, such SDE-based formulations have been applied across several domains, including \emph{stochastic optimal control} for tuning stepsizes~\citep{li2017stochastic,li2019stochastic} and batch sizes~\citep{zhao2022batch}. Notably, SDEs have been instrumental in analyzing \emph{convergence bounds} and \emph{stationary distributions}~\citep{compagnoni2023sde,compagnoni2024sde,compagnoni2025adaptive}, \emph{scaling laws}~\citep{jastrzkebski2017three,compagnoni2025unbiased,compagnoni2025adaptive,compagnoni2026adaptive}, \emph{implicit regularization} effects~\citep{smith2021origin,compagnoni2023sde}, and \textit{implicit preconditioning}~\citep{xiao2025exact}.

\textbf{Why continuous-time models are useful.} As mentioned above, continuous-time modeling via ODEs/SDEs is a classical tool in optimization and has seen significant growth in recent years~\citep{li2017stochastic}: The main advantage is getting access to sophisticated tools such as Itô calculus, which turns SDEs for the iterates into SDEs for the loss, enabling Lyapunov-style stability arguments to derive convergence bounds. To the best of our knowledge, no prior work has used SDEs to analyze optimizers under $(L_0,L_1)$-smoothness. This represents a significant gap, as this assumption has emerged as an alternative to $L$-smoothness to model nonconvex landscapes.

\vspace{-0.15cm}

\textbf{A note on modeling.}
Before proceeding, we clarify what we mean by \emph{model}: It is a deliberately simplified mathematical surrogate tailored to a specific question; it is not expected to be universally accurate, but rather predictive within the regime relevant to that question. An example comes from physics: General Relativity describes macroscopic gravity, but it is not a quantum theory. In contrast, quantum mechanics excels in the microscopic regime, yet does not describe spacetime geometry. Both are valuable, and their limitations outside their domains are a reminder that the ``right'' model depends on the phenomenon of interest.

\vspace{-0.3cm}

\paragraph{$(L_0,L_1)$-smoothness, normalization, and compression.}
The $(L_0,L_1)$-smoothness condition \citep{zhang2019gradient} relaxes global $L$-smoothness and has motivated analyses of first-order methods beyond bounded-variance settings, e.g., for Normalized SGD/AdaGrad/Adam/SignSGD variants
\citep{chen2023generalized,faw2023beyond,wang2023convergence,li2024convergence,crawshaw2022robustness}. For compressed communication under relaxed smoothness, \citet{khirirat2024error} provides a representative analysis, but existing results do not offer a unified treatment that simultaneously captures compression together with affine-variance growth and heavy-tailed regimes under $(L_0,L_1)$-smoothness. Table~\ref{tab:comparison} summarizes the closest guarantees and clarifies our positioning. Finally, we refer the interested reader to Appendix \ref{sec:RelatedWorks_App} for an extended discussion on the related works.

\vspace{-0.25cm}

\paragraph{Positioning.}
We identify a modeling gap relevant to stepsize stability: classic first- and second-order modified equations can contradict discrete-time stability behavior, especially under $(L_0,L_1)$-smoothness. We address this by deriving stability-corrected first-order SDEs and using them to provide, to the best of our knowledge, the first unified SDE-based analysis that is both stability-faithful under $(L_0,L_1)$-smoothness and suited to analyze the joint effect of distributed compression with general noise regimes, including affine-variance growth and heavy tails.

\vspace{-0.35cm}

\section{Preliminaries}\label{sec:Insights}

\vspace{-0.25cm}

\paragraph{Distributed Setup.} Let us consider the problem of minimizing an objective function expressed as an average of $N$ functions: $ \min_{x\in\R^d} \left[f(x) \eqdef \frac{1}{N}\sum_{j=1}^N f_j(x)\right]$, where each $f_j:\R^d\to\R$ is lower bounded and twice continuously differentiable, and represents the loss over the local data of the $i$-th client. In our stochastic setup, each client only has access to gradient estimates: let $n_i$ be the number of datapoints accessible to client $i$; at a given $x\in\R^d$, client $i$ estimates $\nabla f(x)$ using a batch of data $\gamma_i\subseteq \{ 1, \dots, n_i \}$, sampled uniformly with replacement and uncorrelated from the previously sampled batches. Under homogeneous data/client assumption, $\nabla f_{i,\gamma_{i}}(x)$ can be modeled as a perturbation of the global gradient: $\nabla f_{i,\gamma_{i}}(x) = \nabla f(x) + Z_i(x)$\footnote{Our SDE framework can in principle be extended to heterogeneous objectives.}.  Finally, $S_0:= f(X_0) - f(X_*)$.

\paragraph{Noise assumptions.} We assume the sampling process and client configurations are such that, for all $x\in\R^d$ and each client pair $(i,j)$ with $i\ne j$, $Z_i(x)$ is independent of $Z_j(x)$. Regarding the noise structure, we always assume that at each $x\in\mathbb{R}^d$, $Z_i(x)$ is absolutely continuous and has a coordinate-wise symmetric distribution. For context, we highlight that numerous works in the literature assume a much more restrictive assumption, e.g. that $Z_i(x)$ are Gaussian~\citep{ahn2012bayesian,chen2014stochastic,mandt2016variational,stephan2017stochastic,zhu2019anisotropic,wu2020noisy,Xie2021}, and  \cite{li2017stochastic,mertikopoulos2018convergence,raginsky2012continuous,zhu2019anisotropic,mandt2016variational,ahn2012bayesian,jastrzkebski2017three} even assume the covariance matrix of the batch noise to be constant: we refer the reader to \cite{jastrzkebski2017three} for the intuition behind this modeling choices. If $Z_i(x) \in L^1(\R^d)$, we assume $\E[Z_i(x)] = 0$; if $Z_i(x) \in L^2(\R^d)$, we denote $\Sigma_i(x):=\Cov(Z_i(x))$, where $L^p(\R^d)$ denotes the space of random variables with finite $p$-th moment.

Next, we define our two structural assumptions. The first one strictly concerns the global landscape; the second concerns how global landscape features affect the noise distribution of each client.

\begin{definition}$f$ is $\left(L_0, L_1\right)$-smooth~($L_0,L_1\ge0$) if, $\forall x\in\R^d$, $\left\|\nabla^2 f(x)\right\|_2 \leq L_0+L_1\|\nabla f(x)\|_2$.
\end{definition}

\begin{definition}[Extension of the assumptions from \cite{schmidt2013fast}] The gradient noise for client $i$ has \textit{affine} $(\sigma_{0,i}^2, \sigma_{1,i}^2)$-variance if $\lVert \Sigma_i(x) \rVert_{\infty}  \le \sigma_{0,i}^2 + \sigma_{1,i}^2 \lVert \nabla f(x) \rVert_2^2$. If $\sigma_{1,i}=0$, the noise has bounded variance.
\end{definition}

Finally, we define which compressors we use to reduce the communication costs of gradients.
\begin{definition}
An unbiased compressor is a stochastic map $\mathcal{C}_{\xi}: \R^d \rightarrow \R^d$ such that $(a)$ $\E_{\xi} \left[ \mathcal{C}_{\xi} (x) \right] = x$ and $(b)$ $\E_{\xi} \left[ \lVert \mathcal{C}_{\xi} (x) - x \rVert_2^2 \right] \leq \omega  \lVert x \rVert_2^2$ for a compression rate $\omega\geq 0$.
\end{definition}

\paragraph{SDE approximations.} The following definition presents the most commonly used notion that formalizes how an SDE can be a ``reliable surrogate'' to model an optimizer. It is drawn from the field of numerical analysis of SDEs (see \cite{mil1986weak}), and it quantifies the disparity between the discrete and the continuous processes.

\begin{definition}\label{def:weak_approximation}
A continuous-time stochastic process $(X_t)_{ t \in [0, T]}$ is an $\alpha$-order weak approximation of a discrete stochastic process $(x_k)_{k=0}^{\lfloor  T/\eta \rfloor }$ if for every polynomial growth function $g$, $\exists C>0$, independent of $\eta$, such that $ \max _{k=0, \ldots, \lfloor  T/\eta \rfloor }\left|\E g\left(x_k\right)-\E g\left(X_{k \eta}\right)\right| \leq C \eta^\alpha$.
\end{definition}

We will often refer to $1$-order and $2$-order weak approximations as \textit{first}- and \textit{second}-order SDEs. Here, ``\textit{first}-order'' refers to the order of the weak-approximation error in $\eta$, not to the absence of Hessian terms in the drift.

\paragraph{Scope of SDE guarantees and results.}
As is standard in the literature using SDEs for optimization, e.g., \citep{li2017stochastic,li2019stochastic,luo2024explicit,orvieto2019shadowing}, theoretical guarantees are stated for the \emph{continuous-time models} rather than for the discrete optimizers. The link to the underlying algorithm is provided by weak-approximation (or shadowing) results, which quantify how closely the statistics of the discrete iterates and of the SDE match over finite horizons, with an error controlled by the stepsize (cf. Def. ~\ref{def:weak_approximation}). In addition, several recent works empirically validate that such SDEs do track the respective optimizers on modern DNNs and datasets,
supporting their use as models \citep{compagnoni2025unbiased,compagnoni2025adaptive,marshall2025to}.

\paragraph{Example: classical vs. new SDEs for SGD.}
Thm 1 of~\citep{li2017stochastic} originally formally derived the first- and second-order SDEs for single-node SGD with stepsize $\eta$. As we denote the covariance batch noise with $\Sigma(x) = \frac{1}{n} \sum_{i=1}^n (\nabla f(x) - \nabla f_i(x))(\nabla f(x) -\nabla f_i(x))^T$, the \textit{first}-order one reads

\begin{figure}[ht!]
    \centering
    \hspace{.15cm}
    {\includegraphics[width=0.48\linewidth]{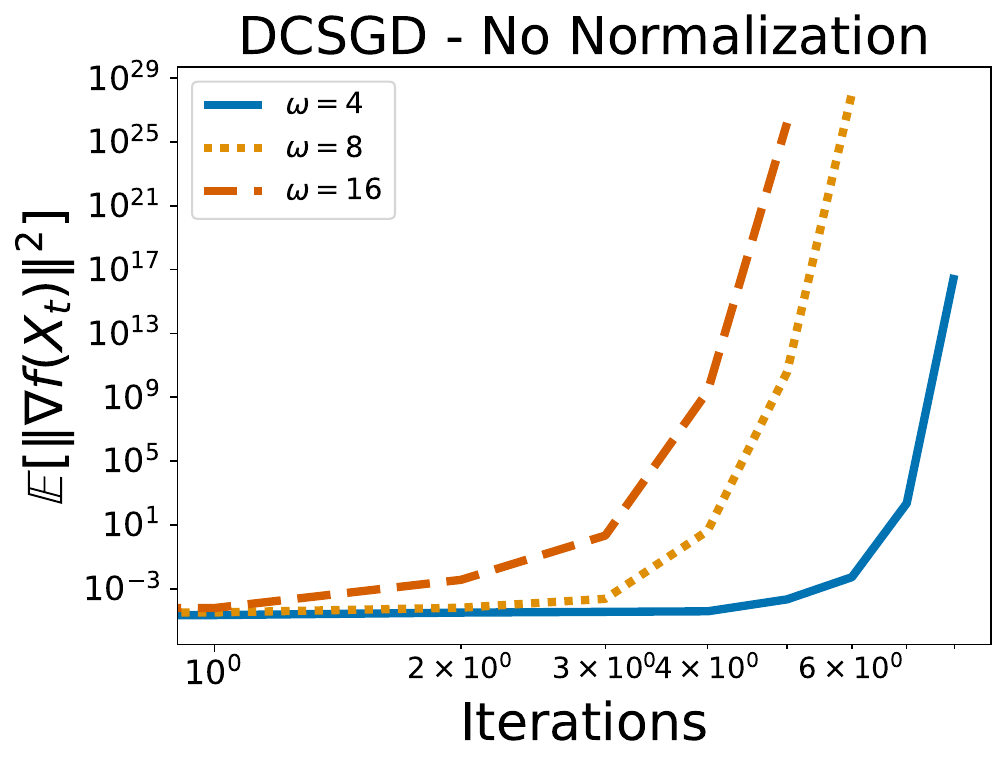}}%
    {\includegraphics[width=0.5\linewidth]{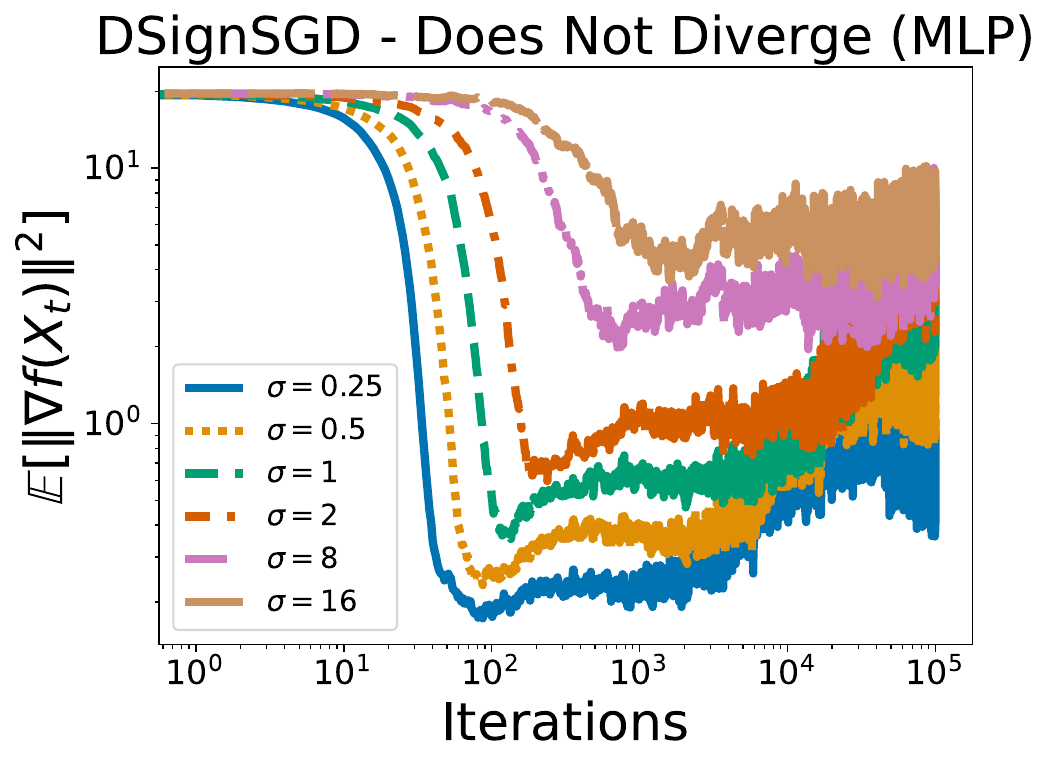}}\\
    {\includegraphics[width=0.5\linewidth]{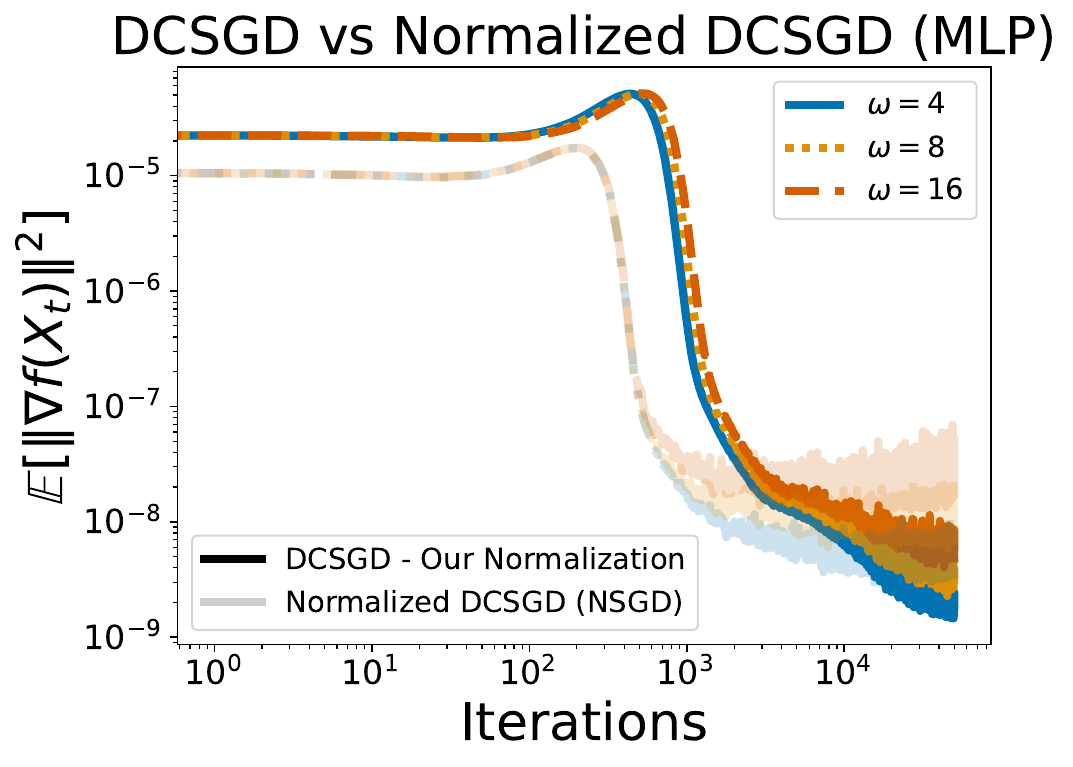}}%
    {\includegraphics[width=0.48\linewidth]{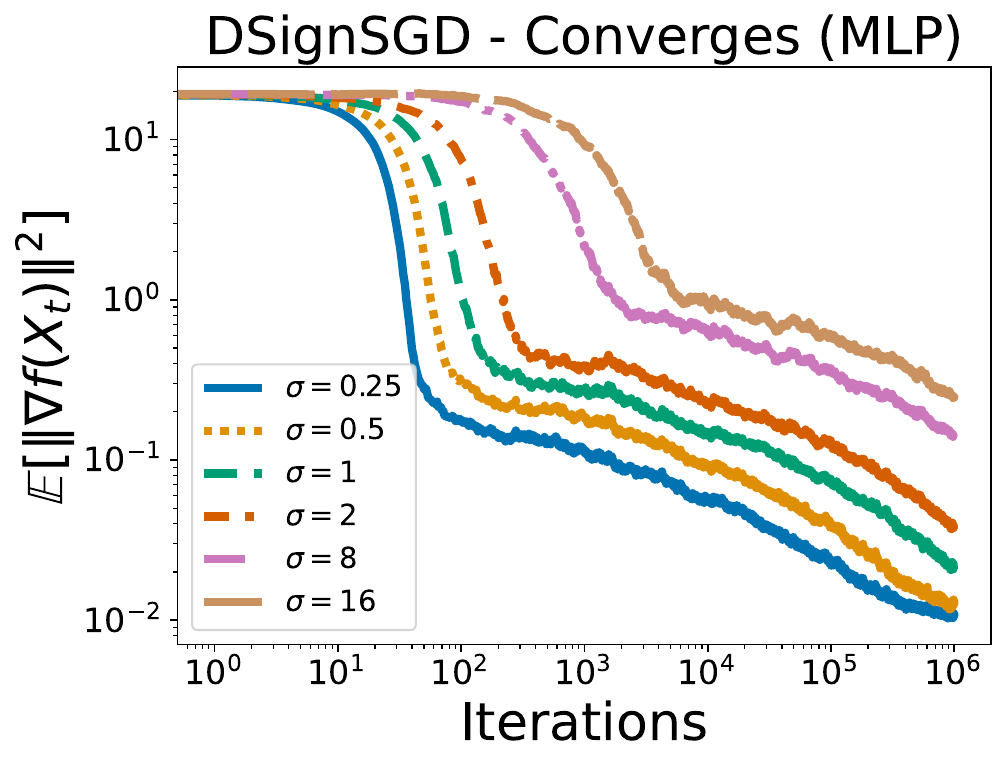}}
    \caption{\textbf{Sanity checks for our stability prescriptions on a noisy, compressed MLP training}.
\textbf{Left (DCSGD + unbiased sparsification).}
We train an MLP with $N{=}8$ clients, inject additive Gaussian gradient noise with affine variance
$Z_t\sim\mathcal{N}(0,\sigma^2\|g_t\|_2^2 I)$, and then apply unbiased random sparsification at compression levels
$\omega$ (larger $\omega$ means more aggressive sparsification).
Without any scheduler/normalization, DCSGD becomes unstable and the divergence worsens as $\omega$ increases (top-left).
Using the adaptive normalization suggested by Thm.~\ref{thm:DCSGD_2} (Eq.~\ref{eq:StepDCSGD_UV}) stabilizes training and yields convergence for all $\omega$ (bottom-left);
We also report a baseline that applies plain Normalized SGD under the same noise and compression, which exhibits a (here less stable) profile.
\textbf{Right (DSignSGD under heavy tails).}
We inject Student's $t$ gradient noise with $\nu{=}1$ for different scale values $\sigma$.
With a constant stepsize, DSignSGD remains stable but does not converge (top-right), whereas the diminishing schedule prescribed by Thm.~\ref{thm:sign}
(here $\eta_k = 1/\sqrt{k+1}$) yields convergence across noise scales (bottom-right).
See Appendix~\ref{sec:Exper} for full implementation details.}
    \label{fig:InsightValidation_MLP}
\end{figure}

\begin{equation}\label{eq:SGD_FO}
dX_t = -\nabla f(X_t) dt + \sqrt{\eta} \sqrt{\Sigma(X_t)} dW_t,
\end{equation}

while the \textit{2nd}-order SDE adds a curvature term to the drift,

\begin{equation}\label{eq:SGD_SO}
d X_t = - \nabla f(X_t) dt  - \textcolor{purple}{\frac{\eta}{2} \nabla ^2 f(X_t) \nabla f(X_t)} dt + \sqrt{\eta} \sqrt{\Sigma(X_t)} dW_t.
\end{equation}

In this work, we show that both Eq. \ref{eq:SGD_FO} and Eq. \ref{eq:SGD_SO} miss discrete-time stepsize stability restrictions, especially under $(L_0,L_1)$-smoothness. We therefore introduce a new \textit{first}-order SDE obtained by modifying the drift term:
\begin{equation}\label{eq:SGD_SO_New}
d X_t = - \nabla f(X_t) dt   \textcolor{orange}{\scalebox{1.1}{$\boldsymbol{+}$} \frac{\eta}{2} \nabla ^2 f(X_t) \nabla f(X_t)} dt + \sqrt{\eta} \sqrt{\Sigma(X_t)} dW_t.
\end{equation}

Theorem \ref{thm:gd_sde} formally proves that Eq. \ref{eq:SGD_SO_New} is a \textit{first}-order weak approximation for SGD in the sense of Def.~\ref{def:weak_approximation}, so the modification is not merely heuristic.

\paragraph{Algorithms and their SDEs.} For a constant stepsize $\eta$ and scheduler $\eta_k$, we study DCSGD defined as

\begin{equation}
x_{k+1} = x_{k} - \tfrac{\eta \eta_k}{N} \sum_{i=1}^{ N} \mathcal{C}_{\xi_i}\bigl( \nabla f_{i,\gamma_i} (x_k) \bigr),
\end{equation}
with unbiased compressors $\mathcal{C}_{\xi_i}$ with SDE models in Eq.~\ref{eq:DCSGD_SDE_App_First}--\ref{eq:DCSGD_SDE_App_Second} and DSignSGD defined as
\begin{equation}
x_{k+1} = x_{k} - \tfrac{\eta \eta_k}{N} \sum_{i=1}^{N} \sign(\nabla f_{i,\gamma_i} (x_k)),
\end{equation}
with SDEs in Eq.~\ref{eq:SDE_HSignSGD_Full_App}--\ref{eq:SDE_HSignSGD_Full_App_SO}.  

\section{Theoretical Results}

\paragraph{Step-size convention.}
Recall that, in the continuous-time setup, the dynamics of the iterates is modeled by a stochastic process {\small$X_t$} solution to an SDE model. To separate \emph{adaptivity} from \emph{scheduling}, we \emph{parameterize} the stepsize as $\eta \eta_t$, where $\eta>0$ is a base scale and $(\eta_t)_{t\ge 0}$ is a deterministic scheduler. We assume the standard Robbins-Monro conditions: defining $\phi^{(i)}_t \eqdef \int_0^t (\eta_s)^i ds$, we require $\phi^{(1)}_t \to \infty$ and $\phi^{(2)}_t/\phi^{(1)}_t \to 0$ as $t\to\infty$.
A typical choice is $\eta_t = (1+t)^{-a}$ for $a\in(0,1)$ (the boundary cases $a\in\{\tfrac12,1\}$ are also admissible with minor modifications). 

\paragraph{Overview}
Our insights concern the conditions on the learning rate $\eta\eta_t$ for convergence, where $\eta_t$ is a predetermined scheduler. We aim to determine how factors such as compression, noise structure, and adaptivity influence the level of normalization required to guarantee convergence.
First, we show how first- and second-order continuous-time models from the literature lead to misleading conclusions, as they fail to capture the stability thresholds of the learning rate of GD. Then, we justify the derivation of new models that capture this aspect of the dynamics. Finally, we present Thm. \ref{thm:DCSGD_2} and Thm. \ref{thm:sign}, which are derived under these new formulations and empirically validated in Fig. \ref{fig:InsightValidation_MLP}. These experiments should be read as \textbf{mechanism validation} rather than competitive benchmarking: their purpose is to test whether the instability and stabilization patterns predicted by the SDE analysis appear in controlled neural-network settings. To further check that these qualitative mechanisms persist beyond the MLP setup, Appendix~\ref{sec:additional_cifar_experiments} reports additional sanity checks on ResNet-18 and ViT models trained on CIFAR-10 in a distributed setting with $N=8$ clients.

\subsection{On the Failure of Classic First-Order SDE Models}\label{sec:fail1}

Consistent with the literature, we start our analysis and derive a convergence bound for DCSGD from its \textit{first}-order SDE: On the one hand, this result is very insightful and certainly captures important aspects of the dynamics. On the other hand, we quickly figure out its limitations as it fails to capture the fact that Gradient Descent on an $L$-smooth loss only converges if $\eta \eta_t < \frac{2}{L}$.

\begin{theorem}
    \label{thm:DCSGD_1}
   {\textbf{(DCSGD, unbiased compression, affine variance)}} Let $f$ be $(L_0,L_1)$-smooth, and each client has $(\sigma_{0,i}^2, \sigma_{1,i}^2)$-variance. Define $\overline{\sigma_0^2}\eqdef \frac{1}{N} \sum_{i=1}^{N} \sigma_{0,i}^2$, $\overline{\sigma_1^2}\eqdef \frac{1}{N} \sum_{i=1}^{N} \sigma_{1,i}^2$,  $\overline{\sigma_0^2 \omega}\eqdef \frac{1}{N} \sum_{i=1}^{N} \sigma_{i,0}^2 \omega_i$, and $\overline{\sigma_1^2 \omega}\eqdef \frac{1}{N} \sum_{i=1}^{N} \sigma_{i,1}^2 \omega_i$. For $\epsilon \in (0,1)$, assume $\eta \eta_t$ is smaller than
{\small\begin{equation}\label{eq:StepDCSGD_UV_1}
    \frac{2 \epsilon}{ \left( L_0 + L_1 \E\left[\lVert \nabla f(X_t) \rVert_2 \right] \right) \frac{\overline{\omega} + d (\overline{\sigma_1^2 \omega} + \overline{\sigma_1^2})}{N} +  \frac{  L_1 d\left(\overline{\sigma_0^2}  +  \overline{ \sigma_0^2 \omega } \right)}{N}   }.
\end{equation}}
Then, for a random time $\Hat{t}$ (indep. of $X_t$) with distribution $\frac{\eta_t}{\phi^{(1)}_t}$, we have that $\E \left[\lVert \nabla f(X_{\Hat{t}}) \rVert_2^2 \right]$ is smaller than
{\small\begin{equation*}
    \frac{S_0}{(1-\epsilon) \phi^{(1)}_t}  + \phi^{(2)}_t \frac{  L_0(\overline{\omega} + d (\overline{\sigma_1^2 \omega} + \overline{\sigma_1^2})) +  L_1 d\left(\overline{\sigma_0^2}  +  \overline{ \sigma_0^2 \omega } \right)}{2 N (1-\epsilon) \phi^{(1)}_t} \overset{t \rightarrow \infty}{\rightarrow} 0.
\end{equation*}}
\end{theorem}

\textbf{Takeaways.} This result highlights the role of the regularity of the loss landscape and its interaction with both gradient noise and compression. Eq.~\ref{eq:StepDCSGD_UV_1} makes explicit how stability tightens as the system becomes noisier, higher-dimensional, and more aggressively compressed. A way to read it is: \textbf{(i) More clients help:} increasing $N$ relaxes the constraint by countering noise and compression effects; \textbf{(ii) Dimension hurts:} larger $d$ tightens the condition, especially under affine variance ($\overline{\sigma_1^2}>0$); \textbf{(iii) Compression hurts:} $\overline{\omega}>0$ behaves as an additional distortion term and tightens the stability region;  \textbf{(iv) $(L_0,L_1)$ matters:} the constraint couples $L_1$ with noise and compression; when $L_1>0$, stability becomes
more sensitive to the local geometry through $\E\|\nabla f(X_t)\|_2$.

At the same time, the bound reveals the core \textbf{modeling limitation}: in the noiseless, uncompressed, $L$-smooth case
($\sigma_{0,i}=\sigma_{1,i}=\omega_i=L_1=0$), the denominator in Eq.~\ref{eq:StepDCSGD_UV_1} vanishes and the condition does not
restrict $\eta\eta_t$, contradicting the basic discrete-time requirement $\eta\eta_t<2/L_0$ for GD/SGD stability \cite{ahn2022understanding}.
This motivates going beyond classical first-order surrogates.

Similarly, in Theorem \ref{thm:sign_1}, we leverage the first-order SDE of DSignSGD to derive the convergence bound of DSignSGD. While it recovers the results from \citep{compagnoni2025unbiased} when $L_1 =\sigma_1 = 0$, it also predicts no restrictions on the learning rate in the noiseless scenario.

Finally, we highlight that this bound is not implementable in practice, as constants such as $L_0$, $L_1$, $\sigma_{0,i}$, $\sigma_{1,i}$, and even $\E\left[\lVert \nabla f(X_t) \rVert_2 \right]$ are not known a priori. While this might look like a limitation, \textbf{this is common} in the literature \citep{gorbunov2024methods}: As with classical $\eta<2/L$-type conditions, only identify stability regions and show how different factors driving the dynamics influence it.

\vspace{-0.2cm}

\subsection{Second-Order Models Fail: We Need New Models}\label{sec:fail2}
\vspace{-0.1cm}
For didactic reasons, we now showcase a clear example of how first- and second-order models for GD fail at capturing its learning-rate restrictions. To keep the discussion transparent, we use the noiseless single-node case; Similar issues arise in the stochastic case and motivate our new SDE: We formalize this in Sec.~\ref{sec:new_models} and Thm.~\ref{thm:gd_sde}.

\paragraph{A quadratic sanity check.}
Consider $f(x)=\tfrac{\lambda x^2}{2}$ with $\lambda>0$ in one dimension.
Discrete GD with constant stepsize $\eta$ satisfies
$x_{k+1} = (1-\eta\lambda) x_k$: it is stable iff $\eta<2/\lambda$. The first-order ODE $dX_t = -\nabla f(X_t) dt = -\lambda X_t dt$ yields
$f(X_t)=f(X_0)e^{-2\lambda t}$, i.e., convergence for \emph{all} $\eta$: it does not even ``see'' $\eta$. One might expect that moving from the first-order to the  \emph{second-order} ODE might fix the issue.
Surprisingly, this is not the case. The classical 2nd-order ODE from the literature,

\vspace{-0.35cm}

\begin{equation}
    dX_t = -\nabla f(X_t) dt - \frac{\eta}{2}\nabla^2 f(X_t)\nabla f(X_t) dt,
\end{equation}

\vspace{-0.25cm}

implies $f(X_t)=f(X_0)e^{-2\lambda(1+\lambda\eta/2)t}$: it predicts unconditional convergence and even suggests that larger $\eta$ accelerates convergence, where instead GD would diverge if the learning rate were large enough.

\vspace{-0.25cm}

\paragraph{A quartic sanity check: no universal stepsize under $(L_0,L_1)$-smoothness.}
The quadratic example above highlights that classical ODE/SDE surrogates can miss the \emph{global} $L$-smooth stepsize
restriction. Under $(L_0,L_1)$-smoothness, an even sharper pathology occurs: there may be \emph{no single constant stepsize}
that guarantees stability \emph{uniformly over initializations}.
A minimal example is the one-dimensional quartic
\(
f(x)=\tfrac{x^4}{4},
\)
which is not $L$-smooth. In this case, a GD step with constant stepsize $\eta$ reads
\[
x_{k+1}=x_k-\eta \nabla f(x_k)=x_k-\eta x_k^3 = x_k(1-\eta x_k^2).
\]
Hence, one needs $\eta< 2/x_k^2$ to prevent instantaneous expansion, meaning that stability depends on the \emph{current scale} of the iterate. In particular, there is \emph{no universal constant} $\eta>0$ that stabilizes GD for all initialization: for any $\eta$, choosing $|x_0|>\sqrt{2/\eta}$ makes GD expand at the first step.

The first-order ODE is
\(
dX_t=-X_t^3\,dt,
\)
which yields an explicit globally convergent trajectory
\(
X_t=(X_0^{-2}+2t)^{-1/2}
\)
and thus predicts convergence \emph{independently} of $\eta$.
The classical second-order ODE,
\(
dX_t=-X_t^3\,dt-\tfrac{3\eta}{2}X_t^5\,dt,
\)
makes the loss decrease even \emph{faster} as $\eta$ grows (the additional term is purely damping), again contradicting the discrete-time instability for large initialization.

\textbf{Takeaway:} 
The modeling failure of classic ODE models is twofold: \textbf{(i)} Failure to restrict the learning rate to ensure stability; \textbf{(ii)} Failure to reproduce divergence at large learning rates. While the first is non-contestable, the second might feel a bit stretched, as ODEs are more reliable when the learning rate is small, and one could argue we pushed the use of ODEs outside of their ``validity region''. However, our analysis of the quartic function shows that even for an infinitesimally small learning rate, both the ODEs of GD predict convergence independently of the learning rate, while GD actually has a restriction on the learning rate depending on the initialization. Therefore, neither first- nor second-order ODEs of GD correctly model the dynamics of GD on $(L_0,L_1)$-smooth functions even at infinitesimally small learning rates: More details in Sec. \ref{sec:quartic}.

\paragraph{Comparison With Discrete-Time Analysis}
Here, we take a step back and closely compare the dynamics of the loss function in discrete-time with that in continuous time as prescribed by the ODEs of GD. Using a second-order Taylor expansion around $x_k$ along the GD step gives
\begin{align}\label{eq:loss_discr}
    \frac{f(x_{k+1}) - f(x_k)}{\eta}
    & = - \|\nabla f(x_k)\|^2 + O_{x_k}(\eta^2)  + \textcolor{purple}{\tfrac{\eta}{2}\nabla f(x_k)^\top \nabla^2 f(x_k)\, \nabla f(x_k)}.
\end{align}
Here, $O_{x_k}(\eta^2)$ denotes a term bounded by $C(x_k)\eta^2$ for small enough $\eta$, with $C(x_k)$ independent of $\eta$. However, the first-order ODE of GD implies that
\begin{equation}
    d f(X_t) = - \lVert \nabla f(X_t) \rVert_2^2 dt.
\end{equation}
We notice that this continuous-time loss drift is completely missing the $O(\eta)$ correction highlighted in \textcolor{purple}{purple color} in Eq.~\ref{eq:loss_discr}. The natural step is to use the 2nd-order ODE, which implies that

\begin{align}
    d f(X_t) & = - \lVert \nabla f(X_t) \rVert_2^2 dt  \textcolor{emphgreen}{\scalebox{1.1}{$\boldsymbol{-}$}}      \textcolor{purple}{\tfrac{\eta}{2} \nabla f(X_t)^{\top}\nabla^2 f(X_t) \nabla f(X_t) dt}.
\end{align}

While this ODE of the loss \textit{does} incorporate some second-order information highlighted in \textcolor{purple}{purple color}, we notice that its \textcolor{emphgreen}{sign} is flipped with respect to the discrete-time loss drift in Eq.~\ref{eq:loss_discr}. This flipped sign is exactly the factor responsible for the failures of this second-order ODE.

\begin{table}[t]
\centering
\caption{Comparison of the learning rate constraints on $\eta \eta_t$ derived from classic SDEs vs.\ our SDEs.
Each row corresponds to the theorem pairs: DCSGD (Thm.~\ref{thm:DCSGD_1} vs.~Thm.~\ref{thm:DCSGD_2}) and DSignSGD (Thm.~\ref{thm:sign_1} vs.~Thm.~\ref{thm:sign}).
Below, $G \eqdef \left( L_0 + L_1 \E\left[\lVert \nabla f(X_t) \rVert_2 \right] \right)$.}
\label{tab:compare_1st_2nd}

\setlength{\tabcolsep}{6pt}
\renewcommand{\arraystretch}{1.25}
\footnotesize

\begin{threeparttable}
\resizebox{\columnwidth}{!}{%
\begin{tabular}{@{}c c c@{}}
\toprule[1pt]
\textbf{Setting} & \textbf{Classic SDE} & \textbf{Our SDE} \\
\midrule[1pt]

\textbf{DCSGD}
&
$\displaystyle
\frac{2 \epsilon}{
G \tfrac{\overline{\omega} + d (\overline{\sigma_1^2 \omega} + \overline{\sigma_1^2})}{N}
+ \tfrac{L_1 d\left(\overline{\sigma_0^2} + \overline{\sigma_0^2 \omega}\right)}{N}
}
$
&
$\displaystyle
\frac{2 \epsilon}{
G\left(
\textcolor{orange}{\mathbf{1}}
+ \tfrac{\overline{\omega} + d (\overline{\sigma_1^2 \omega} + \overline{\sigma_1^2})}{N}
\right)
+ \tfrac{L_1 d\left(\overline{\sigma_0^2} + \overline{\sigma_0^2 \omega}\right)}{N}
}
$
\\

\midrule[1pt]

\textbf{DSignSGD}
&
$\displaystyle
\frac{\ell_{\nu}}{K}
\quad \text{s.t.} \quad
K = \frac{L_1 d \sigma_{\mathcal{H},1}}{2N}
$
&
$\displaystyle
\frac{\ell_{\nu}}{K}
\quad \text{s.t.} \quad
K = \frac{L_1 d \sigma_{\mathcal{H},1}}{2N}
+ \textcolor{orange}{\sqrt{d}\left(L_0+L_1\right)M_{\nu}}
$
\\

\bottomrule[1pt]
\end{tabular}%
}
\end{threeparttable}
\end{table}

\textbf{Deriving a New Model: An Ansatz Approach.}
Therefore, we understand that choosing the right model for the iterates is critical to capture the aspects of the dynamics under analysis. Inspired by a classic approach in physics, we propose an \textit{ansatz} for an ODE of the iterates of GD and seek one that models the loss dynamics more closely, which is exactly the information needed to recover the correct stepsize stability threshold. For $\alpha \in \mathbb{R}$, we propose:

\begin{equation}\label{eq:ansatz}
    dX_t = - \nabla f(X_t) dt + \alpha\eta\nabla^2 f(X_t) \nabla f(X_t) dt,
\end{equation}
which implies that the loss dynamics is driven by
\begin{align}
    d f(X_t) & = - \lVert  \nabla f(X_t) \rVert_2^2 dt  + \alpha\eta \nabla f(X_t)^{\top}\nabla^2 f(X_t) \nabla f(X_t) dt. 
\end{align}
Matching the induced loss drift with the discrete-time generator expansion in Eq.~\ref{eq:loss_discr} (i.e., comparing $\frac{d}{dt}f(X_t)$ with $\frac{f(x_{k+1})-f(x_k)}{\eta}$ up to $O(\eta)$) suggests choosing $\alpha = \frac{1}{2}$, which is the unique choice within the ansatz family that recovers the exact quadratic stability threshold. Therefore, we obtain

\begin{equation}\label{eq:our_ode}
    dX_t = - \nabla f(X_t) dt  \textcolor{orange}{\scalebox{1.1}{$\boldsymbol{+}$}}       \frac{\eta}{2}\nabla^2 f(X_t) \nabla f(X_t) dt,
\end{equation}
is our candidate ODE for GD: We formalize this in Section \ref{sec:new_models}, and Theorem \ref{thm:gd_sde} extends this formalization to the stochastic setting. Importantly, in the \textbf{quadratic case} studied above, it implies that
\begin{equation}
    f(X_t) = f(X_0) e^{-2\lambda\left(1-\tfrac{\lambda \eta}{2}\right)t},
\end{equation}
which, consistently with GD, converges only if $\eta < \frac{2}{\lambda}$. In the \textbf{quartic case}, this ODE implies that
\[
df(X_t) = \Bigl(-X_t^6 + \tfrac{3\eta}{2}X_t^8\Bigr)\,dt,
\]
which matches the discrete-time loss expansion
\[
\frac{f(x_{k+1})-f(x_k)}{\eta}
= -x_k^6 + \tfrac{3}{2}\eta x_k^8 + O(\eta^2).
\]

Crucially, this model predicts that for sufficiently large $|X_t|$ (roughly when $\eta X_t^2 \gtrsim 1$), the drift becomes \emph{repulsive} and the loss can increase, mirroring the fact that on $(L_0,L_1)$-smooth landscapes admissible stepsizes must shrink with the local gradient scale, rather than being globally constant or deterministic.

\paragraph{Conclusion:}
This analysis suggests that a higher order of a continuous-time model does not necessarily translate into it better modeling the discrete-time dynamics, not even in the simplest cases, and even less in the $(L_0,L_1)$-smoothness setting. In particular, we find that appropriate \textbf{first}-order SDEs are more faithful than both the classic first- \textbf{and} second-order models when it comes to accurately capturing the stability of the optimizers.

\subsection{Results Derived via Our SDEs} \label{sec:SecondOrder}

In this subsection, we report the convergence bounds for newly derived models of DCSGD and DSignSGD.  Compared to standard first- and second-order models, our models reveal the interaction between learning rate schedules, loss landscape, batch noise, and compression in a way that is consistent with the discrete dynamics of known cases in the literature. Before presenting these results, Table~\ref{tab:compare_1st_2nd} summarizes how the constraint on $\eta\eta_t$ changes when moving from leveraging the classic SDEs to ours. The \textcolor{orange}{orange color} indicates terms that \emph{only} appear due to the use of our SDEs.

\begin{theorem}
    \label{thm:DCSGD_2}
   {\textbf{(DCSGD, unbiased compression, affine variance)}} Let $f$ be $(L_0,L_1)$-smooth, and each client have $(\sigma_{0,i}^2, \sigma_{1,i}^2)$-variance. Define $\overline{\sigma_0^2}\eqdef \frac{1}{N} \sum_{i=1}^{N} \sigma_{0,i}^2$, $\overline{\sigma_1^2}\eqdef \frac{1}{N} \sum_{i=1}^{N} \sigma_{1,i}^2$,  $\overline{\sigma_0^2 \omega}\eqdef \frac{1}{N} \sum_{i=1}^{N} \sigma_{i,0}^2 \omega_i$, and $\overline{\sigma_1^2 \omega}\eqdef \frac{1}{N} \sum_{i=1}^{N} \sigma_{i,1}^2 \omega_i$. For $\epsilon \in (0,1)$, assume $\eta \eta_t$ is smaller than
{\small\begin{equation}\label{eq:StepDCSGD_UV}
    \frac{2 \epsilon}{ \left( L_0 + L_1 \E\left[\lVert \nabla f(X_t) \rVert_2 \right] \right)\left( \textcolor{orange}{\mathbf{1}} + \frac{\overline{\omega} + d (\overline{\sigma_1^2 \omega} + \overline{\sigma_1^2})}{N} \right) +  \frac{  L_1 d\left(\overline{\sigma_0^2}  +  \overline{ \sigma_0^2 \omega } \right)}{N}   }.
\end{equation}}
Then, for a random time $\Hat{t}$ (indep. of $X_t$) with distribution $\frac{\eta_t}{\phi^{(1)}_t}$, we have that $\E \left[\lVert \nabla f(X_{\Hat{t}}) \rVert_2^2 \right]$ is smaller than
{\small\begin{equation}
   \frac{1}{(1-\epsilon) \phi^{(1)}_t} \left( S_0  + \phi^{(2)}_t \frac{\eta (L_0+L_1)d(\overline{\sigma_0^2} + \overline{\sigma_0^2 \omega})}{2 N}\right) \overset{t \rightarrow \infty}{\rightarrow} 0.
\end{equation}}
\end{theorem}

\textbf{Takeaways.}
The key qualitative difference from Thm.~\ref{thm:DCSGD_1} is the \emph{baseline term} $1$ inside the stability constraint
(Eq.~\ref{eq:StepDCSGD_UV}), which restores the correct noiseless $L$-smooth restriction:
setting $\sigma_{0,i}=\sigma_{1,i}=\omega_i=L_1=0$ yields $\eta\eta_t < 2/L_0$. More broadly: \textbf{(i)} The constraint quantifies how \textbf{compression} and \textbf{affine variance} amplify the effective ``adaptivity pressure'': the more aggressive the compression (larger $\overline{\omega}$) and/or the stronger the variance growth ($\overline{\sigma_1^2}$), the more the stepsize must be normalized. \textbf{(ii)} Under $(L_0,L_1)$-smoothness, stability is inherently \textbf{geometry dependent} (through $\E\|\nabla f(X_t)\|_2$), reflecting that there may be no universal constant stepsize guaranteeing stability uniformly over the trajectory. \textbf{(iii)} Although Eq.~\ref{eq:StepDCSGD_UV} is a sufficient condition rather than a closed-form tuning rule, it plays the same conceptual role
as the classical $\eta<2/L$ condition: it delineates a stability region and reveals how the right amount of normalization necessary to ensure convergence is dictated jointly by the compression rate, the variance structure of the noise, and the geometry of the landscape. In this sense, this result offers practitioners concrete guidance on when and how to stabilize DCSGD in challenging regimes. \textbf{(iv)} The insights derived are in accordance with some known in the literature: $1.$ In the noiseless setup $\overline{\sigma_0} = \overline{\sigma_1} = 0$, normalizing the update step naturally emerges as a condition for convergence, in accordance with \citep{gorbunov2024methods}; $2.$ When $L_1 \overline{\sigma_1^2} > 0$, stronger adaptivity is required, in line with insights derived from analyses of Adagrad \citep{wang2023convergence} and NSGD \citep{chen2023generalized}. \textbf{(v)} \textbf{Novelty.} To the best of our knowledge, Eq.~\ref{eq:StepDCSGD_UV} is the first stability region that simultaneously captures $(L_0,L_1)$-geometry, \emph{unbiased} gradient compression, and \emph{affine-variance} noise growth for DCSGD; Additionally, existing $(L_0,L_1)$ results typically treat either affine-variance without compression or compression under more restrictive noise assumptions (cf. Table~\ref{tab:comparison}).

{\textbf{DSignSGD, structured noise, unbounded expected value.}}
To provide informative results for the convergence of DSignSGD under heavy-tailed batch noise, we additionally assume structured noise following a student-$t$ distribution: {\small$\nabla f_{\gamma_i}(x) = \nabla f(x) +  \sqrt{\Sigma_i} Z_i$} s.t. {\small$Z_i \sim t_{\nu}(0, I_d)$}, $\nu$ are the degrees of freedom, and \textit{scale matrices}\footnote{These are \textit{not} covariance matrices, but we use the same notation to facilitate comparability.} {\small$\Sigma_i= \diag(\sigma_{1,i}^2, \cdots, \sigma_{d,i}^2)$}. Note that if $\nu=1$, the \textit{expected value} of $Z_i$ is \textit{unbounded}, thus modeling much more pathological noise than simple affine $(\sigma_0^2,\sigma_1^2)$-variance.

\begin{theorem}\label{thm:sign} Let $f$ be $(L_0,L_1)$-smooth, $\Sigma_i \leq \sigma_{\text{max}, i}^2$, $\sigma_{\mathcal{H},1}$ be the harmonic mean of $\{\sigma_{\text{max}, i} \}$, and {\small$K \eqdef  \left(\frac{L_1 d \sigma_{\mathcal{H},1}}{2 N} + \textcolor{orange}{\sqrt{d}(L_0+L_1) M_{\nu}}\right)$}. Then, for $\eta \eta_t <\frac{\ell_{\nu}}{K} $ and random time $\Tilde{t}$ (indep. of $X_t$) with distribution $\frac{\eta_t \ell_{\nu} - \eta_t^2 \eta K}{\phi^{(1)}_t \ell_{\nu}  - \phi^{(2)}_t \eta K}$, we have that $\E \left[ \lVert \nabla f\left(X_{\Tilde{t}}\right)\rVert_2^2 \right]$ is smaller than
 {\small\begin{equation}
     \frac{\sigma_{\mathcal{H},1}}{\phi^{(1)}_t \ell_{\nu}  - \phi^{(2)}_t \eta K} \left[ S_0 + \phi^{(2)}_t \eta (L_0+L_1) \left[ \frac{d}{2 N} + \frac{M_{\nu}\sqrt{d}}{\sigma_{\mathcal{H},1} } \right] \right]
\end{equation}}
and converges to $0$ as $t \rightarrow \infty$.

\end{theorem}

\textbf{Takeaways.}
Heavier tails (smaller $\nu$) and larger noise scales (larger $\sigma_{\max,i}$) tighten the admissible stepsize through $K$.
The crucial difference from the first-order SDE analysis (Theorem~\ref{thm:sign_1}) is the appearance of an additional baseline term: Even when $\sigma_{\max,i}=0$, the bound still enforces a nontrivial restriction
$\eta\eta_t \lesssim 1/\sqrt{d}(L_0+L_1)$.
Finally, unlike DCSGD (see Eq.~\ref{eq:StepDCSGD_UV}), DSignSGD does \emph{not} require the stepsize to scale inversely with $\|\nabla f\|$:
The sign update already induces an implicit normalization, which is consistent with the stability behavior in Fig.~\ref{fig:InsightValidation_MLP}.

\vspace{-0.2cm}

\section{Conclusion}

\vspace{-0.2cm}

\paragraph{\textbf{Conceptual contribution: stability-faithful continuous-time models.}}
We developed continuous-time models for distributed stochastic optimizers under $(L_0,L_1)$-smoothness that are explicitly \emph{stability-faithful}.
A key conceptual finding is that, in this regime, standard SDEs used in the literature can be qualitatively misleading: both the first- and second-order SDEs may fail to reproduce discrete-time stepsize/stability restrictions, and can even predict benign or accelerated behavior where the underlying algorithm is unstable.
Our viewpoint is to prioritize \emph{stability fidelity}, by modifying the curvature-dependent drift so that the dynamics of our novel models recover the correct stability thresholds.

\paragraph{\textbf{Technical contribution: new SDE surrogates and convergence guarantees beyond classical assumptions.}}
Technically, this paper provides (to the best of our knowledge) the \emph{first} SDE-based treatment of stochastic optimization under $(L_0,L_1)$-smoothness. Using our stability-corrected SDEs, we derive bounds under \emph{unbiased compression} and general noise models that go beyond bounded variance, including \emph{affine-variance growth} and \emph{heavy-tailed} regimes. The resulting stepsize/stability prescriptions make the dependence on geometry, compression strength, dimension, and the noise structure explicit, clarifying precisely which mechanisms shrink the stable region and why classical models can miss this interaction under $(L_0,L_1)$-smoothness.

\paragraph{\textbf{Practical implications and outlook.}}
From a practical standpoint, our analysis yields a concrete stability message: for DCSGD, \emph{noise- and compression-dependent normalization} emerges as the natural mechanism to remain stable, and our novel results quantify how the appropriate normalization level is jointly determined by the compression rate, the noise-growth parameters, and $(L_0,L_1)$-geometry. In contrast, DSignSGD’s elementwise normalization makes it robust to heavy-tailed noise, enabling convergence under standard diminishing stepsize schedules even in regimes where the noise can have unbounded expectation. Experiments on DNNs corroborate these qualitative predictions.

More broadly, our contribution is methodological: it shows how to build SDE surrogates that are faithful to the stability restrictions of the discrete optimizer in a regime where classical SDEs can be qualitatively misleading. This matters because SDEs are already used to study convergence bounds, learning-rate schedules, batch-size control, and scaling laws~\citep{orvieto2019continuous,zhao2022batch,jastrzkebski2017three,compagnoni2024sde,compagnoni2025unbiased,compagnoni2025adaptive}. If the continuous-time surrogate predicts convergence while missing the discrete-time stability region, then prescriptions derived from it may push the actual optimizer outside its stable regime. Our corrected models open the door to revisiting these questions under $(L_0,L_1)$-smoothness with the relevant stability restrictions built in. For instance, batch-size schemes based on classical SDE control problems~\citep{zhao2022batch} should incorporate constraints preventing ratios such as $\eta/B(t)$ from leaving the stable regime; learning-rate schedules derived from continuous-time convergence bounds on suboptimality or gradient norm~\citep{orvieto2019continuous,compagnoni2024sde,compagnoni2025unbiased,compagnoni2025adaptive} should be interpreted together with admissible stepsize regions, since otherwise they may misinform how to set learning-rate schedules; and SDE-based scaling laws, such as the $\eta/B$ linear scaling rule~\citep{jastrzkebski2017three} and square-root scaling rules for adaptive methods~\citep{compagnoni2025adaptive}, should be viewed as valid only within the stability limits imposed by the corrected dynamics, for instance in terms of admissible batch-size and stepsize ranges.

Finally, our contribution is intentionally foundational: rather than proposing new optimizers, we provide a unified, stability-faithful SDE framework for understanding how noise, compression, and normalization interact under relaxed smoothness assumptions, and we expect it to support future extensions (e.g., heterogeneous clients, error-feedback and biased compressors) and analogous analyses of modern methods that incorporate implicit second-order structure, such as momentum/NAG and adaptive optimizers (AdamW, RMSprop).

\paragraph{\textbf{Limitations.}}
Our analysis focuses on homogeneous clients, server-aggregated distributed optimization, and unbiased or signed compression. Fully decentralized topologies, heterogeneous client distributions, error-feedback mechanisms, and general biased compressors are outside the scope of the present work. Moreover, our guarantees are stated for the continuous-time surrogate models, and our experiments are intended as mechanism validation rather than competitive benchmarking. We provide a more detailed discussion in App.~\ref{sec:limitations}.

\section*{Acknowledgments}
Enea Monzio Compagnoni, Rustem Islamov, and Aurelien Lucchi acknowledge the financial support of the Swiss
National Foundation, SNF grant No 207392.

\clearpage
\bibliography{references}

@inproceedings{
compagnoni2026adaptive,
title={Adaptive Methods Are Preferable in High Privacy Settings: An {SDE} Perspective},
author={Enea Monzio Compagnoni and Alessandro Stanghellini and Rustem Islamov and Aurelien Lucchi and Anastasia Koloskova},
booktitle={The Fourteenth International Conference on Learning Representations},
year={2026},
}

@article{chezhegov2025convergence,
  title={Convergence of Clipped-SGD for Convex $({L}_0, {L}_1) $-Smooth Optimization with Heavy-Tailed Noise},
  author={Chezhegov, Savelii and Beznosikov, Aleksandr and Horv{\'a}th, Samuel and Gorbunov, Eduard},
  journal={arXiv preprint arXiv:2505.20817},
  year={2025}
}

@article{gaash2025convergence,
  title={Convergence of Clipped SGD on Convex $({L}_0, {L}_1)$-Smooth Functions},
  author={Gaash, Ofir and Levy, Kfir Yehuda and Carmon, Yair},
  journal={arXiv preprint arXiv:2502.16492},
  year={2025}
}

@article{schmidt2013fast,
  title={Fast convergence of stochastic gradient descent under a strong growth condition},
  author={Schmidt, Mark and Le Roux, Nicolas},
  journal={arXiv preprint arXiv:1308.6370},
  year={2013}
}

@inproceedings{reisizadeh2025variance,
  title={Variance-reduced clipping for non-convex optimization},
  author={Reisizadeh, Amirhossein and Li, Haochuan and Das, Subhro and Jadbabaie, Ali},
  booktitle={ICASSP 2025-2025 IEEE International Conference on Acoustics, Speech and Signal Processing (ICASSP)},
  pages={1--5},
  year={2025},
  organization={IEEE}
}

@article{richtarik2021ef21,
  title={{EF21}: A new, simpler, theoretically better, and practically faster error feedback},
  author={Richt{\'a}rik, Peter and Sokolov, Igor and Fatkhullin, Ilyas},
  journal={Advances in Neural Information Processing Systems},
  volume={34},
  pages={4384--4396},
  year={2021}
}

@article{kornilov2025sign,
  title={Sign Operator for Coping with Heavy-Tailed Noise in Non-Convex Optimization: High Probability Bounds Under $(L_0,L_1)$-Smoothness},
  author={Kornilov, Nikita and Zmushko, Philip and Semenov, Andrei and Ikonnikov, Mark and Gasnikov, Alexander and Beznosikov, Alexander},
  journal={arXiv preprint arXiv:2502.07923},
  year={2025}
}

@article{li2023convex,
  title={Convex and non-convex optimization under generalized smoothness},
  author={Li, Haochuan and Qian, Jian and Tian, Yi and Rakhlin, Alexander and Jadbabaie, Ali},
  journal={Advances in Neural Information Processing Systems},
  volume={36},
  pages={40238--40271},
  year={2023}
}

@inproceedings{koloskova2023revisiting,
  title={Revisiting gradient clipping: Stochastic bias and tight convergence guarantees},
  author={Koloskova, Anastasia and Hendrikx, Hadrien and Stich, Sebastian U},
  booktitle={International Conference on Machine Learning},
  pages={17343--17363},
  year={2023},
  organization={PMLR}
}

@article{vankov2024optimizing,
  title={Optimizing $({L}_0, {L}_1) $-Smooth Functions by Gradient Methods},
  author={Vankov, Daniil and Rodomanov, Anton and Nedich, Angelia and Sankar, Lalitha and Stich, Sebastian U},
  journal={International Conference on Learning Representations},
  year={2025}
}

@article{gorbunov2024methods,
  title={Methods for convex $({L}_0, {L}_1) $-smooth optimization: Clipping, acceleration, and adaptivity},
  author={Gorbunov, Eduard and Tupitsa, Nazarii and Choudhury, Sayantan and Aliev, Alen and Richt{\'a}rik, Peter and Horv{\'a}th, Samuel and Tak{\'a}{\v{c}}, Martin},
  journal={International Conference on Learning Representations},
  year={2025}
}

@inproceedings{chen2023generalized,
  title={Generalized-smooth nonconvex optimization is as efficient as smooth nonconvex optimization},
  author={Chen, Ziyi and Zhou, Yi and Liang, Yingbin and Lu, Zhaosong},
  booktitle={International Conference on Machine Learning},
  pages={5396--5427},
  year={2023},
  organization={PMLR}
}

@inproceedings{hubler2024parameter,
  title={Parameter-agnostic optimization under relaxed smoothness},
  author={H{\"u}bler, Florian and Yang, Junchi and Li, Xiang and He, Niao},
  booktitle={International Conference on Artificial Intelligence and Statistics},
  pages={4861--4869},
  year={2024},
  organization={PMLR}
}

@article{li2024convergence,
  title={Convergence of {Adam} under relaxed assumptions},
  author={Li, Haochuan and Rakhlin, Alexander and Jadbabaie, Ali},
  journal={Advances in Neural Information Processing Systems},
  volume={36},
  year={2023},
  pages={52166--52196}
}

@inproceedings{wang2023convergence,
  title={Convergence of adagrad for non-convex objectives: Simple proofs and relaxed assumptions},
  author={Wang, Bohan and Zhang, Huishuai and Ma, Zhiming and Chen, Wei},
  booktitle={The Thirty Sixth Annual Conference on Learning Theory},
  pages={161--190},
  year={2023},
  organization={PMLR}
}

@article{orvieto2019shadowing,
  title={Shadowing properties of optimization algorithms},
  author={Orvieto, Antonio and Lucchi, Aurelien},
  journal={Advances in Neural Information Processing Systems},
  volume={32},
  year={2019}
}

@inproceedings{faw2023beyond,
  title={Beyond Uniform Smoothness: A Stopped Analysis of Adaptive {SGD}},
  author={Faw, Matthew and Rout, Litu and Caramanis, Constantine and Shakkottai, Sanjay},
  booktitle={The Thirty Sixth Annual Conference on Learning Theory},
  pages={89--160},
  year={2023},
  organization={PMLR},
  volume={195},
  series={Proceedings of Machine Learning Research}
}

@article{crawshaw2022robustness,
  title={Robustness to unbounded smoothness of generalized sign{SGD}},
  author={Crawshaw, Michael and Liu, Mingrui and Orabona, Francesco and Zhang, Wei and Zhuang, Zhenxun},
  journal={Advances in Neural Information Processing Systems},
  volume={35},
  pages={9955--9968},
  year={2022}
}

@article{zhao2021convergence,
  title={On the convergence and improvement of stochastic normalized gradient descent},
  author={Zhao, Shen-Yi and Xie, Yin-Peng and Li, Wu-Jun},
  journal={Science China Information Sciences},
  volume={64},
  pages={1--13},
  year={2021},
  publisher={Springer}
}

@inproceedings{zhang2020improved,
  title={Improved Analysis of Clipping Algorithms for Non-convex Optimization},
  author={Bohang Zhang and Jikai Jin and Cong Fang and Liwei Wang},
  booktitle={Advances in Neural Information Processing Systems},
  year={2020},
}

@inproceedings{raginsky2012continuous,
  title={Continuous-time stochastic mirror descent on a network: Variance reduction, consensus, convergence},
  author={Raginsky, Maxim and Bouvrie, Jake},
  booktitle={2012 IEEE 51st IEEE Conference on Decision and Control (CDC)},
  pages={6793--6800},
  year={2012},
  organization={IEEE}
}

@article{mertikopoulos2018convergence,
  title={On the convergence of gradient-like flows with noisy gradient input},
  author={Mertikopoulos, Panayotis and Staudigl, Mathias},
  journal={SIAM Journal on Optimization},
  volume={28},
  number={1},
  pages={163--197},
  year={2018},
  publisher={SIAM}
}

@article{zhang2020adaptive,
  title={Why are adaptive methods good for attention models?},
  author={Zhang, Jingzhao and Karimireddy, Sai Praneeth and Veit, Andreas and Kim, Seungyeon and Reddi, Sashank and Kumar, Sanjiv and Sra, Suvrit},
  journal={Advances in Neural Information Processing Systems},
  year={2020}
}

@article{kunstner2024heavy,
  title={Heavy-Tailed Class Imbalance and Why Adam Outperforms Gradient Descent on Language Models},
  author={Kunstner, Frederik and Yadav, Robin and Milligan, Alan and Schmidt, Mark and Bietti, Alberto},
  journal={arXiv preprint arXiv:2402.19449},
  year={2024}
}

@inproceedings{
compagnoni2025unbiased,
title={Unbiased and Sign Compression in Distributed Learning: Comparing Noise Resilience via {SDE}s},
author={Enea Monzio Compagnoni and Rustem Islamov and Frank Norbert Proske and Aurelien Lucchi},
booktitle={The 28th International Conference on Artificial Intelligence and Statistics},
year={2025},
volume={258},
pages={4087--4095},
series={Proceedings of Machine Learning Research},
publisher={PMLR}
}

@inproceedings{
compagnoni2025adaptive,
title={Adaptive Methods through the Lens of {SDE}s: Theoretical Insights on the Role of Noise},
author={Enea Monzio Compagnoni and Tianlin Liu and Rustem Islamov and Frank Norbert Proske and Antonio Orvieto and Aurelien Lucchi},
booktitle={The Thirteenth International Conference on Learning Representations},
year={2025},
}

@article{mishchenko2024distributed,
  title={Distributed learning with compressed gradient differences},
  author={Mishchenko, Konstantin and Gorbunov, Eduard and Tak{\'a}{\v{c}}, Martin and Richt{\'a}rik, Peter},
  journal={Optimization Methods and Software},
  pages={1--16},
  year={2024},
  publisher={Taylor \& Francis}
}

@inproceedings{gorbunov2020unified,
  title={A unified theory of SGD: Variance reduction, sampling, quantization and coordinate descent},
  author={Gorbunov, Eduard and Hanzely, Filip and Richt{\'a}rik, Peter},
  booktitle={International Conference on Artificial Intelligence and Statistics},
  pages={680--690},
  year={2020},
  organization={PMLR}
}

@article{khirirat2024error,
  title={Error Feedback under $(L_0,L_1)$-Smoothness: Normalization and Momentum},
  author={Khirirat, Sarit and Sadiev, Abdurakhmon and Riabinin, Artem and Gorbunov, Eduard and Richt{\'a}rik, Peter},
  journal={arXiv preprint arXiv:2410.16871},
  year={2024}
}

@article{stich2018sparsified,
  title={Sparsified SGD with memory},
  author={Sebastian U. Stich and Jean-Baptiste Cordonnier and Martin Jaggi},
  journal={Advances in Neural Information Processing Systems},
  volume={31},
  year={2018}
}

@article{alistarh2017qsgd,
  title={QSGD: Communication-efficient SGD via gradient quantization and encoding},
  author={Alistarh, Dan and Grubic, Demjan and Li, Jerry and Tomioka, Ryota and Vojnovic, Milan},
  journal={Advances in Neural Information Processing Systems},
  volume={30},
  year={2017}
}

@article{gorbunov2021near,
  title={High Probability Complexity Bounds for Non-Smooth Stochastic Optimization with Heavy-Tailed Noise},
  author={Gorbunov, Eduard and Danilova, Marina and Shibaev, Innokentiy and Dvurechensky, Pavel and Gasnikov, Alexander},
  journal={arXiv preprint arXiv:2106.05958},
  year={2021}
}

@inproceedings{smith2021origin,
  title={On the Origin of Implicit Regularization in Stochastic Gradient Descent},
  author={Smith, Samuel L. and Dherin, Benoit and Barrett, David G. T. and De, Soham},
  booktitle={International Conference on Learning Representations},
  year={2021}
}

@InProceedings{bernstein2018signsgd,
  title = 	 {sign{SGD}: Compressed Optimisation for Non-Convex Problems},
  author =       {Bernstein, Jeremy and Wang, Yu-Xiang and Azizzadenesheli, Kamyar and Anandkumar, Animashree},
  booktitle = 	 {Proceedings of the 35th International Conference on Machine Learning},
  year={2018},
  volume={80},
  pages={560--569},
  series={Proceedings of Machine Learning Research},
  publisher={PMLR}
}

@InProceedings{safaryan2021signsgd,
  title = 	 {Stochastic Sign Descent Methods: New Algorithms and Better Theory},
  author =       {Safaryan, Mher and Richt{\'a}rik, Peter},
  booktitle = 	 {Proceedings of the 38th International Conference on Machine Learning},
  year = 	 {2021},
  volume={139},
  pages={9224--9234},
  series={Proceedings of Machine Learning Research},
  publisher={PMLR}
}

@misc{zhang2024convergence,
      title={Convergence Guarantees for RMSProp and Adam in Generalized-smooth Non-convex Optimization with Affine Noise Variance}, 
      author={Qi Zhang and Yi Zhou and Shaofeng Zou},
      year={2024},
      journal={arXiv preprint arXiv:2404.01436}
}

@inproceedings{
zhang2019gradient,
title={Why Gradient Clipping Accelerates Training: A Theoretical Justification for Adaptivity},
author={Jingzhao Zhang and Tianxing He and Suvrit Sra and Ali Jadbabaie},
booktitle={International Conference on Learning Representations},
year={2020},
}

@article{csimcsekli2019heavy,
  title={On the heavy-tailed theory of stochastic gradient descent for deep neural networks},
  author={{\c{S}}im{\c{s}}ekli, Umut and G{\"u}rb{\"u}zbalaban, Mert and Nguyen, Thanh Huy and Richard, Ga{\"e}l and Sagun, Levent},
  journal={arXiv preprint arXiv:1912.00018},
  year={2019}
}

@inproceedings{balles2018dissecting,
  title={Dissecting {Adam}: The sign, magnitude and variance of stochastic gradients},
  author={Balles, Lukas and Hennig, Philipp},
  booktitle={International Conference on Machine Learning},
  pages={404--413},
  year={2018},
  organization={PMLR}
}

@inproceedings{compagnoni2023sde,
  title={An {SDE} for Modeling {SAM}: Theory and Insights},
  author={Compagnoni, Enea Monzio and Biggio, Luca and Orvieto, Antonio and Proske, Frank Norbert and Kersting, Hans and Lucchi, Aurelien},
  booktitle={International Conference on Machine Learning},
  pages={25209--25253},
  year={2023},
  organization={PMLR}
}

@inproceedings{compagnoni2024sde,
  title={SDEs for Minimax Optimization},
  author={Compagnoni, Enea Monzio and Orvieto, Antonio and Kersting, Hans and Proske, Frank and Lucchi, Aurelien},
  booktitle={International Conference on Artificial Intelligence and Statistics},
  pages={4834--4842},
  year={2024},
  organization={PMLR}
}

@article{mil1986weak,
  title={Weak approximation of solutions of systems of stochastic differential equations},
  author={Milshtein, G. N.},
  journal={Theory of Probability \& Its Applications},
  volume={30},
  number={4},
  pages={750--766},
  year={1986},
  publisher={SIAM}
}

@article{zhou2020towards,
  title={Towards theoretically understanding why {SGD} generalizes better than {Adam} in deep learning},
  author={Zhou, Pan and Feng, Jiashi and Ma, Chao and Xiong, Caiming and Hoi, Steven Chu Hong and E, Weinan},
  journal={Advances in Neural Information Processing Systems},
  volume={33},
  pages={21285--21296},
  year={2020}
}

@InProceedings{li2017stochastic,
  title = 	 {Stochastic Modified Equations and Adaptive Stochastic Gradient Algorithms},
  author =       {Qianxiao Li and Cheng Tai and Weinan E},
  booktitle = 	 {Proceedings of the 34th International Conference on Machine Learning},
  pages = 	 {2101--2110},
  year = 	 {2017},
  editor = 	 {Precup, Doina and Teh, Yee Whye},
  volume = 	 {70},
  series = 	 {Proceedings of Machine Learning Research},
  month = 	 {06--11 Aug},
  publisher =    {PMLR},
}

@article{li2019stochastic,
  title={Stochastic modified equations and dynamics of stochastic gradient algorithms {I}: Mathematical Foundations},
  author={Li, Qianxiao and Tai, Cheng and E, Weinan},
  journal={The Journal of Machine Learning Research},
  volume={20},
  number={1},
  pages={1474--1520},
  year={2019},
  publisher={JMLR.org}
}

@InProceedings{simsekli2019tailindex,
  title = 	 {A Tail-Index Analysis of Stochastic Gradient Noise in Deep Neural Networks},
  author =       {Simsekli, Umut and Sagun, Levent and Gurbuzbalaban, Mert},
  booktitle = 	 {International Conference on Machine Learning},
  year = 	 {2019},
  volume={97},
  pages={5827--5837},
  series={Proceedings of Machine Learning Research},
  publisher={PMLR}
}

@book{kushner2003stochastic,
  title={Stochastic approximation and recursive algorithms and applications},
  author={Kushner, Harold and Yin, G George},
  volume={35},
  year={2003},
  publisher={Springer Science \& Business Media}
}

@book{helmke1994optimization,
    title = {Optimization and Dynamical Systems},
    author={Helmke, Uwe and Moore, John B},
    edition = {1st},
    publisher = {Springer London},
    year = {1994}
}

@inproceedings{Su2014nesterov,
 author = {Su, Weijie and Boyd, Stephen and Candés, Emmanuel},
 booktitle = {Advances in Neural Information Processing Systems},
 title = {A Differential Equation for Modeling {Nesterov's} Accelerated Gradient Method: Theory and Insights},
 year = {2014}
}

@inproceedings{
marshall2025to,
title={To Clip or not to Clip: the Dynamics of {SGD} with Gradient Clipping in High-Dimensions},
author={Noah Marshall and Ke Liang Xiao and Atish Agarwala and Elliot Paquette},
booktitle={The Thirteenth International Conference on Learning Representations},
year={2025}
}

@inproceedings{
xiao2025exact,
title={Exact risk curves of sign{SGD} in High-Dimensions: quantifying preconditioning and noise-compression effects},
author={Ke Liang Xiao and Noah Marshall and Atish Agarwala and Elliot Paquette},
booktitle={Forty-second International Conference on Machine Learning},
year={2025},
volume={267},
pages={68391--68439},
series={Proceedings of Machine Learning Research},
publisher={PMLR}
}

@inproceedings{
Malladi2022AdamSDE,
title={On the {SDE}s and Scaling Rules for Adaptive Gradient Algorithms},
author={Sadhika Malladi and Kaifeng Lyu and Abhishek Panigrahi and Sanjeev Arora},
booktitle={Advances in Neural Information Processing Systems},
editor={Alice H. Oh and Alekh Agarwal and Danielle Belgrave and Kyunghyun Cho},
year={2022},
}

@inproceedings{zhao2022batch,
    title={Batch size selection by stochastic optimal control},
    author={Jim Zhao and Aurelien Lucchi and Frank Norbert Proske and Antonio Orvieto and Hans Kersting},
    booktitle={Has it Trained Yet? NeurIPS 2022 Workshop},
    year={2022}
}

@inproceedings{jastrzkebski2017three,
  title={Three Factors Influencing Minima in SGD},
  author={Jastrzebski, Stanis{\l}aw and Kenton, Zachary and Arpit, Devansh and Ballas, Nicolas and Fischer, Asja and Bengio, Yoshua and Storkey, Amos},
  year={2018},
  booktitle={Artificial Neural Networks and Machine Learning -- ICANN 2018},
  series={Lecture Notes in Computer Science},
  volume={11141},
  pages={392--402},
  publisher={Springer International Publishing},
  doi={10.1007/978-3-030-01424-7_39}
}

@article{orvieto2019continuous,
  title={Continuous-time models for stochastic optimization algorithms},
  author={Orvieto, Antonio and Lucchi, Aurelien},
  journal={Advances in Neural Information Processing Systems},
  volume={32},
  year={2019}
}

@inproceedings{smith2020sde,
  title={On the Generalization Benefit of Noise in Stochastic Gradient Descent},
  author={Smith, Samuel and Elsen, Erich and De, Soham},
  booktitle={International Conference on Machine Learning},
  year={2020},
  volume={119},
  pages={9058--9067},
  series={Proceedings of Machine Learning Research},
  publisher={PMLR}
}

@inproceedings{zhu2019anisotropic,
  title={The anisotropic noise in stochastic gradient descent: Its behavior of escaping from sharp minima and regularization effects},
  author={Zhu, Zhanxing and Wu, Jingfeng and Yu, Bing and Wu, Lei and Ma, Jinwen},
  year={2019},
  booktitle={Proceedings of the 36th International Conference on Machine Learning},
  series={Proceedings of Machine Learning Research},
  volume={97},
  pages={7654--7663},
  publisher={PMLR}
}

@InProceedings{Xie2021,
  title = 	 {Positive-Negative Momentum: Manipulating Stochastic Gradient Noise to Improve Generalization},
  author =       {Xie, Zeke and Yuan, Li and Zhu, Zhanxing and Sugiyama, Masashi},
  booktitle = 	 {Proceedings of the 38th International Conference on Machine Learning},
  pages = 	 {11448--11458},
  year = 	 {2021},
  volume = 	 {139},
  series = 	 {Proceedings of Machine Learning Research},
  month = 	 {18--24 Jul},
  publisher =    {PMLR},
}

@article{stephan2017stochastic,
  title={Stochastic gradient descent as approximate Bayesian inference},
  author={Mandt, Stephan and Hoffman, Matthew D and Blei, David M},
  journal={Journal of Machine Learning Research},
  volume={18},
  number={134},
  pages={1--35},
  year={2017}
}

@InProceedings{ahn2012bayesian,
  author =    {Sungjin Ahn and Anoop Korattikara and Max Welling},
  title =     {Bayesian Posterior Sampling via Stochastic Gradient Fisher Scoring},
  booktitle = {Proceedings of the 29th International Conference on Machine Learning (ICML-12)},
  series =    {ICML '12},
  year =      {2012},
  editor =    {John Langford and Joelle Pineau},
  location =  {Edinburgh, Scotland, GB},
  isbn =      {978-1-4503-1285-1},
  month =     {July},
  publisher = {Omnipress},
  address =   {New York, NY, USA},
  pages=      {1591--1598},
}

@inproceedings{mandt2016variational,
  title={A variational analysis of stochastic gradient algorithms},
  author={Mandt, Stephan and Hoffman, Matthew and Blei, David},
  booktitle={International Conference on Machine Learning},
  pages={354--363},
  year={2016},
  organization={PMLR}
}

@inproceedings{wu2020noisy,
  title={On the noisy gradient descent that generalizes as {SGD}},
  author={Wu, Jingfeng and Hu, Wenqing and Xiong, Haoyi and Huan, Jun and Braverman, Vladimir and Zhu, Zhanxing},
  booktitle={International Conference on Machine Learning},
  pages={10367--10376},
  year={2020},
  organization={PMLR}
}

@inproceedings{chen2014stochastic,
  title={Stochastic Gradient Hamiltonian Monte Carlo},
  author={Chen, Tianqi and Fox, Emily and Guestrin, Carlos},
  booktitle={International Conference on Machine Learning},
  pages={1683--1691},
  year={2014},
  organization={PMLR}
}

@inproceedings{
luo2024explicit,
title={Explicit Eigenvalue Regularization Improves Sharpness-Aware Minimization},
author={Haocheng Luo and Tuan Truong and Tung Pham and Mehrtash Harandi and Dinh Phung and Trung Le},
booktitle={The Thirty-eighth Annual Conference on Neural Information Processing Systems},
year={2024},
}

@article{tovmasyan2025revisiting,
  title={Revisiting Stochastic Proximal Point Methods: Generalized Smoothness and Similarity},
  author={Tovmasyan, Zhirayr and Malinovsky, Grigory and Condat, Laurent and Richt{\'a}rik, Peter},
  journal={arXiv preprint arXiv:2502.03401},
  year={2025}
}

@article{bubeck2015convex,
  title={Convex optimization: Algorithms and complexity},
  author={Bubeck, S{\'e}bastien},
  journal={Foundations and Trends{\textregistered} in Machine Learning},
  volume={8},
  number={3-4},
  pages={231--357},
  year={2015},
  publisher={Now Publishers, Inc.}
}

@article{tyurin2025near,
  title={Near-Optimal Convergence of Accelerated Gradient Methods under Generalized and $(L_0,L_1)$-Smoothness},
  author={Tyurin, Alexander},
  journal={arXiv preprint arXiv:2508.06884},
  year={2025}
}

@inproceedings{pethick2025generalized,
  title={Generalized Gradient Norm Clipping \& Non-Euclidean $(L_0,L_1)$-Smoothness},
  author={Pethick, Thomas and Xie, Wanyun and Erdogan, Mete and Antonakopoulos, Kimon and Silveti-Falls, Tony and Cevher, Volkan},
  booktitle={Advances in Neural Information Processing Systems},
  year={2025}
}

@inproceedings{demidovich2024methods,
  title={Methods with local steps and random reshuffling for generally smooth non-convex federated optimization},
  author={Demidovich, Yury and Ostroukhov, Petr and Malinovsky, Grigory and Horv{\'a}th, Samuel and Tak{\'a}{\v{c}}, Martin and Richt{\'a}rik, Peter and Gorbunov, Eduard},
  booktitle={International Conference on Learning Representations},
  year={2025}
}

@article{cao2025efficiently,
  title={Efficiently Escaping Saddle Points under Generalized Smoothness via Self-Bounding Regularity},
  author={Cao, Daniel Yiming and Chen, August Y and Sridharan, Karthik and Tang, Benjamin},
  journal={arXiv preprint arXiv:2503.04712},
  year={2025}
}

@article{tyurin2024toward,
  title={Toward a unified theory of gradient descent under generalized smoothness},
  author={Tyurin, Alexander},
  journal={arXiv preprint arXiv:2412.11773},
  year={2024}
}

@article{yu2025mirror,
  title={Mirror descent under generalized smoothness},
  author={Yu, Dingzhi and Jiang, Wei and Wan, Yuanyu and Zhang, Lijun},
  journal={arXiv preprint arXiv:2502.00753},
  year={2025}
}

@inproceedings{oikonomidis2025nonlinearly,
  title={Nonlinearly Preconditioned Gradient Methods under Generalized Smoothness},
  author={Oikonomidis, Konstantinos and Quan, Jan and Laude, Emanuel and Patrinos, Panagiotis},
  booktitle={International Conference on Machine Learning},
  year={2025}
}

@article{yu2025convergence,
  title={Convergence Analysis of Stochastic Accelerated Gradient Methods for Generalized Smooth Optimizations},
  author={Yu, Chenhao and Hong, Yusu and Lin, Junhong},
  journal={arXiv preprint arXiv:2502.11125},
  year={2025}
}

@article{lobanov2025power,
  title={Power of Generalized Smoothness in Stochastic Convex Optimization: First- and Zero-Order Algorithms},
  author={Lobanov, Aleksandr and Gasnikov, Alexander},
  journal={arXiv preprint arXiv:2501.18198},
  year={2025}
}

@article{jiang2025decentralized,
  title={Decentralized Relaxed Smooth Optimization with Gradient Descent Methods},
  author={Jiang, Zhanhong and Balu, Aditya and Sarkar, Soumik},
  journal={arXiv preprint arXiv:2508.08413},
  year={2025}
}

@inproceedings{staib2019escaping,
  title={Escaping saddle points with adaptive gradient methods},
  author={Staib, Matthew and Reddi, Sashank and Kale, Satyen and Kumar, Sanjiv and Sra, Suvrit},
  booktitle={International Conference on Machine Learning},
  pages={5956--5965},
  year={2019},
  organization={PMLR}
}

@inproceedings{ahn2022understanding,
  title={Understanding the unstable convergence of gradient descent},
  author={Ahn, Kwangjun and Zhang, Jingzhao and Sra, Suvrit},
  booktitle={Proceedings of the 39th International Conference on Machine Learning},
  pages={247--257},
  year={2022},
  organization={PMLR},
  volume={162},
  series={Proceedings of Machine Learning Research}
}

@article{lobanov2024linear,
  title={Linear Convergence Rate in Convex Setup is Possible! Gradient Descent Method Variants under $(L_0,L_1)$-Smoothness},
  author={Lobanov, Aleksandr and Gasnikov, Alexander and Gorbunov, Eduard and Tak{\'a}{\v{c}}, Martin},
  journal={arXiv preprint arXiv:2412.17050},
  year={2024}
}

@article{yang2024adaptive,
  title={Adaptive Gradient Normalization and Independent Sampling for (Stochastic) Generalized-Smooth Optimization},
  author={Yang, Yufeng and Tripp, Erin and Sun, Yifan and Zou, Shaofeng and Zhou, Yi},
  journal={arXiv preprint arXiv:2410.14054},
  year={2024}
}
\bibliographystyle{plainnat}

\newpage

\appendix

\section{Additional Related work}
\label{sec:RelatedWorks_App}


\textbf{SDE Approximations in Optimization.} Continuous-time models in the form of differential equations are a well-established tool to study discrete-time optimizers, e.g.~\citep{helmke1994optimization,kushner2003stochastic}. Recent works also derived differential equations to model SGD under heavy-tailed batch noise~\citep{simsekli2019tailindex}, and~\citep{zhou2020towards} derived a Lévy-driven stochastic differential equation to model the non-Gaussianity of the noise. It was~\citep{li2017stochastic} that first introduced a \textit{rigorous} theoretical framework to derive SDEs that faithfully model the stochastic behavior intrinsic to optimization algorithms widely employed in machine learning. Since then, such SDE-based formulations have been applied across several domains, including \emph{stochastic optimal control} for tuning stepsizes~\citep{li2017stochastic,li2019stochastic} and batch sizes~\citep{zhao2022batch}. Notably, SDEs have been instrumental in analyzing \emph{convergence bounds} and \emph{stationary distributions}~\citep{compagnoni2023sde,compagnoni2024sde,compagnoni2025adaptive}, \emph{scaling laws}~\citep{jastrzkebski2017three,compagnoni2025adaptive,compagnoni2025unbiased}, \emph{implicit regularization} effects~\citep{smith2021origin,compagnoni2023sde}, and \textit{implicit preconditioning}~\citep{xiao2025exact,marshall2025to}. We refer the interested reader to~\citep{orvieto2019shadowing,orvieto2019continuous} for a didactic introduction to this topic, especially for how Itô calculus is used to derive these results.

We contribute to this line by highlighting a key gap: both the classic \textit{first}-order and the \textit{second-order} SDEs from the literature can yield conclusions that contradict the discrete-time dynamics of SGD. While this is somewhat expected from a first-order model, it is surprising that a higher-order one also fails, possibly even more. While previous studies did derive second-order SDEs for various optimizers~\citep{li2017stochastic,li2019stochastic,luo2024explicit}, they did not exploit them to obtain theoretical insights and thus overlooked these limitations. To remediate this, we derive new \textit{first}-order SDEs whose drift is modified to recover the relevant stability thresholds of discrete optimizers, and we use them to study the joint effect of compression and general noise in distributed learning.

\begin{table*}[t]
\centering
\caption{Comparison of existing convergence results for stochastic methods applied to $(L_0,L_1)$-smooth problems. All results are derived for non-convex problems, and the bounds are given in expectation unless stated otherwise. All works assume bounded noise or bounded variance unless stated otherwise. Abbreviations: ``HT'' = heavy-tailed noise, ``Affine var.'' = affine variance.}
\label{tab:comparison_App}

\setlength{\tabcolsep}{4pt}        
\renewcommand{\arraystretch}{0.9}  
\footnotesize

\begin{threeparttable}
\resizebox{\textwidth}{!}{%
\begin{tabular}{@{}lccccc@{}}
\toprule[1pt]
\multirow{2}{*}{\textbf{Reference}} &
\multirow{2}{*}{\textbf{Dynamics}} &
\multicolumn{2}{c}{\textbf{Noise}} &
\multirow{2}{*}{\textbf{$(L_0,L_1)$-smooth}} &
\multirow{2}{*}{\textbf{Compression}} \\[-0.3em]
\cmidrule(lr){3-4}
& & \textbf{HT} & \textbf{Affine var.} & & \\
\midrule[1pt]
\makecell[l]{\citep{zhang2019gradient, zhang2020improved}\\
\citep{zhao2021convergence}\\
\citep{crawshaw2022robustness}\\
\citep{koloskova2023revisiting}\\
\citep{li2023convex}\tnote{\color{blue}(1)}~~~~\tnote{\color{blue}(2)}\\
\citep{hubler2024parameter}\\
\citep{li2024convergence}\tnote{\color{blue}(1)}~~~~\tnote{\color{blue}(3)}\\
\citep{gaash2025convergence}\tnote{\color{blue}(1)}~~~~\tnote{\color{blue}(3)}}
& Discrete & \xmark & \xmark & \cmark & \xmark \\
\midrule[1pt]
\makecell[l]{\citep{faw2023beyond}\tnote{\color{blue}(1)}~~~~\tnote{\color{blue}(2)}\\
\citep{wang2023convergence}\tnote{\color{blue}(1)}~~~~\tnote{\color{blue}(2)}\\
\citep{chen2023generalized}}
& Discrete & \xmark & \cmark & \cmark & \xmark \\
\midrule[1pt]
\citep{khirirat2024error} & Discrete & \xmark & \xmark & \cmark & \cmark \\
\midrule[1pt]
\citep{chezhegov2025convergence}\tnote{\color{blue}(1)}~~~~\tnote{\color{blue}(3)}
& Discrete & \cmark & \xmark & \cmark & \xmark \\
\midrule[1pt]\midrule[1pt]
\citep{compagnoni2025unbiased} & Continuous & \cmark & \xmark & \xmark & \cmark \\
\textbf{This work} & Continuous & \cmark & \cmark & \cmark & \cmark \\
\bottomrule[1pt]
\end{tabular}%
}

\begin{tablenotes}[flushleft]
\footnotesize
\item[{\color{blue}(1)}] High-probability convergence analysis.
\item[{\color{blue}(2)}] Convergence bounds have inverse-power dependence on the failure probability.
\item[{\color{blue}(3)}] Derived for convex problems.
\end{tablenotes}
\end{threeparttable}

\end{table*}

\paragraph{Interplay of noise, compression, and adaptivity under $(L_0,L_1)$-smoothness.}
Previous research has extensively studied the effect of batch noise, compression, and adaptivity on the convergence of optimizers. Batch noise significantly influences stochastic gradient algorithms, affecting their convergence speed and stability \citep{simsekli2019tailindex, zhang2019gradient, kunstner2024heavy, compagnoni2025adaptive}. Noise characteristics such as heavy-tailed distributions have been shown to profoundly impact the optimization trajectories, necessitating robust algorithmic strategies \citep{csimcsekli2019heavy, gorbunov2021near}. Compression methods, including unbiased techniques such as sparsification and quantization \citep{alistarh2017qsgd, stich2018sparsified, mishchenko2024distributed} and biased approaches such as SignSGD \citep{bernstein2018signsgd, balles2018dissecting}, are critical for reducing communication overhead in distributed training. These compression techniques come with theoretical guarantees under various smoothness assumptions \citep{alistarh2017qsgd, gorbunov2020unified, mishchenko2024distributed, compagnoni2025unbiased}, and recent results also develop linear-rate or near-optimal behavior under generalized/$(L_0,L_1)$-smoothness \citep{vankov2024optimizing, tyurin2024toward}. Adaptive methods such as SignSGD normalize gradient elements to cope effectively with large or heavy-tailed gradient noise, thus demonstrating improved empirical robustness \citep{safaryan2021signsgd, compagnoni2025adaptive, compagnoni2025unbiased, kornilov2025sign}.

However, most of the works mentioned above rely on restrictive assumptions such as $L$-smoothness, i.e., the $L$-Lipschitz continuity of the gradient. To relax this condition, \citet{zhang2019gradient} introduces and empirically validates the \emph{$(L_0,L_1)$-smoothness} assumption, which allows the norm of the Hessian to be bounded by an affine function of the gradient norm, thereby significantly expanding the class of admissible problems. A growing body of work now analyzes (stochastic) \textit{first}-order methods under $(L_0,L_1)$ or more ``generalized-smoothness'' assumptions, including Clip-SGD and related clipping schemes \citep{zhang2019gradient, zhang2020adaptive, koloskova2023revisiting, reisizadeh2025variance, gorbunov2024methods, vankov2024optimizing, gaash2025convergence, pethick2025generalized}, Normalized SGD and variants with normalization-based schedules \citep{zhao2021convergence, chen2023generalized, hubler2024parameter, yang2024adaptive}, SignSGD \citep{crawshaw2022robustness}, AdaGrad \citep{faw2023beyond, wang2023convergence}, Adam \citep{li2024convergence, zhang2024convergence}, and SGD \citep{li2023convex}. Beyond these, there are accelerated and proximal/mirror-descent developments under generalized or $(L_0,L_1)$-smoothness \citep{tyurin2025near, yu2025convergence, tovmasyan2025revisiting, yu2025mirror}, nonlinearly preconditioned methods \citep{oikonomidis2025nonlinearly}, results on escaping saddle points \citep{cao2025efficiently}, zero-/first-order complexity under generalized smoothness \citep{lobanov2025power}, and decentralized/federated formulations with generalized smoothness and local steps \citep{demidovich2024methods, jiang2025decentralized}. For compressed communication, \citet{khirirat2024error} proposed and analyzed a momentum-based variant of normalized EF21-SGD \citep{richtarik2021ef21} under bounded variance. Additional generalized-smoothness analyses further connect normalization, compression, and relaxed smoothness guarantees \citep{lobanov2024linear, tyurin2024toward, yang2024adaptive}.

\paragraph{Positioning.}
To the best of our knowledge, there is no unified SDE-based framework that (a) is \emph{stability-faithful} under $(L_0,L_1)$-smoothness and (b) is then used to analyze the joint effect of distributed compression and general noise regimes, including affine variance and heavy tails.
This is the gap addressed by the present work (Table~\ref{tab:comparison_App}).

\section{Theoretical Framework}\label{sec:Asde}

In this section, we introduce the theoretical framework, assumptions, and notations used to formally derive the SDE models used in this paper. 
\begin{definition}
    Let $\mathcal{G}$ denote the set of continuous functions $g:\R^d\to\R$ of at most polynomial growth, namely such that there exist positive integers $k_1, k_2 > 0$ such that $       \lvert g(x)\rvert < k_1(1+\norm{x}_2^2 )^{k_2}$, for all $x\in\R^d$.
\end{definition}

To simplify the notation, we will write
\begin{equation*}
    b(x+\eta) = b_0(x)+\eta b_1(x)+O(\eta^2),
\end{equation*}
whenever there exists $g\in \mathcal{G}$, independent of $\eta$, such that
\begin{equation*}
    \lvert b(x+\eta)-b_0(x)-\eta b_1(x)\rvert \le g(x)\eta^2.
\end{equation*}

We now introduce the definition of weak approximation, which formalizes in which sense the solution to an SDE, which is a continuous-time random process, models a discrete-time optimizer.

\begin{definition}\label{def:weakapx_appendix}
A continuous-time stochastic process $(X_t)_{ t \in [0, T]}$ is an $\alpha$-order weak approximation of a discrete stochastic process $(x_k)_{k=0}^{\lfloor  T/\eta \rfloor }$ if for every polynomial growth function $g$, there exists a positive constant $C$, independent of $\eta$, such that $ \max _{k=0, \ldots, \lfloor  T/\eta \rfloor }\left|\E g\left(x_k\right)-\E g\left(X_{k \eta}\right)\right| \leq C \eta^\alpha$. We will often refer to $1$-order and $2$-order weak approximations as \textit{first}- and \textit{second}-order SDEs.
\end{definition}

This framework focuses on approximation in a \textit{weak sense}, meaning in distribution rather than path-wise. Since $\mathcal{G}$ contains all polynomials, all the moments of both processes become closer at a rate of $\eta^\alpha$ and thus their distributions. 
Thus, while the processes exhibit similar average behavior, their sample paths may differ significantly, justifying the term weak approximation.

The key ingredient for deriving the SDE is given by the following result (see Theorem 1, \citep{li2017stochastic}), which provides sufficient conditions to get a weak approximation in terms of the single step increments of both $X_t$ and $x_k$. Before stating the theorem, we list the regularity assumption under which we are working.

\paragraph{Assumptions:}\label{ass:sde}
    Assume that the following conditions are satisfied:
    \begin{itemize}
        \item $f,f_i\in \mathcal{C}_b^8(\R^d,\R)$;
        \item $f, f_i$ and its partial derivatives up to order $7$ belong to $\mathcal{G}$;
        \item $\nabla f, \nabla f_i$ satisfy the following Lipschitz condition: there exists $L>0$ such that
        \begin{equation*}
            \norm{\nabla f(u) - \nabla f(v)}_2+\sum_{i=1}^d \norm{\nabla f_i(u) - \nabla f_i(v)}_2\le L\norm{u-v}_2;   
        \end{equation*}
        \item $\nabla f, \nabla f_i$ satisfy the following growth condition: there exists $M>0$ such that
        \begin{equation*}
            \norm{\nabla f(x)}_2 + \sum_{i=1}^n\norm{\nabla f_i(x)}_2\le M(1+\norm{x}_2).
        \end{equation*}
            
    \end{itemize}
\begin{remark}
Although these assumptions are very strong, they mainly reflect limitations of the available proof techniques for weak approximation results. In particular, they can be so restrictive that even simple objectives, such as quadratic functions, may fall outside their formal scope. For this reason, we view the weak-approximation theorems as a rigorous guarantee under additional regularity, but not as a prerequisite for the usefulness of the resulting SDE models. In practice, the same weak-approximation heuristics can be applied far beyond the class of functions covered by these assumptions, and extensive empirical evidence indicates that the resulting modified equations closely track the dynamics of the corresponding optimizers across a range of architectures, including ResNets, ViTs, MLPs, and CNNs~\cite{compagnoni2025adaptive}.
\end{remark}

\begin{lemma}\label{lem:sdemoments}
    Let $0<\eta <1$. Consider a stochastic process $X_t, t\ge0$ satisfying the SDE
    \begin{equation}\label{eq:basicsde}
        dX_t = b(X_t)dt + \sqrt{\eta}\sigma(X_t)dW_t, \qquad X_0=x
    \end{equation}
    where $b,\sigma$ together with their derivatives belong to $\mathcal{G}$. Define the one-step difference $\Delta = X_\eta-x$, and indicate the $i$-th component of $\Delta$ with $\Delta_i$. Then we have
    \begin{enumerate}
        \item $\E\Delta_i = b_i\eta + \frac{1}{2}\left[\sum_{j=1}^d b_j\partial_j b_i\right]\eta^2+O(\eta^3)\qquad \forall i=1,\dots,d$;
        \item $\E\Delta_i\Delta_j = \left[b_ib_j+\sigma\sigma^\top_{ij}\right]\eta^2+O(\eta^3)\qquad\forall i,j=1,\dots, d$;
        \item $\E\prod_{j=1}^s \Delta_{i_j}=O(\eta^3)\qquad \forall s\ge 3,  i_j=1,\dots,d$.
    \end{enumerate}
    All functions above are evaluated at $x$.
\end{lemma}

\begin{theorem}\label{thm:sdeformal}
    Let $0<\eta<1$, $\tau>0$ and set $T = \floor{\tau/\eta}$. Let Assumption~\ref{ass:sde} hold and let $X_t$ be a stochastic process as in Lemma~\ref{lem:sdemoments}. Define $\Bar{\Delta} = x_1-x$ to be the increment of the discrete-time algorithm, and indicate the $i$-th component of $\Bar{\Delta}$ with $\Bar{\Delta}_i$. If in addition there exist $K_1, K_2, K_3, K_4\in \mathcal{G}$ so that
    \begin{enumerate}
        \item $|\E\Delta_i - \E\Bar{\Delta}_i|\le K_1(x)\eta^2,\qquad \forall i=1,\dots,d $;
        \item $|\E \Delta_i\Delta_j -\E \Bar{\Delta}_i\Bar{\Delta}_j|\le K_2(x)\eta^2,\qquad \forall i,j=1,\dots,d$;
        \item $|\E\prod_{j=1}^s \Delta_{i_j}-\E\prod_{j=1}^s\Bar{\Delta}_{i_j}|\le K_3(x)\eta^2,\qquad \forall s\ge 3, \forall i_j = 1,\dots, d$;
        \item $\E\prod_{j=1}^s|\Bar{\Delta}_{i_j}|\le K_4(x)\eta^2, \qquad \forall i_j=1,\dots,d$.
    \end{enumerate}
    Then, there exists a constant $C$ so that for all $k=0,1,\dots,N$ we have
    \begin{equation}
        |\E g(X_{k\eta})-\E g(x_k)|\le C\eta.
    \end{equation}
    We say Eq.~\ref{eq:basicsde} is an order $1$ weak approximation of the update step of $x_k$.
\end{theorem}

\section{New ODEs and SDEs for GD and SGD}\label{sec:new_models}

\subsection{Comparison With Discrete-Time Analysis}
In this section, we closely compare the dynamics of the loss function in discrete-time with that in continuous time as prescribed by the ODEs of GD. Consider GD with constant stepsize $\eta > 0$:
\begin{equation}
    x_{t+1} = x_t - \eta \nabla f(x_t).
\end{equation}

Using a second-order Taylor expansion around $x_t$ gives
\begin{equation}
    f(x_{t+1}) - f(x_t)
    = -\eta \|\nabla f(x_t)\|^2
    + \textcolor{purple}{\tfrac{\eta^2}{2}\nabla f(x_t)^\top \nabla^2 f(x_t)\, \nabla f(x_t)}
    + O_{x_t}(\eta^3).
\end{equation}
and therefore the normalized one-step loss drift (the discrete-time generator applied to $f$)
\begin{equation}\label{eq:loss_discr_App}
    \frac{f(x_{t+1}) - f(x_t)}{\eta}
    = - \|\nabla f(x_t)\|^2
    + \textcolor{purple}{\tfrac{\eta}{2}\nabla f(x_t)^\top \nabla^2 f(x_t)\, \nabla f(x_t)}
    + O_{x_t}(\eta^2).
\end{equation}
Here, $O_{x_t}(\eta^r)$ denotes a term bounded by $C(x_t)\eta^r$ for small enough $\eta$, with $C(x_t)$ independent of $\eta$.

\noindent\textit{Remark (generator vs.\ finite increment).}
Eq.~\ref{eq:loss_discr_App} compares the \emph{discrete-time generator} of GD applied to $f$ (a difference quotient) with the continuous-time drift $\frac{d}{dt}f(X_t)$.
Matching the finite increment $f(X_{t+\eta})-f(X_t)$ would instead introduce additional $\ddot f$ terms and corresponds to classical modified-equation analysis; this is \emph{not} the notion of matching we use in this paper.

However, the first-order ODE of GD implies that
\begin{equation}
    d f(X_t) = - \lVert \nabla f(X_t) \rVert_2^2 dt.
\end{equation}
We immediately notice that this continuous-time loss drift is completely missing the $O(\eta)$ correction highlighted in \textcolor{purple}{purple} color in Eq.~\ref{eq:loss_discr_App}.
The natural step is to shift to the second-order ODE, which implies that
\begin{equation}
    d f(X_t) = - \lVert \nabla f(X_t) \rVert_2^2 dt  \textcolor{emphgreen}{\scalebox{1.1}{$\boldsymbol{-}$}}      \textcolor{purple}{\tfrac{\eta}{2} \nabla f(X_t)^{\top}\nabla^2 f(X_t) \nabla f(X_t) dt}.
\end{equation}

While this ODE of the loss does incorporate some second-order information highlighted in \textcolor{purple}{purple color}, we notice that its \textcolor{emphgreen}{sign} is flipped with respect to the discrete-time loss drift in Eq.~\ref{eq:loss_discr_App}. This flipped sign is exactly the factor responsible for the failures of this second-order ODE.

\textbf{Deriving a New Model: An Ansatz Approach.}
Therefore, we understand that choosing the right model is critical to capture the aspects of the dynamics under analysis. Inspired by a classic approach in mathematical physics, we propose an ansatz for an ODE of the iterates of GD and look for one that models the loss dynamics more closely. For a real number $\alpha$, we propose:
\begin{equation}\label{eq:ansatz_App}
    dX_t = - \nabla f(X_t) dt + \alpha\eta\nabla^2 f(X_t) \nabla f(X_t) dt,
\end{equation}
which implies that the loss dynamics is driven by
\begin{equation}
    df(X_t) = - \lVert\nabla f(X_t)\rVert_2^2 dt  + \alpha\eta \nabla f(X_t)^{\top}\nabla^2 f(X_t) \nabla f(X_t) dt,
\end{equation}
Matching the induced loss drift with the discrete-time generator expansion in Eq.~\ref{eq:loss_discr_App} (i.e., comparing $\frac{d}{dt}f(X_t)$ with $\frac{f(x_{t+1})-f(x_t)}{\eta}$ up to $O(\eta)$) suggests choosing $\alpha = \frac{1}{2}$, which recovers the exact quadratic stability threshold. Therefore, we obtain 
\begin{equation}\label{eq:our_ode_App}
    dX_t = - \nabla f(X_t) dt + \frac{\eta}{2}\nabla^2 f(X_t) \nabla f(X_t) dt,
\end{equation}
is our new candidate ODE for GD.
\subsection{New Models}
First, we define two new models for GD and SGD. Then, we introduce a technical lemma and proceed to prove that our new models are \textbf{first}-order models for (S)GD.

\begin{definition}\label{def:new_sde}
Based on the discussion above, we define the new ODE model for GD:
\begin{equation}
    d X_t = -\nabla f(X_t) dt  \textcolor{orange}{\scalebox{1.1}{$\boldsymbol{+}$}}\tfrac{\eta}{2} \nabla^2 f(X_t) \nabla f(X_t) dt, \label{eq:gd_ode} \end{equation}
and the new SDE model for SGD:
\begin{equation}
    d X_t= -\nabla f(X_t) dt  \textcolor{orange}{\scalebox{1.1}{$\boldsymbol{+}$}} \tfrac{\eta}{2} \nabla^2 f(X_t) \nabla f(X_t) dt
    + \sqrt{\eta}\sqrt{\Sigma(X_t)} d W_t. \label{eq:sgd_sde}
\end{equation}

\end{definition}
\begin{remark}\label{rem:NewModels}
    Notice that, contrary to the second-order ODE and SDE from the literature (i.e., the classical modified-equation models), our stability-faithful models place a $\textcolor{orange}{\scalebox{1.1}{$\boldsymbol{+}$}}$ rather than a $\textcolor{emphgreen}{\scalebox{1.1}{$\boldsymbol{-}$}}$ in front of the $\tfrac{\eta}{2} \nabla^2 f(X_t) \nabla f(X_t)$ correction.
    This sign choice is deliberate: it matches the discrete-time generator expansion of the loss and, in particular, recovers the exact quadratic stepsize stability threshold.
    Matching the finite increment $f(X_{t+\eta})-f(X_t)$ instead leads back to the classical modified equation with the opposite sign; this is not our objective here.
\end{remark}

\begin{theorem}
Under the dynamics $\dot{x} = F(x)$ such that $F\in C^3(\mathbb{R})$, fix $t$. One has the expansion
    \[
x(t+\eta) = x + \eta F + \frac{\eta^2}{2} F'F 
+ \frac{\eta^3}{6}\big(F'' F^2 + (F')^2 F\big) + O(\eta^4),
\]
where all derivatives of $F$ are with respect to $x$, evaluated at $x(t)$.
\end{theorem}

\begin{proof}
By Taylor's theorem about $t$,
\[
x(t+\eta)
= x(t) + \eta x'(t) + \frac{\eta^2}{2}x''(t) + \frac{\eta^3}{6}x'''(t) + O(\eta^4).
\]
Note that:
\[
x'(t)=F(x(t)),\quad
x''(t)=F'(x(t))F(x(t)),\quad
x'''(t)= F''(x(t))F(x(t))^2 + \big(F'(x(t))\big)^2F(x(t)).
\]
\end{proof}

\begin{theorem}[ODE approximations of Gradient Descent]\label{thm:gd_ode}
Consider gradient descent (GD) with constant stepsize $\eta > 0$.  
The following ODEs are all weak-approximations of GD:
\begin{enumerate}
    \item The \textit{first-order} approximation from the literature:
    \begin{equation}\label{eq:SGD_FO_App}
        d X_t = - \nabla f(X_t)\, dt.
    \end{equation}
    \item The \textit{second-order} approximation from the literature:
    \begin{equation}\label{eq:SGD_SO_App}
        d X_t = - \nabla f(X_t)\, dt
        - \frac{\eta}{2}\,\nabla^2 f(X_t)\, \nabla f(X_t)\, dt.
    \end{equation}
    \item Our newly proposed \textit{first-order} approximation:
    \begin{equation}\label{eq:SGD_New_App}
        d X_t = - \nabla f(X_t)\, dt
        + \frac{\eta}{2}\,\nabla^2 f(X_t)\, \nabla f(X_t)\, dt.
    \end{equation}
\end{enumerate}
\end{theorem}

\begin{proof}
For simplicity, we consider gradient descent in one dimension, as generalizing to higher dimensions follows the same steps:
\[
x_{k+1} = x_k - \eta f'(x_k).
\]

We now seek a flow of the form
\[
F(x) = -f'(x) +\alpha f'(x) f''(x),
\]
and just substitute in the expressions in the previous result. Then, we will study the error as a function of $\alpha$. Note that we want to compute
\[
x(t+\eta) = x + \eta F + \frac{\eta^2}{2} F'F 
+ \frac{\eta^3}{6}\big(F'' F^2 + (F')^2 F\big) + O(\eta^4).
\]

We have:
\[
F=-f' + \alpha f' f'',\qquad
F'=-f''+\alpha\big((f'')^2+f' f'''\big),\qquad
F''=-f'''+\alpha\big(3 f'' f'''+f' f''''\big).
\]

So
\[
\begin{aligned}
x(t+\eta)
&= x
+ \eta\Big(-f' + \alpha f' f''\Big)\\[2mm]
&\quad+ \frac{\eta^2}{2}\Big[
f'f'' - \alpha\big(f'^2 f''' + 2 f'(f'')^2\big)
+ \alpha^2\big(f'^2 f'' f''' + f'(f'')^3\big)
\Big]\\[2mm]
&\quad+ \frac{\eta^3}{6}\Big[
-\big(f'^2 f''' + f'(f'')^2\big)\\
&\qquad\qquad
+ \alpha\big(f'^3 f'''' + 7 f'^2 f'' f''' + 3 f'(f'')^3\big)\\
&\qquad\qquad
- \alpha^2\big(2 f'^3 f'' f'''' + f'^3 (f''')^2 + 11 f'^2 (f'')^2 f''' + 3 f'(f'')^4\big)\\
&\qquad\qquad
+ \alpha^3\big(f'^3 (f'')^2 f'''' + f'^3 f'' (f''')^2 + 5 f'^2 (f'')^3 f''' + f'(f'')^5\big)
\Big]\\[2mm]
&\quad+ O(\eta^4).
\end{aligned}
\]

Assume now that $\alpha =\beta\eta$, we get

\[
\begin{aligned}
x(t+\eta)
&= x - \eta\, f' \\[2mm]
&\quad + \eta^2\Big(\beta+\tfrac12\Big) f' f'' \\[2mm]
&\quad - \frac{\eta^3}{6}\Big[\,(3\beta+1)\, f'^2 f''' + (6\beta+1)\, f'(f'')^2 \Big] \\[2mm]
&\quad + O(\eta^4),
\end{aligned}
\]

For $\alpha=0$ we get gradient flow and hence

\[
x(t+\eta)
= x - \eta f' 
+ \tfrac{1}{2}\,\eta^2 f' f''
- \tfrac{1}{6}\,\eta^3 \Big( f'^2 f''' + f'(f'')^2 \Big)
+ O(\eta^4),
\]
which is the first-order ODE from the literature.

For $\alpha = -\eta/2$

\[
x(t+\eta)
= x - \eta f' 
+ \eta^3\left(\tfrac{1}{12}\, f'^2 f''' + \tfrac{1}{3}\, f'(f'')^2 \right)
+ O(\eta^4),
\]
which is the second-order ODE from the literature.

Finally, for $\alpha = \eta/2$,

\[
x(t+\eta)
= x - \eta f' 
+ \eta^2 f' f'' 
- \tfrac{5}{12}\,\eta^3 f'^2 f''' 
- \tfrac{2}{3}\,\eta^3 f'(f'')^2
+ O(\eta^4),
\]
which is our newly proposed first-order ODE.

\end{proof}

The following theorem formalizes that our new SDE model from Eq. \ref{eq:sgd_sde} is formally a first-order weak approximation for SGD.

\begin{theorem}[SDE approximations of Stochastic Gradient Descent]\label{thm:gd_sde}
Consider stochastic gradient descent with constant stepsize $\eta > 0$.  
Its continuous-time approximations are given by the following SDEs:
\begin{enumerate}
    \item The \textit{first-order} approximation from the literature:
    \begin{equation}\label{eq:SGD_FO_App_S}
        d X_t = - \nabla f(X_t)\, dt + \sqrt{\eta \Sigma(X_t)} dW_t.
    \end{equation}
    \item The \textit{second-order} approximation from the literature:
    \begin{equation}\label{eq:SGD_SO_App_S}
        d X_t = - \nabla f(X_t)\, dt
        - \frac{\eta}{2}\,\nabla^2 f(X_t)\, \nabla f(X_t)\, dt  + \sqrt{\eta \Sigma(X_t)} dW_t.
    \end{equation}
    \item Our newly proposed \textit{first-order} approximation:
    \begin{equation}\label{eq:SGD_New_App_S}
        d X_t = - \nabla f(X_t)\, dt
        + \frac{\eta}{2}\,\nabla^2 f(X_t)\, \nabla f(X_t)\, dt  + \sqrt{\eta \Sigma(X_t)} dW_t.
    \end{equation}
\end{enumerate}
\end{theorem}
Here is the formal proof:
\begin{proof}
We work under the framework and assumptions of Section~\ref{sec:Asde}. Let $(x_k)_{k\ge 0}$ denote the SGD iterates with constant stepsize $\eta>0$:
\[
    x_{k+1} = x_k - \eta\, g(x_k,\xi_{k+1}),
\]
where $\{\xi_k\}_{k\ge 1}$ are i.i.d. random variables and $g(\cdot,\xi)$ is an unbiased stochastic gradient estimator:
\[
    \E[g(x,\xi)] = \nabla f(x), \qquad
    \mathrm{Cov}(g(x,\xi)) = \Sigma(x).
\]
We write the gradient noise as
\[
    \zeta(x,\xi) \coloneqq g(x,\xi) - \nabla f(x),
\]
so that $\E[\zeta(x,\xi)]=0$ and $\E[\zeta(x,\xi)\zeta(x,\xi)^\top]=\Sigma(x)$, with bounded moments up to order $3$ (consistent with \citep{li2017stochastic} and Assumption~\ref{ass:sde}).

Throughout, we fix $0<\eta<1$ and condition on $x_k=x$.

\paragraph{Step 1: One-step moments of SGD.}
Define the one-step increment of SGD by
\[
    \Bar{\Delta} \coloneqq x_{k+1} - x_k
    = -\eta\,\nabla f(x_k) - \eta\,\zeta(x_k,\xi_{k+1}).
\]
Conditioned on $x_k=x$, we thus have
\[
    \Bar{\Delta}
    = -\eta\,\nabla f(x) - \eta\,\zeta,
\]
where we write $\zeta := \zeta(x,\xi_{k+1})$.

Writing components as $\Bar{\Delta}_i$, $f_i = \partial_i f$, and $\Sigma_{ij}$ for the $(i,j)$-th entry of $\Sigma(x)$, we get:
\begin{align*}
    \E[\Bar{\Delta}_i \mid x_k=x]
    &= -\eta\,\E\big[ f_i(x) + \zeta_i \mid x_k=x\big]
     = -\eta\,f_i(x),\\[0.5ex]
    \E[\Bar{\Delta}_i \Bar{\Delta}_j \mid x_k=x]
    &= \eta^2\,\E\big[(f_i(x)+\zeta_i)(f_j(x)+\zeta_j)\mid x_k=x\big]\\
    &= \eta^2\big( f_i(x)f_j(x) + \Sigma_{ij}(x)\big).
\end{align*}
Moreover, for any $s\ge 3$ and indices $i_1,\dots,i_s$, we have
\[
    \E\Big[\prod_{\ell=1}^s \Bar{\Delta}_{i_\ell} \,\Big|\, x_k=x\Big]
    = O(\eta^s) = O(\eta^3),
\]
thanks to the bounded higher moments of $\zeta$ and Assumption~\ref{ass:sde}.  In particular, there exists $K_4\in \mathcal{G}$ such that
\[
    \E\Big[\prod_{\ell=1}^s\big|\Bar{\Delta}_{i_\ell}\big| \,\Big|\, x_k=x\Big]
    \le K_4(x)\,\eta^3 \le K_4(x)\,\eta^2,
\]
for all $0<\eta<1$ and all $s\ge 3$.  Hence condition (4) in Theorem~\ref{thm:sdeformal} holds.

\paragraph{Step 2: One-step moments of the SDE family.}
We now put the three candidate SDE models in Theorem~\ref{thm:gd_sde} into the template of Lemma~\ref{lem:sdemoments} and Theorem~\ref{thm:sdeformal}.

Fix $\alpha\in\R$ and consider the SDE
\begin{equation}\label{eq:sde_alpha_proof}
    dX_t = b_\alpha(X_t)\,dt + \sqrt{\eta}\,\sigma(X_t)\,dW_t,
\end{equation}
with
\[
    b_\alpha(x) \coloneqq -\nabla f(x) + \alpha\,\eta\,\nabla^2 f(x)\,\nabla f(x),
    \qquad
    \sigma(x)\sigma(x)^\top = \Sigma(x).
\]
We will later set
\[
    \alpha = 0, -\tfrac12, \tfrac12
\]
to recover the three SDEs in the statement.

Let $X_0 = x$ and define the one-step increment
\[
    \Delta \coloneqq X_\eta - x,
\]
with components $\Delta_i$.  By Lemma~\ref{lem:sdemoments} applied to \ref{eq:sde_alpha_proof}, we have
\begin{align}
    \E[\Delta_i]
    &= b_{\alpha,i}(x)\,\eta
      + \frac12\sum_{j=1}^d b_{\alpha,j}(x)\,\partial_j b_{\alpha,i}(x)\,\eta^2
      + O(\eta^3),
      \label{eq:delta_first_moment_general}\\
    \E[\Delta_i\Delta_j]
    &= \Big[b_{\alpha,i}(x)b_{\alpha,j}(x)
           + \big(\sigma\sigma^\top\big)_{ij}(x)\Big]\eta^2
      + O(\eta^3)\notag\\
    &= \Big[b_{\alpha,i}(x)b_{\alpha,j}(x)
           + \Sigma_{ij}(x)\Big]\eta^2 + O(\eta^3),
      \label{eq:delta_second_moment_general}\\
    \E\Big[\prod_{\ell=1}^s\Delta_{i_\ell}\Big]
    &= O(\eta^3), \qquad \forall s\ge 3.
      \label{eq:delta_higher_moments_general}
\end{align}
All functions above belong to $\mathcal{G}$ by Assumption~\ref{ass:sde}.

Now plug $b_\alpha(x) = -\nabla f(x) + \alpha\eta\,\nabla^2 f(x)\,\nabla f(x)$ into these expressions.  Write
\[
    h(x) \coloneqq \nabla^2 f(x)\,\nabla f(x),
    \qquad h_i(x) = \sum_{j=1}^d \partial_{ij} f(x)\,\partial_j f(x).
\]
Then
\[
    b_{\alpha,i}(x)
    = -f_i(x) + \alpha\,\eta\,h_i(x).
\]

For the \emph{first} moment, using \ref{eq:delta_first_moment_general}, we expand to order $\eta^2$:
\begin{align*}
    \E[\Delta_i]
    &= \Big(-f_i(x) + \alpha\eta\,h_i(x)\Big)\eta
       + \frac12\sum_{j=1}^d
          \Big(-f_j(x) + O(\eta)\Big)\,
          \partial_j\Big(-f_i(x) + O(\eta)\Big)\,\eta^2
       + O(\eta^3)\\
    &= -\eta f_i(x)
       + \alpha\,\eta^2 h_i(x)
       + \frac12\sum_{j=1}^d f_j(x)\,\partial_{j}f_i(x)\,\eta^2
       + O(\eta^3)\\
    &= -\eta f_i(x)
       + \Big(\alpha+\tfrac12\Big) h_i(x)\,\eta^2
       + O(\eta^3).
\end{align*}
In particular,
\begin{equation}\label{eq:E_delta_minus_E_bardelta}
    \E[\Delta_i] - \E[\Bar{\Delta}_i]
    = \Big(\alpha+\tfrac12\Big) h_i(x)\,\eta^2 + O(\eta^3)
    = O(\eta^2).
\end{equation}

For the \emph{second} moment, from \ref{eq:delta_second_moment_general} we note that
\[
    b_{\alpha,i}(x)b_{\alpha,j}(x)
    = \big(-f_i(x) + O(\eta)\big)\big(-f_j(x) + O(\eta)\big)
    = f_i(x)f_j(x) + O(\eta),
\]
hence
\begin{equation}\label{eq:E_delta_delta_minus_E_bardelta_bardelta}
    \E[\Delta_i\Delta_j]
    = \big( f_i(x)f_j(x) + \Sigma_{ij}(x)\big)\,\eta^2 + O(\eta^3),
\end{equation}
so that
\[
    \E[\Delta_i\Delta_j] - \E[\Bar{\Delta}_i\Bar{\Delta}_j]
    = O(\eta^3) \le K_2(x)\,\eta^2,
\]
for some $K_2\in \mathcal{G}$ and all $0<\eta<1$.  Similarly, combining \ref{eq:delta_higher_moments_general} with the bound on higher moments of $\Bar{\Delta}$ from Step~1, we obtain for $s\ge 3$
\[
    \Big|\E\prod_{\ell=1}^s \Delta_{i_\ell}
        - \E\prod_{\ell=1}^s \Bar{\Delta}_{i_\ell}\Big|
    = O(\eta^3) \le K_3(x)\,\eta^2
\]
for some $K_3\in \mathcal{G}$, and we have already seen that
\[
    \E\prod_{\ell=1}^s|\Bar{\Delta}_{i_\ell}|
    \le K_4(x)\,\eta^2
\]
for appropriate $K_4\in \mathcal{G}$. Thus, all four conditions of Theorem~\ref{thm:sdeformal} are satisfied for \ref{eq:sde_alpha_proof}, for any fixed $\alpha\in\R$.

\paragraph{Step 3: Weak order and identification of the three SDEs.}
By Theorem~\ref{thm:sdeformal}, for any fixed $\alpha\in\R$ the SDE
\[
    dX_t = b_\alpha(X_t)\,dt + \sqrt{\eta}\,\sigma(X_t)\,dW_t
\]
is an \emph{order $1$ weak approximation} of the SGD recursion, in the sense of Definition~\ref{def:weakapx_appendix}.

It remains to identify the three specific choices of $\alpha$ that correspond to the SDEs in the statement.

\begin{itemize}
    \item \textbf{Case $\alpha = 0$.}  
    Then $b_0(x) = -\nabla f(x)$, so the SDE \ref{eq:sde_alpha_proof} becomes
    \[
        d X_t = - \nabla f(X_t)\, dt + \sqrt{\eta\,\Sigma(X_t)}\, dW_t,
    \]
    which is exactly the \textit{first}-order SDE approximation from the literature in Eq.~\ref{eq:SGD_FO_App_S}.  By Theorem~\ref{thm:sdeformal}, this is an order $1$ weak approximation of SGD.

    \item \textbf{Case $\alpha = -\tfrac12$.}  
    Then
    \[
        b_{-1/2}(x) = -\nabla f(x) - \frac{\eta}{2}\,\nabla^2 f(x)\,\nabla f(x),
    \]
    and \ref{eq:sde_alpha_proof} becomes
    \[
        d X_t = - \nabla f(X_t)\, dt
        - \frac{\eta}{2}\,\nabla^2 f(X_t)\, \nabla f(X_t)\, dt
        + \sqrt{\eta\,\Sigma(X_t)}\, dW_t,
    \]
    which is exactly Eq.~\ref{eq:SGD_SO_App_S}, the classical \textit{second}-order SDE from the literature.  
    In this case, \ref{eq:E_delta_minus_E_bardelta} shows that
    \[
        \E[\Delta_i] - \E[\Bar{\Delta}_i] = O(\eta^3),
    \]
    so the drift matches to one order higher; combined with the analysis in \citep{li2017stochastic}, this yields a second-order weak approximation.  In particular, it is also an order $1$ weak approximation.

    \item \textbf{Case $\alpha = \tfrac12$.}  
    Then
    \[
        b_{1/2}(x) = -\nabla f(x) + \frac{\eta}{2}\,\nabla^2 f(x)\,\nabla f(x),
    \]
    and \ref{eq:sde_alpha_proof} becomes
    \[
        d X_t = - \nabla f(X_t)\, dt
        + \frac{\eta}{2}\,\nabla^2 f(X_t)\, \nabla f(X_t)\, dt
        + \sqrt{\eta\,\Sigma(X_t)}\, dW_t,
    \]
    which is exactly our new SDE model in Eq.~\ref{eq:SGD_New_App_S}.  
    As shown above, all four conditions of Theorem~\ref{thm:sdeformal} hold, so this SDE is also an \textit{order $1$ weak approximation} of SGD.
\end{itemize}

This proves that all three SDEs listed in Theorem~\ref{thm:gd_sde} are weak approximations of SGD in the sense of Definition~\ref{def:weakapx_appendix}: the first and third are first-order weak approximations, while the second one is the classical second-order stochastic modified equation from the literature.

\end{proof}

\subsection{Comparing ODEs -  An Insight Perspective} \label{sec:Compare}

In this section, we showcase how models from the literature fail to properly model the dynamics of GD, especially regarding the constraints on the learning rate to ensure convergence. In contrast, we show that our model is in accordance with GD.

\subsubsection{Quadratic Function}
For didactic reasons, we now compare the proofs for a convergence bound on the loss value $f(x)$ when the loss is a $1$-dimensional convex quadratic function $\tfrac{\lambda x^2}{2}$. \textbf{To avoid overloading the proof with technicalities intrinsic in Itô calculus}, we restrict the analysis to the noiseless and single-node case. The \textit{first}-order ODE is
\begin{equation}
    d X_t = - \nabla f(X_t) dt = -\lambda X_t dt,
\end{equation}
which implies that
\begin{equation}\textstyle
    d f(X_t) = - 2 \lambda f(X_t) dt \implies f(X_t) = f(X_0) e^{-2\lambda t} \overset{t \rightarrow \infty}{\rightarrow} 0,
\end{equation}
somewhat implying that GD converges independently of the constant $L$ and of the learning rate $\eta$. Much differently, the \textit{second}-order ODE \textit{from the literature} is
\begin{equation}\textstyle
    d X_t = - \nabla f(X_t) dt - \frac{\eta}{2}\nabla^2 f(X_t) \nabla f(X_t) dt,
\end{equation}
which implies that
{\small
\begin{align}
    d f(X_t) &  = - \lVert  \nabla f(X_t) \rVert_2^2 dt \textcolor{emphgreen}{\mathbf{-}} \tfrac{\eta}{2} \nabla f(X_t)^{\top}\nabla^2 f(X_t) \nabla f(X_t) dt  = - 2 \lambda f(X_t) dt - \frac{\eta}{2}\lambda X_t^{\top} \lambda \lambda X_t \\
    & = -2 \lambda \left(1 + \frac{\lambda \eta}{2}\right) f(X_t) dt \implies f(X_t) = f(X_0) e^{-2 \lambda \left(1 + \frac{\lambda \eta}{2}\right) t} \overset{t \rightarrow \infty}{\rightarrow} 0,
\end{align}}
which is also inconsistent with the discrete-time analysis since we get convergence for any $\eta>0$.

Now, we try to leverage our new ODE derived in Theorem \ref{thm:gd_ode} and get that:
{\small
\begin{align}
    d f(X_t) &  = - \lVert  \nabla f(X_t) \rVert_2^2 dt \textcolor{orange}{\mathbf{+}}\tfrac{\eta}{2} \nabla f(X_t)^{\top}\nabla^2 f(X_t) \nabla f(X_t) dt  = - 2 \lambda f(X_t) dt + \frac{\eta}{2}\lambda X_t^{\top} \lambda \lambda X_t \\
    & = -2 \lambda \left(1 - \frac{\lambda \eta}{2}\right) f(X_t) dt \implies f(X_t) = f(X_0) e^{-2 \lambda \left(1 - \frac{\lambda \eta}{2}\right) t} \overset{t \rightarrow \infty}{\rightarrow} 0,
\end{align}}
which only converges if $\eta < \frac{2}{\lambda}$. This is consistent with the analysis in discrete time.

\textbf{Conclusion:} First of all, it is immediately apparent that while \textit{first}-order approximations may lead to relevant insights, they prevent us from having a full picture. Second, we demonstrated that the classic \textit{second}-order SDE also led us to results that are inconsistent with the discrete-time analysis. Finally, our model provides a qualitatively faithful description of the true GD dynamics.

\subsubsection{Quartic Function} \label{sec:quartic}
Here, we compare the three ODEs listed above as they describe the optimization of a quartic function $f(x) = \frac{x^4}{4}$: We find that the classic ones both fail. First of all, a single step of gradient descent with stepsize $\eta$ reads
\[
    x_{k+1} = x_k - \eta \nabla f(x_k) 
    = x_k - \eta x_k^3 ,
\]
meaning that if $\eta > \tfrac{2}{x_k^2}$ the dynamics \emph{explodes}. In particular,
\begin{equation}
    \frac{f(x_{k+1}) - f(x_k)}{\eta} = - x_k^6 + \frac{3}{2}\eta x_k^8 + O(\eta^2).
\end{equation}
Using the first-order ODE, we obtain
\begin{equation}
    dX_t = - X_t^3 dt \implies f(X_t) = \frac{1}{4(2t + X_0^{-2})^2}
\end{equation}
This model predicts universal convergence with a polynomial rate, but it does \emph{not} capture the exploding behaviour observed in GD. Using the second-order ODE, we obtain
\begin{equation}
    dX_t = -X_t^3 dt -\tfrac{3\eta}{2} X_t^5 dt \implies df(X_t) = -X_t^6 dt - \tfrac{3\eta}{2} X_t^8 dt,
\end{equation}
from which we understand that since the additional term is \emph{negative}, this ODE suggests \emph{faster convergence} for larger $\eta$. Using our new ODE, we obtain
\begin{equation}
    dX_t = -X_t^3 dt +\tfrac{3\eta}{2} X_t^5 dt \implies df(X_t) = -X_t^6 dt + \tfrac{3\eta}{2} X_t^8 dt,
\end{equation}
which matches the discrete-time expansion of the loss difference quotient up to $O(\eta^2)$. Importantly, it captures the phenomenon that the learning rate $\eta$ needs to scale inversely to the norm of the iterates for GD to converge.

\paragraph{Conclusion.}  
On the quartic loss, the first-order ODE predicts convergence for all $\eta$, missing the instability. The second-order ODE from the literature predicts \emph{accelerated convergence} for larger $\eta$, in direct contradiction with GD. In contrast, our new ODE reproduces the key phenomenon: the learning rate $\eta$ needs to scale inversely to the norm of the iterates for GD to converge. Hence, our model provides a qualitatively faithful description of the true GD dynamics.

\subsection{Diffusion Approximation for the Loss in SGD}
In this section, we propose an alternative approach to the derivation of a continuous-time model for SGD. Rather than modeling the iterates and use the Itô Lemma to study the SDE of the loss function, we try a new approach: We directly investigate the possibility of directly modeling the dynamics of the loss. Consider stochastic gradient descent (SGD) with constant stepsize $\eta > 0$:
\begin{equation}
    x_{t+1} = x_t - \eta g_t,
    \qquad g_t = \nabla f(x_t) + \zeta_t,
\end{equation}
where $f:\mathbb{R}^d \to \mathbb{R}$ is smooth, $\zeta_t$ is the gradient noise satisfying
\[
\mathbb{E}[\zeta_t \mid x_t] = 0,
\qquad
\mathrm{Cov}(\zeta_t \mid x_t) = \Sigma(x_t).
\]
We study the dynamics of the \emph{loss process} $Y_t := f(x_t)$.

\paragraph{Step 1. Taylor expansion of the loss.}
Using a second-order Taylor expansion around $x_t$, for $h = -\eta g_t$ we have
\begin{align}
    f(x_{t+1})
    &= f(x_t + h) \nonumber \\
    &= f(x_t) + \nabla f(x_t)^\top h
       + \tfrac{1}{2} h^\top \nabla^2 f(x_t) h
       + O(\|h\|^3).
\end{align}
Substituting $h = -\eta g_t$ gives
\begin{equation}
    f(x_{t+1}) - f(x_t)
    = -\eta \nabla f(x_t)^\top g_t
    + \tfrac{\eta^2}{2} g_t^\top \nabla^2 f(x_t)\, g_t
    + O(\eta^3).
    \label{eq:loss_increment}
\end{equation}

\paragraph{Step 2. Expansion of stochastic terms.}
Expanding with $g_t = \nabla f(x_t) + \zeta_t$ yields
\begin{align}
    f(x_{t+1}) - f(x_t)
    &= -\eta \|\nabla f(x_t)\|^2
       - \eta \nabla f(x_t)^\top \zeta_t \nonumber \\
    &\quad + \tfrac{\eta^2}{2}\nabla f(x_t)^\top \nabla^2 f(x_t)\nabla f(x_t) \\
    & \quad + \tfrac{\eta^2}{2}\zeta_t^\top \nabla^2 f(x_t)\zeta_t
       + \eta^2 \nabla f(x_t)^\top \nabla^2 f(x_t)\zeta_t
       + O(\eta^3).
    \label{eq:loss_increment_expanded}
\end{align}

\paragraph{Step 3. Drift and volatility.}
Taking the conditional expectation given $x_t$,
\begin{align}
    \mathbb{E}[f(x_{t+1}) - f(x_t) \mid x_t]
    &= -\eta \|\nabla f(x_t)\|^2 \nonumber \\
    &\quad + \tfrac{\eta^2}{2}\nabla f(x_t)^\top \nabla^2 f(x_t)\nabla f(x_t)
       + \tfrac{\eta^2}{2}\operatorname{tr}\big(\nabla^2 f(x_t)\Sigma(x_t)\big)
       + O(\eta^3).
\end{align}
The stochastic fluctuations arise from the linear terms in $\zeta_t$,
\[
-\eta \nabla f(x_t)^\top \zeta_t
+ \eta^2 \nabla f(x_t)^\top \nabla^2 f(x_t)\zeta_t,
\]
whose leading-order contribution is
\[
-\eta \nabla f(x_t)^\top \zeta_t.
\]
This term has conditional variance
\[
\mathrm{Var}\left(-\eta \nabla f(x_t)^\top \zeta_t \,\big|\, x_t\right)
= \eta^2 \nabla f(x_t)^\top \Sigma(x_t)\, \nabla f(x_t).
\]

\paragraph{Step 4. Continuous-time limit.}
Rescaling time by $s = t\eta$ and letting $\eta \to 0$, the increments
\ref{eq:loss_increment_expanded} converge in distribution to the diffusion
\begin{equation}
    dY_s
    = \Big(-\|\nabla f(X_s)\|^2
       + \tfrac{\eta}{2}\nabla f(X_s)^\top \nabla^2 f(X_s)\nabla f(X_s)
       + \tfrac{\eta}{2}\operatorname{tr}\big(\nabla^2 f(X_s)\Sigma(X_s)\big)\Big)\,ds
       + G(X_s)\, dW_s,
\end{equation}
where $W_s$ is a standard Brownian motion and the scalar volatility
$G(x)$ is defined by
\begin{equation}
    G(x)^2 = \nabla f(x)^\top \Sigma(x)\, \nabla f(x).
\end{equation}
Interestingly, this SDE is the same one that one gets by applying Itô's Lemma on $f(X_t)$ under the dynamics of our newly proposed SDE in Eq. \ref{eq:sgd_sde}, which consolidates the intuition that our model properly captures the dynamics of SGD faithfully.

\section{Theoretical Results}

\paragraph{Assumptions and notation.} In line with \citep{compagnoni2025unbiased}, we assume that the stochastic gradient of the $i$-th client is given by $\nabla f_{\gamma_i}(x) = \nabla f(x) + Z_i(x)$, where $Z_i(x)$ denotes the gradient noise and $Z_i(x)$ is independent of $Z_j(x)$ for $i \neq j$. If $Z_i(x) \in L^1(\R^d)$, we assume $\E[Z_i(x)] = 0$, and if $Z_i(x) \in L^2(\R^d)$, we assume $Cov(Z_i(x)) = \Sigma_i(x)$ (we omit the size of the batch $\gamma$ unless relevant) s.t. $\sqrt{\Sigma_i(x)}$ is bounded, Lipschitz, satisfies affine growth, and together with its derivatives, it grows at most polynomially fast (Definition 2.5 in \cite{Malladi2022AdamSDE}). Importantly, we assume that all $Z_i(x)$ have a smooth and bounded probability density function whose derivatives are all integrable: A common assumption in the literature is for $Z_i(x)$ to be Gaussian \cite{ahn2012bayesian,chen2014stochastic,mandt2016variational,stephan2017stochastic,zhu2019anisotropic,wu2020noisy,Xie2021}: See \cite{jastrzkebski2017three} for the justification why this could be the case. Differently, our assumption allows for heavy-tailed distributions such as the Student's t. It is important to point out that \cite{li2017stochastic,mertikopoulos2018convergence,raginsky2012continuous,zhu2019anisotropic,mandt2016variational,ahn2012bayesian,jastrzkebski2017three} use a Gaussian noise with a constant covariance matrix to model batch noise.

\subsection{Distributed SGD}

\subsubsection{First Order SDE}

The following is the \textit{first}-order SDE model of DSGD (see Theorem 3.2 in \cite{compagnoni2025unbiased}). Let us consider the stochastic process $ X_t \in \mathbb{R}^{d} $ defined as the solution of

\begin{equation}\label{eq:DUSGD_SDE_App_First}
    d X_t = - \nabla f(X_t) dt + \sqrt{\frac{\eta}{N}} \sqrt{\hat{\Sigma}(X_t)} dW_t,
\end{equation}
where $\hat{\Sigma}(x)\eqdef  \frac{1}{N} \sum_{i=1}^{N} \Sigma_i(x)$ is the average of the covariance matrices of the $N$ clients.

\begin{theorem}
\label{thm:DSGD_App}
    Let $f$ be $(L_0,L_1)$-smooth, $ \lVert\Sigma_i(x)\rVert_{\infty} < \sigma_{0,i}^2 + \sigma_{1,i}^2 \lVert \nabla f(x) \rVert_2^2$, the learning rate scheduler $\eta_t$ s.t. $\phi^{(i)}_t = \int_0^t (\eta_s)^i ds$, $\phi^{(1)}_t \overset{t \rightarrow \infty}{\rightarrow} \infty$, $\frac{\phi^{(2)}_t}{\phi^{(1)}_t} \overset{t \rightarrow \infty}{\rightarrow} 0$, $\overline{\sigma_0^2}\eqdef \frac{1}{N} \sum_{i=1}^{N} \sigma_{0,i}^2$, and $\overline{\sigma_1^2}\eqdef \frac{1}{N} \sum_{i=1}^{N} \sigma_{1,i}^2$. Then, for $0 < \epsilon <1$,
\begin{equation}\label{eq:StepDSGD_App_First}
    \eta \eta_t < \frac{2 N \epsilon}{d\left(  \overline{\sigma_1^2} L_0 + \overline{\sigma_0^2}L_1 + L_1 \overline{\sigma_1^2}  \E\left[\lVert \nabla f(X_t) \rVert_2 \right]\right) },
\end{equation}
and for a random time $\Hat{t}$ with distribution $\frac{\eta_t}{\phi^{(1)}_t}$, we have that
\begin{equation}
   \E \left[\lVert \nabla f(X_{\Hat{t}}) \rVert_2^2 \right] \leq \frac{1}{ \phi^{(1)}_t(1-\epsilon)} \left( f(X_0) - f(X_*)  + \phi^{(2)}_t \frac{\eta d (L_0 + L_1) (\overline{\sigma_0^2} + \overline{\sigma_1^2} )}{2 N}  \right) \overset{t \rightarrow \infty}{\rightarrow} 0.
\end{equation}
\end{theorem}

\begin{proof}
Using Itô's Lemma and using a learning rate scheduler $\eta_t$ during the derivation of the SDE, we have
\begin{align}
    d (f(X_t) - f(X_*))  = & - \eta_t \lVert \nabla f(X_t) \rVert_2^2 dt + \mathcal{O}(\text{Noise}) + (\eta_t)^2\frac{\eta}{2 N} \tr(\nabla^2 f(X_t) \tilde{\Sigma}(X_t)) dt \\
     \leq & - \eta_t \lVert \nabla f(X_t) \rVert_2^2 dt + \mathcal{O}(\text{Noise}) \\
      &+ (\eta_t)^2\frac{\eta  (\overline{\sigma_0^2} + \overline{\sigma_1^2}\lVert \nabla f(X_t) \rVert_2^2 )d(L_0+L_1\lVert \nabla f(X_t) \rVert)}{2 N} dt,
\end{align}
where we used that $\tr \left( \nabla^2 f(x) \tilde{\Sigma}(x) \right) \leq  d \lVert \nabla^2 f(x) \rVert_{\infty} \lVert \tilde{\Sigma}(x) \rVert_{\infty}$ together with the smoothness and noise assumptions. Importantly, $\mathcal{O}(\text{Noise}) = \sqrt{\tilde{\Sigma}(X_t)} \nabla f(X_t) d W_t$.

\textbf{Phase $1$:} If $\lVert \nabla f(X_t) \rVert \leq 1 $, we have that
\begin{equation}
    d (f(X_t) - f(X_*)) \leq - \eta_t \lVert \nabla f(X_t) \rVert_2^2 dt + (\eta_t)^2\frac{\eta  (\overline{\sigma_0^2} + \overline{\sigma_1^2} )d(L_0+L_1)}{2 N} dt  + \mathcal{O}(\text{Noise}),
\end{equation}

\textbf{Phase $2$:} If $\lVert \nabla f(X_t) \rVert > 1 $, we have
\begin{align}
    d (f(X_t) - f(X_*)) & = - \eta_t \lVert \nabla f(X_t) \rVert_2^2 dt + \mathcal{O}(\text{Noise}) + (\eta_t)^2\frac{\eta}{2 N} \tr(\nabla^2 f(X_t) \tilde{\Sigma}(X_t)) dt \\
    & \leq - \eta_t \lVert \nabla f(X_t) \rVert_2^2 dt + \mathcal{O}(\text{Noise}) \\
    & + (\eta_t)^2\frac{\eta  (\overline{\sigma_0^2} + \overline{\sigma_1^2}\lVert \nabla f(X_t) \rVert_2^2 )d(L_0+L_1\lVert \nabla f(X_t) \rVert)}{2 N} dt \\
    & = -\eta_t \lVert \nabla f(X_t) \rVert_2^2\left(1 - \frac{\eta_t \eta    d }{2 N} \left( \overline{\sigma_1^2} L_0 + \overline{\sigma_0^2}L_1 + L_1 \overline{\sigma_1^2}\lVert \nabla f(X_t) \rVert_2 \right)\right) dt \\
    & + (\eta_t)^2\frac{\eta  \overline{\sigma_0^2}dL_0}{2 N} dt + \mathcal{O}(\text{Noise}).
\end{align}
By taking a worst-case scenario approach, we merge these two bounds into a single one:
\begin{align}
    d (f(X_t) - f(X_*)) & \leq -\eta_t \lVert \nabla f(X_t) \rVert_2^2\left(1 - \frac{\eta_t \eta d }{2 N} \left( \overline{\sigma_1^2} L_0 + \overline{\sigma_0^2}L_1 + L_1 \overline{\sigma_1^2}\lVert \nabla f(X_t) \rVert_2 \right)\right) dt\\
    & + (\eta_t)^2\frac{\eta d (L_0 + L_1) (\overline{\sigma_0^2} + \overline{\sigma_1^2} )}{2 N} dt + \mathcal{O}(\text{Noise}).
\end{align}
Therefore, for $0<\epsilon <1$ we have that if
\begin{equation}
    1 - \frac{\eta_t \eta d }{2 N} \left( \overline{\sigma_1^2} L_0 + \overline{\sigma_0^2}L_1 + L_1 \overline{\sigma_1^2}\lVert \nabla f(X_t) \rVert_2 \right) > 1 - \epsilon,
\end{equation}
or, equivalently
\begin{equation}
    \eta \eta_t < \frac{2 N \epsilon}{d\left(  \overline{\sigma_1^2} L_0 + \overline{\sigma_0^2}L_1 + L_1 \overline{\sigma_1^2}\lVert \nabla f(X_t) \rVert_2\right) },
\end{equation}
we have that 
\begin{align}
    d (f(X_t) - f(X_*)) & \leq -\eta_t \lVert \nabla f(X_t) \rVert_2^2\left(1 - \epsilon\right) dt  + (\eta_t)^2\frac{\eta d (L_0 + L_1) (\overline{\sigma_0^2} + \overline{\sigma_1^2} )}{2 N} dt + \mathcal{O}(\text{Noise}).
\end{align}
Therefore, 
\begin{equation}
    \eta_t \lVert \nabla f(X_t) \rVert_2^2\left(1 - \epsilon\right) dt \leq - d (f(X_t) - f(X_*)) + (\eta_t)^2\frac{\eta d (L_0 + L_1) (\overline{\sigma_0^2} + \overline{\sigma_1^2} )}{2 N} dt + \mathcal{O}(\text{Noise}).
\end{equation}
Dividing by $1-\epsilon$, integrating over time, and using the martingality of the noise term under the expected value,
\begin{equation}
    \int_0^t \eta_s \E\lVert \nabla f(X_s) \rVert_2^2 ds \leq \frac{1}{1-\epsilon} \left( f(X_0) - f(X_*) +  \phi^{(2)}_t \frac{\eta d (L_0 + L_1) (\overline{\sigma_0^2} + \overline{\sigma_1^2} )}{2 N} \right).
\end{equation}
Dividing by $\phi^{(1)}_t$ and by the Law of the Unconscious Statistician, we have that
\begin{equation}
   \E \left[\lVert \nabla f(X_{\Hat{t}}) \rVert_2^2 \right] \leq \frac{1}{ \phi^{(1)}_t(1-\epsilon)} \left( f(X_0) - f(X_*)  + \phi^{(2)}_t \frac{\eta d (L_0 + L_1) (\overline{\sigma_0^2} + \overline{\sigma_1^2} )}{2 N}  \right) \overset{t \rightarrow \infty}{\rightarrow} 0,
\end{equation}
where $\Hat{t}$, is a random time with distribution $\frac{\eta_{\Hat{t}}}{\phi^{(1)}_t}$.

Finally, for practical reasons, we leverage the distributed setting to tighten the requirements on the learning rate scheduler to make it experimentally viable (see Section \ref{subsec:constructive_adaptivity} for the details), and require
\begin{equation}
    \eta \eta_t < \frac{2 N \epsilon}{d\left(  \overline{\sigma_1^2} L_0 + \overline{\sigma_0^2}L_1 + L_1 \overline{\sigma_1^2}  \E\left[\lVert \nabla f(X_t) \rVert_2 \right] \right) }.
\end{equation}
\end{proof}

\subsubsection{Our New First-Order SDE for DSGD }

The following is the \textit{first}-order SDE model of DSGD and is a straightforward generalization of Theorem 3.2 in \cite{compagnoni2025unbiased} and Remark \ref{rem:NewModels}. Let us consider the stochastic process $ X_t \in \mathbb{R}^{d} $ defined as the solution of

\begin{equation}\label{eq:DUSGD_SDE_App_Second}
    d X_t = - \nabla f(X_t) dt + \frac{\eta}{2} \nabla^2 f(X_t) \nabla f(X_t) dt + \sqrt{\frac{\eta}{N}} \sqrt{\hat{\Sigma}(X_t)} dW_t,
\end{equation}
where $\hat{\Sigma}(x)\eqdef  \frac{1}{N} \sum_{i=1}^{N} \Sigma_i(x)$ is the average of the covariance matrices of the $N$ clients.

\begin{theorem}
\label{thm:DSGD_App_2nd}
    Let $f$ be $(L_0,L_1)$-smooth, $ \lVert\Sigma_i(x)\rVert_{\infty} < \sigma_{0,i}^2 + \sigma_{1,i}^2 \lVert \nabla f(x) \rVert_2^2$, the learning rate scheduler $\eta_t$ s.t. $\phi^{(i)}_t = \int_0^t (\eta_s)^i ds$, $\phi^{(1)}_t \overset{t \rightarrow \infty}{\rightarrow} \infty$, $\frac{\phi^{(2)}_t}{\phi^{(1)}_t} \overset{t \rightarrow \infty}{\rightarrow} 0$, $\overline{\sigma_0^2}\eqdef \frac{1}{N} \sum_{i=1}^{N} \sigma_{0,i}^2$, and $\overline{\sigma_1^2}\eqdef \frac{1}{N} \sum_{i=1}^{N} \sigma_{1,i}^2$. Then, for $0 < \epsilon <1$,
\begin{equation}\label{eq:StepDSGD_App_Second}
    \eta \eta_t < \frac{2 \epsilon}{L_0+L_1 \E\left[\lVert \nabla f(X_t) \rVert \right]+  \frac{d}{N}\left( \overline{\sigma_1^2} L_0 + \overline{\sigma_0^2}L_1 + L_1 \overline{\sigma_1^2}\E\left[\lVert \nabla f(X_t) \rVert \right] \right)},
\end{equation}
and for a random time $\Hat{t}$ with distribution $\frac{\eta_t}{\phi^{(1)}_t}$, we have that
\begin{equation}
   \E \left[\lVert \nabla f(X_{\Hat{t}}) \rVert_2^2 \right] \leq \frac{1}{ \phi^{(1)}_t(1-\epsilon)} \left( f(X_0) - f(X_*)  + \frac{\eta \phi^{(2)}_t}{2 N} (L_0+L_1) d \overline{\sigma_0^2} \right) \overset{t \rightarrow \infty}{\rightarrow} 0.
\end{equation}
\end{theorem}

\begin{proof}
Using Itô's Lemma and using a learning rate scheduler $\eta_t$ during the derivation of the SDE, we have that for $\mathcal{O}(\text{Noise}) = \sqrt{\tilde{\Sigma}(X_t)} \nabla f(X_t) d W_t$,
\begin{align}
    d (f(X_t) - f(X_*)) & = - \eta_t \lVert \nabla f(X_t) \rVert_2^2 dt + \frac{\eta \eta_t^2}{2} \left( \nabla f(X_t) \right)^{\top} \nabla^2 f(X_t) \nabla f(X_t) dt \\
    & + \mathcal{O}(\text{Noise}) + (\eta_t)^2\frac{\eta}{2 N} \tr(\nabla^2 f(X_t) \tilde{\Sigma}(X_t)) dt \\
    & \leq - \eta_t \lVert \nabla f(X_t) \rVert_2^2 dt + \frac{\eta \eta_t^2}{2} (L_0+L_1\lVert \nabla f(X_t) \rVert) \lVert \nabla f(X_t) \rVert^2 dt \\
    & +  \mathcal{O}(\text{Noise})  + (\eta_t)^2\frac{\eta  (\overline{\sigma_0^2} + \overline{\sigma_1^2}\lVert \nabla f(X_t) \rVert_2^2 )d(L_0+L_1\lVert \nabla f(X_t) \rVert)}{2 N} dt.
\end{align}

\textbf{Phase $1$:} If $\lVert \nabla f(X_t) \rVert \leq 1 $, 

\begin{align}
      d (f(X_t) - f(X_*)) & \leq -\lVert \nabla f(X_t) \rVert_2^2 \left( \eta_t - \frac{\eta \eta_t^2}{2}(L_0 +  L_1 \lVert \nabla f(X_t) \rVert_2) \left( 1 + \frac{d \overline{\sigma_1^2}}{N}\right) \right)dt \\
      & + \frac{\eta \eta_t^2}{2 N} .(L_0+L_1) d \overline{\sigma_0^2} dt +  \mathcal{O}(\text{Noise}).
\end{align}

\textbf{Phase $2$:} If $\lVert \nabla f(X_t) \rVert > 1 $, we have
\begin{align}
    d (f(X_t) - f(X_*)) &  \leq - \eta_t \lVert \nabla f(X_t) \rVert_2^2 dt +  \frac{\eta \eta_t^2}{2} (L_0+L_1\lVert \nabla f(X_t) \rVert) \lVert \nabla f(X_t) \rVert^2 dt \\
    & + \mathcal{O}(\text{Noise}) + (\eta_t)^2\frac{\eta  (\overline{\sigma_0^2} + \overline{\sigma_1^2}\lVert \nabla f(X_t) \rVert_2^2 )d(L_0+L_1\lVert \nabla f(X_t) \rVert)}{2 N} dt \\
    & = -\eta_t \lVert \nabla f(X_t) \rVert_2^2\left[1 - \frac{\eta_t \eta    }{2 } \left[ (L_0+L_1\lVert \nabla f(X_t) \rVert ) \left[1 + \frac{d \overline{\sigma_1^2}}{N}\right]+  \frac{d \overline{\sigma_0^2}L_1}{N}   \right] \right] dt \nonumber \\ 
    & + (\eta_t)^2\frac{\eta  \overline{\sigma_0^2}dL_0}{2 N} dt +  \mathcal{O}(\text{Noise}).
\end{align}

By taking a worst-case scenario approach, we merge these two bounds into a single one:

\begin{align}
    d (f(X_t) - f(X_*)) & \leq -\eta_t \lVert \nabla f(X_t) \rVert_2^2\left[1 - \frac{\eta_t \eta    }{2 } \left[( L_0+L_1\lVert \nabla f(X_t) \rVert ) \left[1 + \frac{d \overline{\sigma_1^2}}{N}\right]+  \frac{d \overline{\sigma_0^2}L_1}{N}   \right] \right] dt \nonumber \\
    & + (\eta_t)^2  \frac{\eta}{2 N} (L_0+L_1) d \overline{\sigma_0^2} dt +  \mathcal{O}(\text{Noise}).
\end{align}
With arguments that follow the same steps we detailed in the proof of Theorem \ref{thm:DSGD_App}, for $0<\epsilon <1$, we have that if
\begin{equation}
    \eta \eta_t < \frac{2 \epsilon}{L_0+L_1\lVert \nabla f(X_t) \rVert+  \frac{d}{N}\left( \overline{\sigma_1^2} L_0 + \overline{\sigma_0^2}L_1 + L_1 \overline{\sigma_1^2}\lVert \nabla f(X_t) \rVert_2 \right)},
\end{equation}
by integrating over time and by the Law of the Unconscious Statistician, we have that
\begin{equation}
   \E \left[\lVert \nabla f(X_{\Hat{t}}) \rVert_2^2 \right] \leq \frac{1}{ \phi^{(1)}_t(1-\epsilon)} \left( f(X_0) - f(X_*)  + \frac{\eta \phi^{(2)}_t}{2 N} (L_0+L_1) d \overline{\sigma_0^2} \right) \overset{t \rightarrow \infty}{\rightarrow} 0,
\end{equation}
where $\Hat{t}$, is a random time with distribution $\frac{\eta_{\Hat{t}}}{\phi^{(1)}_t}$.

Finally, for practical reasons, we leverage the distributed setting to tighten the requirements on the learning rate scheduler to make it experimentally viable, and rather require
\begin{equation}
    \eta \eta_t < \frac{2 \epsilon}{L_0+L_1 \E\left[\lVert \nabla f(X_t) \rVert \right]+  \frac{d}{N}\left( \overline{\sigma_1^2} L_0 + \overline{\sigma_0^2}L_1 + L_1 \overline{\sigma_1^2}\E\left[\lVert \nabla f(X_t) \rVert \right] \right)}.
\end{equation}
\end{proof}
\begin{remark}
This condition compares a (deterministic) step-size schedule to a bound that involves $\E\!\left[\lVert \nabla f(X_t) \rVert\right]$, which in turn depends on the (random) training trajectory. As such, it should be understood as a \emph{sufficient} stability criterion that generally cannot be certified \emph{a priori}. Our goal is not to propose an immediately implementable rule, but rather to provide quantitative guidance: the bound makes explicit how the key factors in the dynamics (e.g., curvature, noise, compression) jointly shape the stability region.
\end{remark}

\subsection{Distributed Compressed SGD with Unbiased Compression}

\subsubsection{First Order SDE}

The following is the \textit{first}-order SDE model of DCSGD (see Theorem 3.6 in \cite{compagnoni2025unbiased}). Let us consider the stochastic process $ X_t \in \mathbb{R}^{d} $ defined as the solution of

\begin{equation}\label{eq:DCSGD_SDE_App_First}
    d X_t = - \nabla f(X_t) dt + \sqrt{\frac{\eta}{N}} \sqrt{\Tilde{\Sigma}(X_t)} dW_t,
\end{equation}
where for $\Phi_{\xi_i,\gamma_i}(x) := \mathcal{C}_{\xi_i} \left( \nabla f_{\gamma_i} (x) \right) - \nabla f_{\gamma_i}(x)$
\small 
\begin{equation}
    \Tilde{\Sigma}(x) = \frac{1}{ N} \sum_{i=1}^{ N} \left( \E_{\xi_i \gamma_i} \left[ \Phi_{\xi_i,\gamma_i}(x)\Phi_{\xi_i,\gamma_i}(x)^{\top} \right] + \Sigma_i(x) \right).
\end{equation}

\begin{theorem}
    Let $f$ be $(L_0,L_1)$-smooth, the learning rate scheduler $\eta_t$ such that $\phi^{(i)}_t = \int_0^t (\eta_s)^i ds$, $\phi^{(1)}_t \overset{t \rightarrow \infty}{\rightarrow} \infty$, $\frac{\phi^{(2)}_t}{\phi^{(1)}_t} \overset{t \rightarrow \infty}{\rightarrow} 0$, and $\overline{\sigma^2 \omega}\eqdef \frac{1}{N} \sum_{i=1}^{N} \sigma_i^2 \omega_i$. Then, for $0 < \epsilon <1$,
\begin{equation}\label{eq:StepDCSGD_App_First}
    \eta \eta_t < \frac{2 N \epsilon}{ \overline{\omega} L_0 + \left(\overline{\sigma^2}d  +   d \overline{ \sigma^2 \omega } \right)L_1 +  \overline{\omega} L_1  \E\left[\lVert \nabla f(X_t) \rVert_2 \right] },
\end{equation}
and for a random time $\Hat{t}$ with distribution $\frac{\eta_t}{\phi^{(1)}_t}$, we have that
\begin{equation}
   \E \left[\lVert \nabla f(X_{\Hat{t}}) \rVert_2^2 \right] \leq \frac{1}{ \phi^{(1)}_t (1-\epsilon)} \left( f(X_0) - f(X_*)  + \phi^{(2)}_t\frac{\eta (L_0 + L_1) d \left( \overline{\sigma^2}  +  \overline{ \sigma^2 \omega } \right) }{2 N}  \right) \overset{t \rightarrow \infty}{\rightarrow} 0.
\end{equation}
\end{theorem}
\begin{proof}
Since it holds that $$\E_{\xi_i,\gamma_i} \lVert \left(  \mathcal{C}_{\xi_i} \left( \nabla f_{\gamma_i} (x) \right) - \nabla f(x) \right) \rVert_2^2 \leq  \omega_i \lVert \nabla f(x) \rVert_2^2 + d \sigma_i^2(\omega_i+1),$$ we have that for $\mathcal{O}(\text{Noise}) = \sqrt{\tilde{\Sigma}(X_t)} \nabla f(X_t) d W_t$,
\begin{align}
    d (f(X_t) - f(X_*)) & = - \eta_t \lVert \nabla f(X_t) \rVert_2^2 dt + \mathcal{O}(\text{Noise})\\
    &+ (\eta_t)^2\frac{\eta (L_0+L_1\lVert \nabla f(X_t) \rVert_2)}{2 N} \left(\frac{1}{N} \sum_{i=1}^N  \E_{\xi_i,\gamma_i} \lVert \left(  \mathcal{C}_{\xi_i} \left( \nabla f_{\gamma_i} (x) \right) - \nabla f(x) \right) \rVert_2^2\right) dt \\
    & \leq - \eta_t \lVert \nabla f(X_t) \rVert_2^2 dt + \mathcal{O}(\text{Noise})\\
    &+ (\eta_t)^2\frac{\eta (L_0+L_1\lVert \nabla f(X_t) \rVert_2)}{2 N} \left(\overline{\omega}\lVert \nabla f(X_t) \rVert_2^2 + \overline{\sigma^2}d  +  d \overline{\sigma^2 \omega} \right) dt.
\end{align}

\textbf{Phase 1:} If $\lVert \nabla f(X_t) \rVert_2 \leq 1$, then we have that
\begin{align}
     d (f(X_t) - f(X_*)) & \leq -  \lVert \nabla f(X_t) \rVert_2^2 \left( \eta_t - \frac{\eta (L_0 + L_1) \overline{\omega}}{2 N} (\eta_t)^2 \right)dt \\
    & + (\eta_t)^2 \frac{\eta (L_0 + L_1) d }{2 N} \left( \overline{\sigma^2}  +  \overline{ \sigma^2 \omega } \right)dt + \mathcal{O}(\text{Noise}).
\end{align}

\textbf{Phase 2:} If $\lVert \nabla f(X_t) \rVert_2 >1$, we have that
\begin{align}
    d (f(X_t) - f(X_*)) &  \leq - \eta_t \lVert \nabla f(X_t) \rVert_2^2 dt + \mathcal{O}(\text{Noise})\\
    &+ (\eta_t)^2\frac{\eta (L_0+L_1\lVert \nabla f(X_t) \rVert_2)}{2 N} \left(\overline{\omega}\lVert \nabla f(X_t) \rVert_2^2 + \overline{\sigma^2}d  +   d \overline{ \sigma^2 \omega }  \right) dt \\
    & \leq -\eta_t \lVert \nabla f(X_t) \rVert_2^2 \left( 1 - \frac{\eta_t \eta}{2 N} \left( \overline{\omega} L_0 + d \left(\overline{\sigma^2} + \overline{ \sigma^2 \omega } \right)L_1 +  \overline{\omega} L_1  \lVert \nabla f(X_t) \rVert_2 \right) \right) dt \\
    & + \eta_t^2 \frac{\eta L_0 d}{2 N} \left( \overline{\sigma^2}  +   \overline{ \sigma^2 \omega }  \right) dt +  \mathcal{O}(\text{Noise}).
\end{align}

By taking a worst-case scenario approach, we merge these two bounds into a single one. With arguments that follow the same steps we detailed in the proof of Theorem \ref{thm:DSGD_App}, we have that for $0<\epsilon <1$, we have that if
\begin{equation}
    \eta \eta_t < \frac{2 N \epsilon}{ \overline{\omega} L_0 + d\left(\overline{\sigma^2} + \overline{ \sigma^2 \omega } \right)L_1 +  \overline{\omega} L_1  \lVert \nabla f(X_t) \rVert_2 },
\end{equation}
by integrating over time and by the Law of the Unconscious Statistician, we have that
\begin{equation}
   \E \left[\lVert \nabla f(X_{\Hat{t}}) \rVert_2^2 \right] \leq \frac{1}{ \phi^{(1)}_t (1-\epsilon)} \left( f(X_0) - f(X_*)  + \phi^{(2)}_t\frac{\eta (L_0 + L_1) d \left( \overline{\sigma^2}  +  \overline{ \sigma^2 \omega } \right) }{2 N}  \right) \overset{t \rightarrow \infty}{\rightarrow} 0,
\end{equation}
where $\Hat{t}$, is a random time with distribution $\frac{\eta_{\Hat{t}}}{\phi^{(1)}_t}$.

Finally, for practical reasons, we leverage the distributed setting to tighten the requirements on the learning rate scheduler to make it experimentally viable, and rather require
\begin{equation}
    \eta \eta_t < \frac{2 N \epsilon}{ \overline{\omega} L_0 + \left(\overline{\sigma^2}d  +   d \overline{ \sigma^2 \omega } \right)L_1 +  \overline{\omega} L_1  \E\left[\lVert \nabla f(X_t) \rVert_2 \right] }.
\end{equation}
\end{proof}

Finally, one can generalize this result to cover the $(\sigma_0^2,\sigma_1^2)$-Variance.

\begin{theorem}
    Let $f$ be $(L_0,L_1)$-smooth, $ \max(\Sigma_i(x)) < \sigma_{i,0}^2 + \sigma_{i,1}^2 \lVert \nabla f(x) \rVert_2^2$, the  learning rate scheduler $\eta_t$ such that $\phi^{(i)}_t = \int_0^t (\eta_s)^i ds$, $\phi^{(1)}_t \overset{t \rightarrow \infty}{\rightarrow} \infty$, $\frac{\phi^{(2)}_t}{\phi^{(1)}_t} \overset{t \rightarrow \infty}{\rightarrow} 0$,  $\overline{\sigma_0^2}\eqdef \frac{1}{N} \sum_{i=1}^{N} \sigma_{0,i}^2$, $\overline{\sigma_1^2}\eqdef \frac{1}{N} \sum_{i=1}^{N} \sigma_{1,i}^2$,  $\overline{\sigma_0^2 \omega}\eqdef \frac{1}{N} \sum_{i=1}^{N} \sigma_{i,0}^2 \omega_i$, and $\overline{\sigma_1^2 \omega}\eqdef \frac{1}{N} \sum_{i=1}^{N} \sigma_{i,1}^2 \omega_i$. Then, for $0 < \epsilon <1$,
\begin{equation}
    \eta \eta_t < \frac{2 N \epsilon}{  L_0(\overline{\omega} + d (\overline{\sigma_1^2 \omega} + \overline{\sigma_1^2})) +  L_1 d\left(\overline{\sigma_0^2}  +  \overline{ \sigma_0^2 \omega } \right) +  L_1(\overline{\omega} + d(\overline{\sigma_1^2 \omega} + \overline{\sigma_1^2}))   \E\left[\lVert \nabla f(X_t) \rVert_2 \right] },
\end{equation}
and for a random time $\Hat{t}$ with distribution $\frac{\eta_t}{\phi^{(1)}_t}$, we have that
\begin{equation}
   \E \left[\lVert \nabla f(X_{\Hat{t}}) \rVert_2^2 \right] \leq \frac{1}{(1-\epsilon) \phi^{(1)}_t} \left( f(X_0) - f(X_*)  + \phi^{(2)}_t \frac{  L_0(\overline{\omega} + d (\overline{\sigma_1^2 \omega} + \overline{\sigma_1^2})) +  L_1 d\left(\overline{\sigma_0^2}  +  \overline{ \sigma_0^2 \omega } \right)}{2 N}\right) \overset{t \rightarrow \infty}{\rightarrow} 0.
\end{equation}
\end{theorem}

\subsubsection{Our New First-Order SDE for DCSGD }

The following is the \textit{first}-order SDE model of DCSGD and is a straightforward generalization of Theorem 3.6 in \cite{compagnoni2025unbiased} and Remark \ref{rem:NewModels}. Let us consider the stochastic process $ X_t \in \mathbb{R}^{d} $ defined as the solution of
\begin{equation}\label{eq:DCSGD_SDE_App_Second}
    d X_t = - \nabla f(X_t) dt + \frac{\eta}{2} \nabla^2 f(X_t) \nabla f(X_t) dt + \sqrt{\frac{\eta}{N}} \sqrt{\Tilde{\Sigma}(X_t)} dW_t,
\end{equation}
where for $\Phi_{\xi_i,\gamma_i}(x) := \mathcal{C}_{\xi_i} \left( \nabla f_{\gamma_i} (x) \right) - \nabla f_{\gamma_i}(x)$
\small 
\begin{equation}
    \Tilde{\Sigma}(x) = \frac{1}{ N} \sum_{i=1}^{ N} \left( \E_{\xi_i \gamma_i} \left[ \Phi_{\xi_i,\gamma_i}(x)\Phi_{\xi_i,\gamma_i}(x)^{\top} \right] + \Sigma_i(x) \right).
\end{equation}
\begin{theorem}
    Let $f$ be $(L_0,L_1)$-smooth, the learning rate scheduler $\eta_t$ such that $\phi^{(i)}_t = \int_0^t (\eta_s)^i ds$, $\phi^{(1)}_t \overset{t \rightarrow \infty}{\rightarrow} \infty$, $\frac{\phi^{(2)}_t}{\phi^{(1)}_t} \overset{t \rightarrow \infty}{\rightarrow} 0$, and $\overline{\sigma^2 \omega}\eqdef \frac{1}{N} \sum_{i=1}^{N} \sigma_i^2 \omega_i$. Then, for $0 < \epsilon <1$,
\begin{equation}\label{eq:StepDCSGD_App_Second}
    \eta \eta_t < \frac{2 \epsilon}{L_0+L_1 \E \left[\lVert \nabla f(X_t) \rVert_2 \right] + \frac{\overline{\omega} L_0 + d \left(\overline{\sigma^2} + \overline{ \sigma^2 \omega } \right)L_1 +  \overline{\omega} L_1  \E \left[\lVert \nabla f(X_t) \rVert_2 \right]}{N} },
\end{equation}
and for a random time $\Hat{t}$ with distribution $\frac{\eta_t}{\phi^{(1)}_t}$, we have that
\begin{equation}
   \E \left[\lVert \nabla f(X_{\Hat{t}}) \rVert_2^2 \right] \leq \frac{1}{ \phi^{(1)}_t (1-\epsilon)} \left( f(X_0) - f(X_*)  + \phi^{(2)}_t \frac{\eta (L_0+L_1) d}{2 N} \left( \overline{\sigma^2}  +   \overline{ \sigma^2 \omega }  \right) \right) \overset{t \rightarrow \infty}{\rightarrow} 0.
\end{equation}
\end{theorem}
\begin{proof}
Since it holds that $$\E_{\xi_i,\gamma_i} \lVert \left(  \mathcal{C}_{\xi_i} \left( \nabla f_{\gamma_i} (x) \right) - \nabla f(x) \right) \rVert_2^2 \leq  \omega_i \lVert \nabla f(x) \rVert_2^2 + d \sigma_i^2(\omega_i+1),$$ we have that for $\mathcal{O}(\text{Noise}) = \sqrt{\tilde{\Sigma}(X_t)} \nabla f(X_t) d W_t$,
\begin{align}
    d (f(X_t) - f(X_*)) & = - \eta_t \lVert \nabla f(X_t) \rVert_2^2 dt + \frac{\eta \eta_t^2}{2} \left( \nabla f(X_t) \right)^{\top} \nabla^2 f(X_t) \nabla f(X_t) dt + \mathcal{O}(\text{Noise})\\
    & + \frac{\eta \eta_t^2}{2}\frac{ (L_0+L_1\lVert \nabla f(X_t) \rVert_2)}{N} \left(\frac{1}{N} \sum_{i=1}^N  \E_{\xi_i,\gamma_i} \lVert \left(  \mathcal{C}_{\xi_i} \left( \nabla f_{\gamma_i} (x) \right) - \nabla f(x) \right) \rVert_2^2\right) dt \\
    & \leq - \eta_t \lVert \nabla f(X_t) \rVert_2^2 dt + \frac{\eta \eta_t^2}{2} (L_0+L_1\lVert \nabla f(X_t) \rVert) \lVert \nabla f(X_t) \rVert^2 dt +  \mathcal{O}(\text{Noise})\\
    &+ \frac{\eta \eta_t^2}{2}\frac{ (L_0+L_1\lVert \nabla f(X_t) \rVert_2)}{ N} \left(\overline{\omega}\lVert \nabla f(X_t) \rVert_2^2 + \overline{\sigma^2}d  +  d \overline{\sigma^2 \omega} \right) dt.
\end{align}

\textbf{Phase 1:} If $\lVert \nabla f(X_t) \rVert_2 \leq 1$, then we have that
\begin{align}
     d (f(X_t) - f(X_*)) & \leq -\lVert \nabla f(X_t) \rVert_2^2 \left( \eta_t - \frac{\eta_t^2 \eta}{2} (L_0 + L_1) \left( 1 + \frac{  \overline{\omega}}{ N} \right)  \right) dt \\
    & + (\eta_t)^2 \frac{\eta (L_0 + L_1) d }{2 N} \left( \overline{\sigma^2}  +  \overline{ \sigma^2 \omega } \right)dt +  \mathcal{O}(\text{Noise}).
\end{align}
\textbf{Phase 2:} If $\lVert \nabla f(X_t) \rVert_2 >1$, we have that
\begin{align}
    d (f(X_t) - f(X_*)) &  \leq - \eta_t \lVert \nabla f(X_t) \rVert_2^2 dt + \frac{\eta \eta_t^2}{2} (L_0+L_1\lVert \nabla f(X_t) \rVert) \lVert \nabla f(X_t) \rVert^2 dt+  \mathcal{O}(\text{Noise})\\
    &+ (\eta_t)^2\frac{\eta (L_0+L_1\lVert \nabla f(X_t) \rVert_2)}{2 N} \left(\overline{\omega}\lVert \nabla f(X_t) \rVert_2^2 + \overline{\sigma^2}d  +   d \overline{ \sigma^2 \omega }  \right) dt \\
    & \leq -\eta_t \lVert \nabla f(X_t) \rVert_2^2 \left[ 1 - \frac{\eta_t \eta}{2 } \left[ (L_0+L_1\lVert \nabla f(X_t) \rVert_2)\left[1 + \frac{\overline{\omega}}{N}\right] + \frac{d \left(\overline{\sigma^2} + \overline{ \sigma^2 \omega } \right)L_1}{N}  \right] \right] \nonumber\\
    & + \eta_t^2 \frac{\eta L_0 d}{2 N} \left( \overline{\sigma^2}  +   \overline{ \sigma^2 \omega }  \right) +  \mathcal{O}(\text{Noise}).
\end{align}
By taking a worst-case scenario approach, we merge these two bounds into a single one. With arguments that follow the same steps we detailed in the proof of Theorem \ref{thm:DSGD_App}, we have that for $0<\epsilon <1$, we have that if
\begin{equation}
    \eta \eta_t < \frac{2 \epsilon}{L_0+L_1\lVert \nabla f(X_t) \rVert_2 + \frac{\overline{\omega} L_0 + d \left(\overline{\sigma^2} + \overline{ \sigma^2 \omega } \right)L_1 +  \overline{\omega} L_1  \lVert \nabla f(X_t) \rVert_2}{N} },
\end{equation}
by integrating over time and by the Law of the Unconscious Statistician, we have that
\begin{equation}
   \E \left[\lVert \nabla f(X_{\Hat{t}}) \rVert_2^2 \right] \leq \frac{1}{ \phi^{(1)}_t (1-\epsilon)} \left( f(X_0) - f(X_*)  + \phi^{(2)}_t \frac{\eta (L_0+L_1) d}{2 N} \left( \overline{\sigma^2}  +   \overline{ \sigma^2 \omega }  \right) \right) \overset{t \rightarrow \infty}{\rightarrow} 0,
\end{equation}
where $\Hat{t}$, is a random time with distribution $\frac{\eta_{\Hat{t}}}{\phi^{(1)}_t}$.

Finally, for practical reasons, we leverage the distributed setting to tighten the requirements on the learning rate scheduler to make it experimentally viable, and rather require
\begin{equation}
    \eta \eta_t < \frac{2 \epsilon}{L_0+L_1 \E \left[\lVert \nabla f(X_t) \rVert_2 \right] + \frac{\overline{\omega} L_0 + d \left(\overline{\sigma^2} + \overline{ \sigma^2 \omega } \right)L_1 +  \overline{\omega} L_1  \E \left[\lVert \nabla f(X_t) \rVert_2 \right]}{N} }.
\end{equation}
\end{proof}

Finally, one can generalize this result to cover the $(\sigma_0^2,\sigma_1^2)$-Variance.

\begin{theorem}
    Let $f$ be $(L_0,L_1)$-smooth, $ \max(\Sigma_i(x)) < \sigma_{i,0}^2 + \sigma_{i,1}^2 \lVert \nabla f(x) \rVert_2^2$, the  learning rate scheduler $\eta_t$ such that $\phi^{(i)}_t = \int_0^t (\eta_s)^i ds$, $\phi^{(1)}_t \overset{t \rightarrow \infty}{\rightarrow} \infty$, $\frac{\phi^{(2)}_t}{\phi^{(1)}_t} \overset{t \rightarrow \infty}{\rightarrow} 0$,  $\overline{\sigma_0^2}\eqdef \frac{1}{N} \sum_{i=1}^{N} \sigma_{0,i}^2$, $\overline{\sigma_1^2}\eqdef \frac{1}{N} \sum_{i=1}^{N} \sigma_{1,i}^2$,  $\overline{\sigma_0^2 \omega}\eqdef \frac{1}{N} \sum_{i=1}^{N} \sigma_{i,0}^2 \omega_i$, and $\overline{\sigma_1^2 \omega}\eqdef \frac{1}{N} \sum_{i=1}^{N} \sigma_{i,1}^2 \omega_i$. Then, for $0 < \epsilon <1$,
\begin{equation}
    \eta \eta_t < \frac{2 \epsilon}{ L_0 + L_1 \E\left[\lVert \nabla f(X_t) \rVert_2 \right] +  \frac{L_0(\overline{\omega} + d (\overline{\sigma_1^2 \omega} + \overline{\sigma_1^2})) +  L_1 d\left(\overline{\sigma_0^2}  +  \overline{ \sigma_0^2 \omega } \right) +  L_1(\overline{\omega} + d(\overline{\sigma_1^2 \omega} + \overline{\sigma_1^2}))   \E\left[\lVert \nabla f(X_t) \rVert_2 \right]}{N}   },
\end{equation}
and for a random time $\Hat{t}$ with distribution $\frac{\eta_t}{\phi^{(1)}_t}$, we have that
\begin{equation}
   \E \left[\lVert \nabla f(X_{\Hat{t}}) \rVert_2^2 \right] \leq \frac{1}{(1-\epsilon) \phi^{(1)}_t} \left( f(X_0) - f(X_*)  + \phi^{(2)}_t \frac{\eta (L_0+L_1)d(\overline{\sigma_0^2} + \overline{\sigma_0^2 \omega})}{2 N}\right) \overset{t \rightarrow \infty}{\rightarrow} 0.
\end{equation}
\end{theorem}

\subsection{Distributed SignSGD}

\subsubsection{First Order SDE}

The following is the \textit{first}-order SDE model of DSignSGD (see Theorem 3.10 in \cite{compagnoni2025unbiased}). Let us consider the stochastic process $ X_t \in \mathbb{R}^{d} $ defined as the solution of

\begin{equation}\label{eq:DSignSGD_SDE}
   d X_t = - \frac{1}{N} \sum_{i=1}^{N} \left( 1 - 2 \mathbb{P}(\nabla f_{\gamma_i} (X_t) <0) \right) dt + \sqrt{\frac{\eta}{N}}\sqrt{\overline{\Sigma}(X_t)} dW_t.
\end{equation}
where
\begin{equation}\label{eq:DSignSGD_CovMatr}
    \overline{\Sigma}(X_t) := \frac{1}{N} \sum_{i=1}^{N} \overline{\Sigma_i}(X_t),
\end{equation}
and  $ \overline{\Sigma_i}(x) = \E[\xi_{\gamma_i}(x)\xi_{\gamma_i}(x)^\top]$ where $\xi_{\gamma_i}(x):= \sign (\nabla f_{\gamma_i}(x)) - 1 + 2 \mathbb{P}(\nabla f_{\gamma_i}(x)<0)$ the noise in the sample $ \sign \left(\nabla f_{\gamma_i}(x) \right)$.

\begin{corollary}[Corollary C.10 in \cite{compagnoni2025unbiased}]\label{thm:DSignSGD_Theorem_Student_App}
If the stochastic gradients are $\nabla f_{\gamma_i}(x) = \nabla f(x) +  \sqrt{\Sigma_i} Z_i$ such that $Z_i \sim t_{\nu}(0, I_d) $ does not depend on $x$, $\nu$ are the degrees of freedom, and scale matrices $\Sigma_i= \diag(\sigma_{1,i}^2, \cdots, \sigma_{d,i}^2)$. Then, the SDE of DSignSGD is
\begin{equation}\label{eq:SDE_HSignSGD_Full_App}
    d X_t = - \frac{2}{N} \sum_{i=1}^{N} \Xi_{\nu} \left( \Sigma_i^{-\frac{1}{2}} \nabla f(X_t) \right) dt + \sqrt{\frac{\eta}{N}}\sqrt{\Tilde{\Sigma}(X_t)} dW_t.
\end{equation}
where $\Xi_{\nu}(x)$ is defined as $\Xi_{\nu}(x) := x \frac{\Gamma\left(\frac{\nu+1}{2}\right)}{\sqrt{\pi \nu} \Gamma\left(\frac{\nu}{2}\right)}{ }_2 F_1\left(\frac{1}{2}, \frac{\nu+1}{2} ; \frac{3}{2} ;-\frac{x^2}{\nu}\right)$, ${ }_2 F_1\left(a, b;c; x\right)$ is the hypergeometric function, and 
\begin{equation}
    \Tilde{\Sigma}(X_t) := I_d -  \frac{4}{N}\sum_{i=1}^{N}  \left(\Xi_{\nu} \left( \Sigma_i^{-\frac{1}{2}} \nabla f(X_t) \right) \right)^2.
\end{equation}
\end{corollary}

\paragraph{\textbf{Remark (Beyond Student's $t$ noise).}}
We specialize our analysis to the Student's $t$ noise to obtain a closed-form expression for the scalar function $\Xi_\nu$ governing the drift of the DSignSGD surrogate (and hence explicit constants $M_\nu,\ell_\nu$).
More generally, the same SDE construction applies to any coordinate-wise symmetric noise with an absolutely continuous density: the drift depends on the one-dimensional CDF through $\Xi(x)\coloneqq F(x)-\tfrac12$, and the constants in the stability bound can be defined as $M\coloneqq\sup_x \Xi'(x)$ and $\ell\coloneqq 2\Xi'(0)$ whenever these quantities are finite.
The Student's $t$ family is a convenient heavy-tailed instantiation that also covers regimes with undefined mean (e.g., $\nu\le 1$).

\begin{theorem}\label{thm:sign_1}
    Let $f$ be $(L_0,L_1)$-smooth, $\eta_t$ a  learning rate scheduler such that $\phi^{(i)}_t = \int_0^t (\eta_s)^i ds$, $\phi^{(1)}_t \overset{t \rightarrow \infty}{\rightarrow} \infty$, $\frac{\phi^{(2)}_t}{\phi^{(1)}_t} \overset{t \rightarrow \infty}{\rightarrow} 0$, $\Sigma_i \leq \sigma_{\text{max}, i}^2$, $\sigma_{\mathcal{H},1}$ be the harmonic mean of $\{\sigma_{\text{max}, i} \}$, and $\ell_{\mathbf{\nu}}:= 2\Xi_{\nu}^{'}(0)>0$ a constant. Then, for a scheduler $\eta \eta_t < \frac{2 N \ell_{\nu}}{\sigma_{\mathcal{H},1} d L_1}$ and a random time $\Tilde{t}$ with distribution $\frac{\eta_t \ell_{\nu} \sigma_{\mathcal{H},1}^{-1} - \eta_t^2 \frac{\eta L_1 d}{2 N}}{\phi^{(1)}_t \ell_{\nu} \sigma_{\mathcal{H},1}^{-1} - \phi^{(2)}_t \frac{\eta L_1 d}{2 N}}$, we have that
    \begin{equation}
         \E \lVert \nabla f\left(X_{\Tilde{t}}\right)\rVert_2^2 \leq \frac{ 1}{\phi^{(1)}_t \ell_{\nu} \sigma_{\mathcal{H},1}^{-1} - \phi^{(2)}_t \frac{\eta L_1 d}{2 N}} \left( f(X_0) - f(X_*) + \frac{\eta (L_0+L_1) d \phi^{(2)}_t}{2 N} \right) \overset{t \rightarrow \infty}{\rightarrow} 0.
    \end{equation}
\end{theorem}

\begin{proof}
By Itô Lemma on $f(X_t) - f(X_*)$, we have that for $\mathcal{O}(\text{Noise}) = \sqrt{\Bar{\Sigma}(X_t)} \nabla f(X_t) d W_t$,
\begin{align}
    d(f(X_t) - f(X_*)) & \leq - \ell_{\nu} \sigma_{\mathcal{H},1}^{-1} \eta_t \lVert \nabla f(X_t) \rVert_2^2 dt  + \frac{\eta \eta_t^2 d}{2 N} (L_0 + L_1 \lVert \nabla f(X_t) \rVert_2 ) dt +  \mathcal{O}(\text{Noise})
\end{align}
\textbf{Phase 1: } $\lVert \nabla f(X_t) \rVert_2 \leq 1$:
\begin{align}
    d(f(X_t) - f(X_*)) & \leq - \ell_{\nu} \sigma_{\mathcal{H},1}^{-1} \eta_t \lVert \nabla f(X_t) \rVert_2^2 dt  + \frac{\eta \eta_t^2 d}{2 N} (L_0 + L_1  ) dt +  \mathcal{O}(\text{Noise}).
\end{align}
\textbf{Phase 2: } $\lVert \nabla f(X_t) \rVert_2 > 1$:
\begin{align}
    d(f(X_t) - f(X_*)) & \leq - \ell_{\nu} \sigma_{\mathcal{H},1}^{-1} \eta_t \lVert \nabla f(X_t) \rVert_2^2 dt  + \frac{\eta \eta_t^2 d L_1 \lVert \nabla f(X_t) \rVert_2^2}{2 N} + \frac{\eta \eta_t^2 d L_0}{2 N}  dt +  \mathcal{O}(\text{Noise}).
\end{align}
By taking the worst case of these two phases, we have that 
\begin{align}
    d(f(X_t) - f(X_*)) & \leq - \ell_{\nu} \sigma_{\mathcal{H},1}^{-1} \eta_t \lVert \nabla f(X_t) \rVert_2^2 dt + \frac{\eta \eta_t^2 d L_1 \lVert \nabla f(X_t) \rVert_2^2}{2 N} dt + \frac{\eta \eta_t^2 d}{2 N} (L_0 + L_1  ) dt +  \mathcal{O}(\text{Noise}).
\end{align}
With arguments that follow the same steps we detailed in the proof of Theorem \ref{thm:DSGD_App}, we have that
\begin{equation}
     \E \lVert \nabla f\left(X_{\Tilde{t}}\right)\rVert_2^2 \leq \frac{ 1}{\phi^{(1)}_t \ell_{\nu} \sigma_{\mathcal{H},1}^{-1} - \phi^{(2)}_t \frac{d \eta L_1}{2 N}} \left( f(X_0) - f(X_*) + \frac{\eta (L_0+L_1)  d \phi^{(2)}_t}{2 N} \right) \overset{t \rightarrow \infty}{\rightarrow} 0.
\end{equation}
\end{proof}

\subsubsection{Our New First-Order SDE for DSignSGD}

The following is the \textit{first}-order SDE model of DSignSGD and is a straightforward generalization of Corollary C.10 in \cite{compagnoni2025unbiased} and Remark \ref{rem:NewModels}. We observe that $\Xi_{\nu}^{'}(x)$ is bounded by the positive finite constant $M_{\nu}$.
\begin{align}\label{eq:SDE_HSignSGD_Full_App_SO}
    d X_t & = - \frac{2}{N} \sum_{i=1}^{N} \Xi_{\nu} \left( \Sigma_i^{-\frac{1}{2}} \nabla f(X_t) \right) dt \nonumber \\
    & + \frac{\eta}{N} \sum_{i=1}^{N} \Sigma_i^{-\frac{1}{2}} \nabla^2 f(X_t) \left(  \Xi_{\nu}^{'} \left( \Sigma_i^{-\frac{1}{2}} \nabla f(X_t) \right) \circ \Xi_{\nu} \left( \Sigma_i^{-\frac{1}{2}} \nabla f(X_t) \right)\right) dt \nonumber \\
    & + \sqrt{\frac{\eta}{N}}\sqrt{\Tilde{\Sigma}(X_t)} dW_t.
\end{align}

\begin{theorem}\label{thm:DSignSGD_App}
    Let $f$ be $(L_0,L_1)$-smooth, $\Sigma_i \leq \sigma_{\text{max}, i}^2$, $\sigma_{\mathcal{H},1}$ be the harmonic mean of $\{\sigma_{\text{max}, i} \}$, $M_{\nu}:=\sup \{\Xi_{\nu}^{'}(x)\}>0$ and $\ell_{\mathbf{\nu}}:=2\Xi_{\nu}^{'}(0)>0$ constants, and $K \eqdef \left(\frac{L_1}{2 N} + \frac{(L_0+L_1)\sigma_{\mathcal{H},1}^{-1} M_{\nu}}{\sqrt{d}}\right)$. Then, for a scheduler $\eta \eta_t <\frac{\ell_{\nu} K^{-1}}{\sigma_{\mathcal{H},1} d} $ and a random time $\Tilde{t}$ with distribution $\frac{\eta_t \ell_{\nu} \sigma_{\mathcal{H},1}^{-1} - \eta_t^2 K}{\phi^{(1)}_t \ell_{\nu} \sigma_{\mathcal{H},1}^{-1} - \phi^{(2)}_t K}$, we have that
\begin{equation}
     \E \lVert \nabla f\left(X_{\Tilde{t}}\right)\rVert_2^2 \leq \frac{ 1}{\phi^{(1)}_t \ell_{\nu} \sigma_{\mathcal{H},1}^{-1} - \phi^{(2)}_t K} \left( f(X_0) - f(X_*) + \phi^{(2)}_t \eta (L_0+L_1)d \left( \frac{1}{2 N} + \frac{M_{\nu}}{\sigma_{\mathcal{H},1} \sqrt{d}} \right) \right) \overset{t \rightarrow \infty}{\rightarrow} 0.
\end{equation}
\end{theorem}
\begin{proof}
By Itô Lemma on $f(X_t) - f(X_*)$, we have that {for $\mathcal{O}(\text{Noise}) = \sqrt{\Bar{\Sigma}(X_t)} \nabla f(X_t) d W_t$,}
\begin{align}
    d(f(X_t) - f(X_*)) & \leq - \ell_{\nu} \sigma_{\mathcal{H},1}^{-1} \eta_t \lVert \nabla f(X_t) \rVert_2^2 dt + \eta \eta_t^2 \sigma_{\mathcal{H},1}^{-1} (L_0 + L_1 \lVert \nabla f(X_t) \rVert_2 ) M_{\nu} \lVert \nabla f(X_t) \rVert_1 dt \\
    & + \frac{\eta \eta_t^2 d}{2 N} (L_0 + L_1 \lVert \nabla f(X_t) \rVert_2 ) dt +  \mathcal{O}(\text{Noise}).
\end{align}
\textbf{Phase 1: } $\lVert \nabla f(X_t) \rVert_2 \leq 1$:
\begin{align}
    d(f(X_t) - f(X_*)) & \leq - \ell_{\nu} \sigma_{\mathcal{H},1}^{-1} \eta_t \lVert \nabla f(X_t) \rVert_2^2 dt + \eta \eta_t^2 \sigma_{\mathcal{H},1}^{-1} (L_0 + L_1  ) M_{\nu} \sqrt{d} dt \\
    & + \frac{\eta \eta_t^2 d}{2 N} (L_0 + L_1  ) dt +  \mathcal{O}(\text{Noise}).
\end{align}
\textbf{Phase 2: } $\lVert \nabla f(X_t) \rVert_2 > 1$:
Since $\lVert \nabla f(X_t) \rVert_1 < \sqrt{d}\lVert \nabla f(X_t) \rVert_2 < \sqrt{d} \lVert \nabla f(X_t) \rVert_2^2$, we have that
\begin{align}
    d(f(X_t) - f(X_*)) & \leq - \ell_{\nu} \sigma_{\mathcal{H},1}^{-1} \eta_t \lVert \nabla f(X_t) \rVert_2^2 dt + \eta \eta_t^2 \sigma_{\mathcal{H},1}^{-1} (L_0 + L_1  ) M_{\nu} \sqrt{d} \lVert \nabla f(X_t) \rVert_2^2 dt \\
    & + \frac{\eta \eta_t^2 d L_1 \lVert \nabla f(X_t) \rVert_2^2}{2 N} + \frac{\eta \eta_t^2 d L_0}{2 N}  dt +  \mathcal{O}(\text{Noise}).
\end{align}
By taking the worst case of these two phases, we have that 
\begin{align}
    d(f(X_t) - f(X_*)) & \leq - \ell_{\nu} \sigma_{\mathcal{H},1}^{-1} \eta_t \lVert \nabla f(X_t) \rVert_2^2 dt + \eta \eta_t^2 \sigma_{\mathcal{H},1}^{-1} (L_0 + L_1  ) M_{\nu} \sqrt{d} \lVert \nabla f(X_t) \rVert_2^2 dt \\
    & + \frac{\eta \eta_t^2 d L_1 \lVert \nabla f(X_t) \rVert_2^2}{2 N} dt + \eta \eta_t^2(L_0+L_1)d \left( \frac{1}{2 N} + \frac{M_{\nu}}{\sigma_{\mathcal{H},1} \sqrt{d}} \right)  dt +  \mathcal{O}(\text{Noise}).
\end{align}
With arguments that follow the same steps we detailed in the proof of Theorem \ref{thm:DSGD_App}, we have that
\begin{equation}
     \E \lVert \nabla f\left(X_{\Tilde{t}}\right)\rVert_2^2 \leq \frac{ 1}{\phi^{(1)}_t \ell_{\nu} \sigma_{\mathcal{H},1}^{-1} - \phi^{(2)}_t K} \left( f(X_0) - f(X_*) + \phi^{(2)}_t \eta (L_0+L_1)d \left( \frac{1}{2 N} + \frac{M_{\nu}}{\sigma_{\mathcal{H},1} \sqrt{d}} \right) \right) \overset{t \rightarrow \infty}{\rightarrow} 0.
\end{equation}
\end{proof}

\subsection{Limitations}
\label{sec:limitations}

Our analysis focuses on \emph{homogeneous} client distributions to isolate the effects of noise, compression, and adaptivity without the additional complexity of data heterogeneity. Our implementation-oriented normalization discussion also assumes a server-aggregated topology, in which a central server collects client updates and can aggregate a scalar norm estimate. Fully decentralized topologies are outside the scope of this work and would require separate communication and estimation mechanisms. Extending the results to heterogeneous settings---where clients may have different tail indices, variance structures, or asymmetric noise distributions---is an important direction for future work. We also restrict attention to \emph{unbiased} and \emph{signed} gradient compression, while many practical distributed optimizers employ general \emph{biased} compressors or use \emph{error-feedback} mechanisms to recover convergence guarantees. Extending our SDE framework to these settings would require augmenting the continuous-time state with the error-feedback memory, or introducing suitable bias-correction terms in the drift.

Additionally, we focus on finite-sum minimization, as is common in the SDE literature for stochastic optimization~\citep{jastrzkebski2017three}, and do not tackle questions related to generalization~\citep{smith2020sde}.

Finally, our contribution is intentionally foundational: rather than proposing new optimizers, we build a rigorous, unified framework that captures the joint effects of noise, compression, and adaptivity for distributed methods under $(L_0,L_1)$-smoothness. We view this work as a basis for future extensions (e.g., heterogeneous clients, error-feedback, and general biased compressors) and for subsequent analyses that further systematize large-scale stochastic optimization.

\paragraph{Acknowledgments.} 
We acknowledge the use of OpenAI's ChatGPT as a writing assistant to help us rephrase and refine parts of the manuscript. All technical content, derivations, and scientific contributions remain the sole responsibility of the authors.

\section{Experiments}\label{sec:Exper}

Our experiments are intentionally minimalistic: They are designed to validate the fidelity of the derived insights and to illustrate the qualitative phenomena predicted by our theory, rather than to benchmark performance on specific tasks. This aligns with the theoretical nature of our contribution.

\subsection{DCSGD - Figure \ref{fig:InsightValidation} - (Left Column)}

We optimize $f(x) = \frac{\sum_{j=1}^{1000}(x_j)^4}{4}$ as we inject Gaussian noise with mean $0$ and variance $\sigma^2 \lVert \nabla f(x) \rVert_2^2$ on the gradient. The learning rate is $\eta = 0.1$, $\sigma = 0.1$, we use \textit{random sparsification} with $\omega\in\{4,8,16 \}$, and we average over $1000$ runs. In the top figure, we use no scheduler, while in the bottom one we use a scheduler as per Eq. \ref{eq:StepDCSGD_UV}.

\subsection{DSignSGD - Figure \ref{fig:InsightValidation} - (Right Column)}

We optimize $f(x) = \frac{x^4}{4}$ as we inject student's t noise with $\nu=1$ and scale parameters $\sigma$ on the gradient. The learning rate is $\eta = 0.1$, $\sigma \in \{0.25,0.5,1,2,8,16\}$, and we average over $10000$ runs. In the top figure, we use no scheduler, while in the bottom one we use a scheduler as per Theorem \ref{thm:sign}, e.g. $\eta_t = \frac{1}{\sqrt{t+1}}$.

\subsection{Experiments on a MLP (Figure \ref{fig:InsightValidation_MLP})}

\paragraph{Architecture, data, and metric.}
We consider fully connected neural network $h_{\theta}: \R^{20} \to \R$ with a single hidden layer of width $64$ and ReLU activation. The network has the form
\begin{equation}
h_{\theta}(x) = W_2 \,\phi(W_1 x + b_1) + b_2,
\end{equation}
where $W_1 \in \R^{64 \times 20}$, $W_2 \in \R^{1 \times 64}$, $b_1 \in \R^{64}$, $b_2 \in \R$, and $\phi$ denotes the ReLU activation. We collect all parameters into a single vector $\theta \in \R^{d}$. The loss is the mean squared error
\begin{equation}
f(\theta)
=
\frac{1}{n}\sum_{i=1}^n \bigl(h_{\theta}(x_i) - y_i\bigr)^2,
\end{equation}
with a fixed dataset of $n = 4096$ examples. Inputs are $x_i \in \R^{20}$ sampled as $x_i \sim \mathcal N(0, I_{20})$, and labels are generated from a linear teacher
\begin{equation}
y_i = x_i^{\top} w_\star + \varepsilon_i,
\qquad
w_\star \sim \mathcal N(0, I_{20}),\quad
\varepsilon_i \sim \mathcal N(0, 0.1^2).
\end{equation}
The dataset $(x_i,y_i)_{i=1}^n$ is sampled once and reused for all methods and repetitions, while the number of clients is $N=8$. At iteration $t$, we compute the full-batch gradient $g_t = \nabla_{\theta} f(\theta_t) \in \R^{d}$ and monitor the quantity
\begin{equation}
\|g_t\|_2^2.
\end{equation}
For each setting, we approximate the expectation by averaging this quantity over multiple independent runs.

\begin{figure}[ht!]
   \hspace{.2cm}
    {\includegraphics[width=0.48\linewidth]{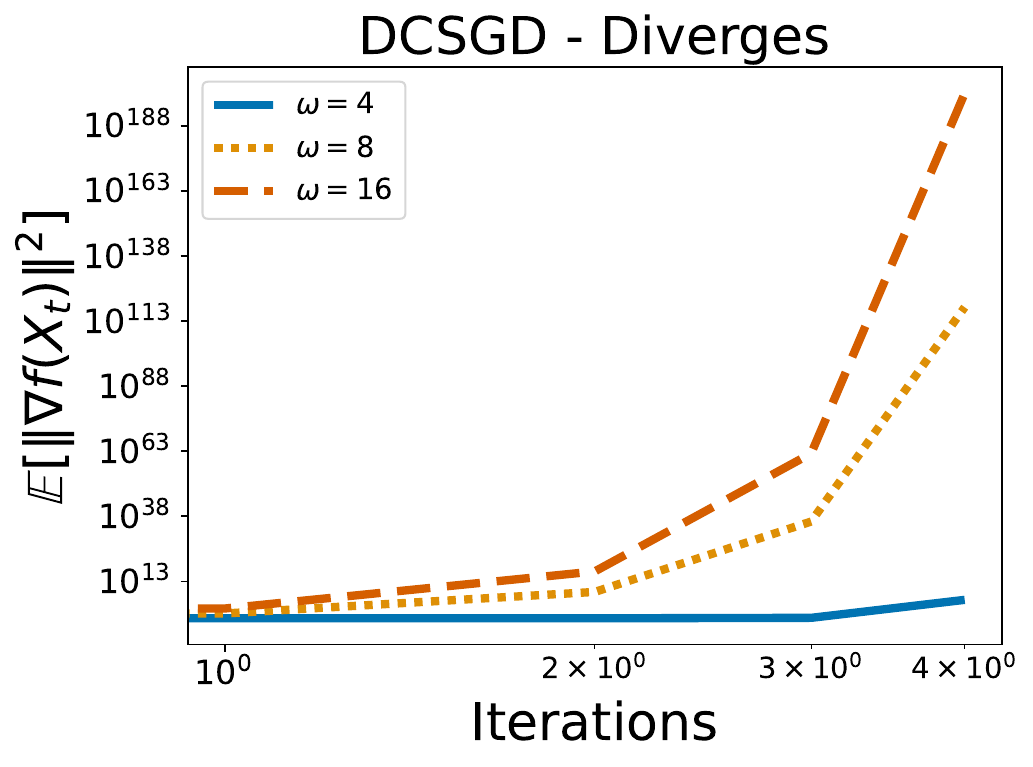} }%
    {\includegraphics[width=0.47\linewidth]{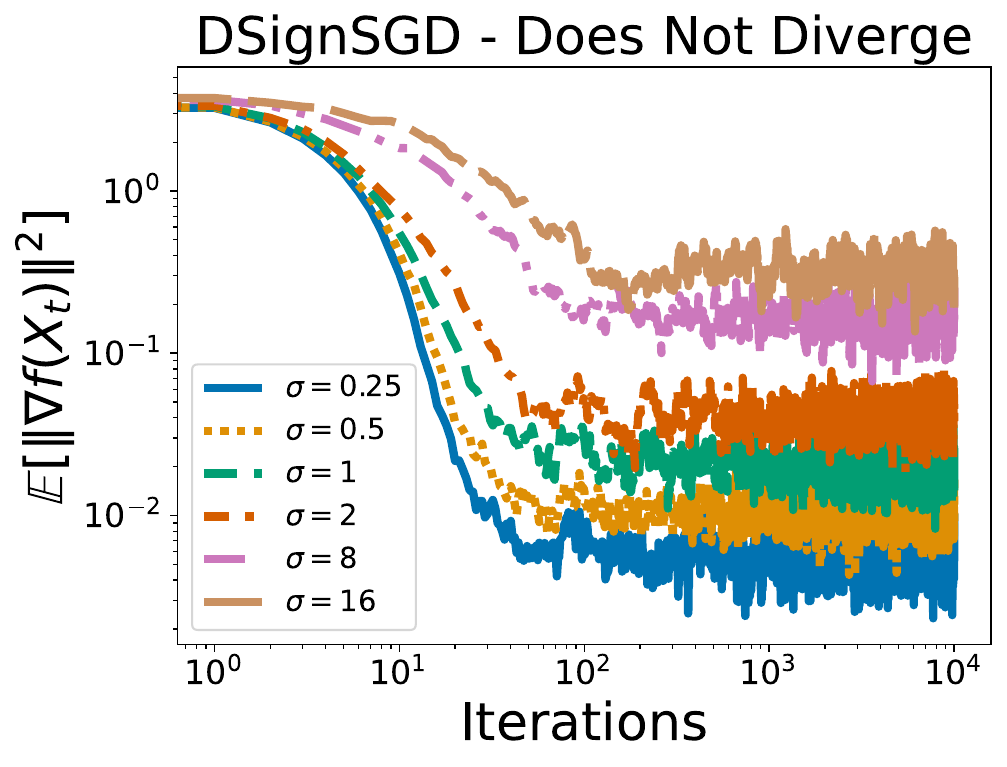} } \\
    {\includegraphics[width=0.48\linewidth]{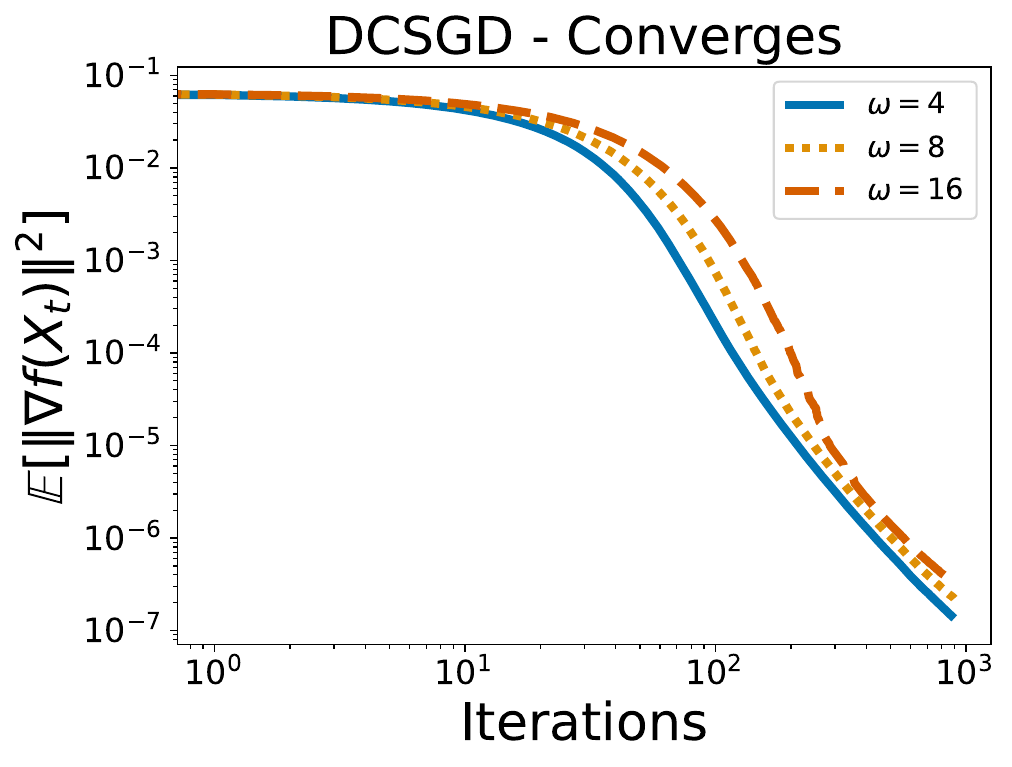} }%
    {\includegraphics[width=0.48\linewidth]{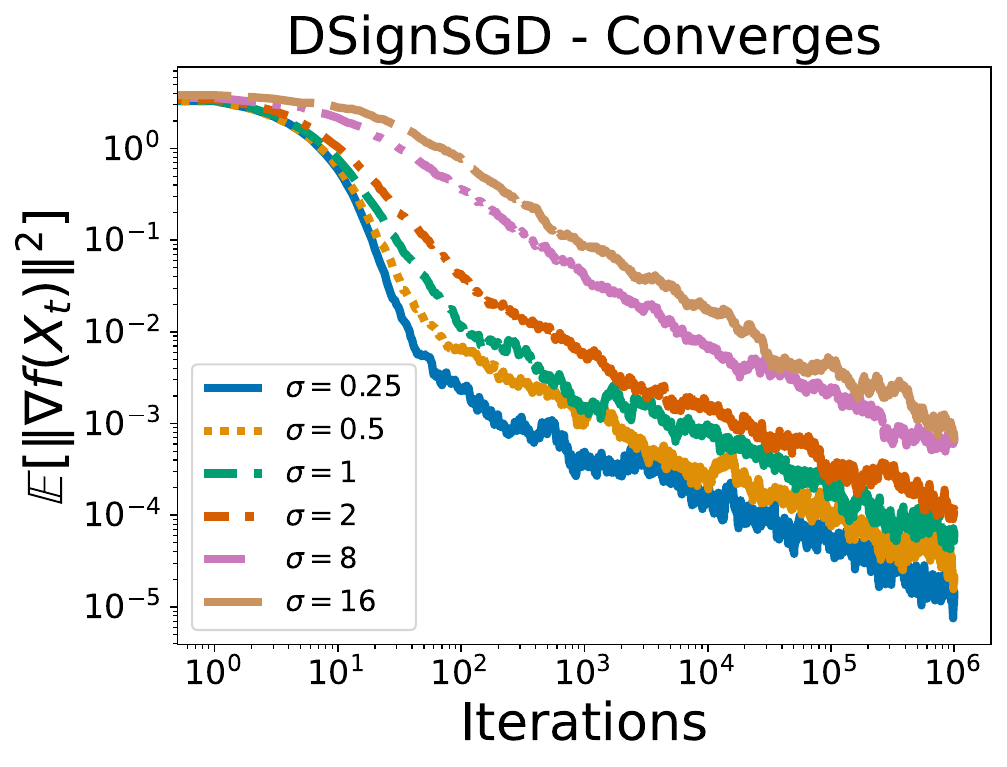} }
    \centering
    \caption{We optimize $f(x) = \frac{\sum_{j=1}^{1000}(x_j)^4}{4}$ with batch noise of variance $\sigma^2 \lVert \nabla f(x) \rVert_2^2$ and use \textit{Random Sparsification} for different compression rates $\omega$: as per Thm. \ref{thm:DCSGD_2}, DCSGD diverges faster and faster for larger values of $\omega$ when the normalization proposed in Eq. \ref{eq:StepDCSGD_UV} \textbf{is not employed} (Top-Left) but always converges if it \textbf{is employed} (Bottom-Left).  We optimize $f(x) = \frac{x^4}{4}$ with batch noise of \textbf{unbounded expected value} and for different \textit{scale parameters} $\sigma$: DSignSGD does not converge to $0$ \textit{without} a proper learning rate scheduler as prescribed by Thm. \ref{thm:sign} (Top-Right), but does converge \textit{with} (Bottom-Right). See Appendix \ref{sec:Exper} for all implementation details.}%
    \label{fig:InsightValidation}%
\end{figure}

\paragraph{Noise injection and compression.}
In all experiments, we inject additive Gaussian gradient noise and then apply random sparsification. At each iteration $t$, we sample $Z_t \sim \mathcal N(0, \sigma^2 \|g_t\|_2^2 I_d)$ with $\sigma = 0.1$ and form the noisy gradient $g_t + Z_t$. For a given sparsification probability $p \in \{\frac{4}{5}, \frac{8}{9}, \frac{16}{17}\}$ we draw an i.i.d. mask $m_t \in \{0,1\}^d$ with
\begin{equation}
\mathbb P\bigl[(m_t)_i = 1\bigr] = 1 - p,
\end{equation}
and define the unbiased random sparsifier
\begin{equation}
C_p(v) = \frac{v \odot m_t}{1-p},
\end{equation}
so that $\E[C_p(v)\mid v] = v$. In the plots, we simply label these three compression levels as $\omega \in \{4,8,16\}$, with larger $\omega$ corresponding to more aggressive sparsification (larger $p$).

\paragraph{DCSGD without scheduler}
Here, we use DCSGD with learning rate $\eta = 0.01$, noise level $\sigma = 0.1$, and sparsification probabilities $p \in \{\frac{4}{5}, \frac{8}{9}, \frac{16}{17}\}$. For each value of $p$, we run $T_{\mathrm{div}} = 10$ iterations and repeat the experiment over $n_{\mathrm{runs}}^{\mathrm{div}} = 100$ independent initializations of the MLP. We report the average of $\|g_t\|_2^2$ over these runs as a function of the iteration index.

\paragraph{DCSGD with our scheduler}
Here, we use DCSGD with learning rate $\eta = 0.01$ and use an adaptive scheduler $\eta_t$ as per Eq. \ref{eq:StepDCSGD_UV} where $\sigma_0=0$, and we assume $L_0=L_1=1$, because these constants are not actually known. As before, we use $\sigma = 0.1$ and $p \in \{\frac{4}{5}, \frac{8}{9}, \frac{16}{17}\}$. For each value of $p$, we run $T_{\mathrm{conv}} = 50000$ iterations and repeat the experiment over $n_{\mathrm{runs}}^{\mathrm{conv}} = 5$ independent initializations. We then average $\|g_t\|_2^2$ over these runs.

\paragraph{Normalized SGD with compression}
This experiment provides a baseline where we apply plain Normalized DCSGD, i.e., we normalize the compressed gradient. We use learning rate $\eta = 0.01$, noise level $\sigma = 0.1$, sparsification probabilities $p \in \{\frac{4}{5}, \frac{8}{9}, \frac{16}{17}\}$, and a small constant $\varepsilon = 10^{-8}$ added for numerical stability. The horizon and number of runs are the same as in the convergent DCSGD experiment, that is $T_{\mathrm{conv}} = 50000$ iterations and $n_{\mathrm{runs}}^{\mathrm{conv}} = 5$ independent initializations for each value of $p$. As before, we track and report the averaged trajectories of $\|g_t\|_2^2$.

\paragraph{DSignSGD}
Here, we apply DSignSGD as we inject Student's t noise with $\nu=1$ and scale parameters $\sigma$ on the gradient. The learning rate is $\eta = 0.01$, $\sigma \in \{0.25,0.5,1,2,8,16\}$, and we average over $5$ runs. In the top figure, we use no scheduler, while in the bottom one we use a scheduler as per Theorem \ref{thm:sign}, $\eta_t = \frac{1}{\sqrt{t+1}}$. As before, we track and report the averaged trajectories of $\|g_t\|_2^2$.

\subsection{Experiments on ResNet-18 and ViT (Figures \ref{fig:InsightValidation_ResNet18} and \ref{fig:InsightValidation_ViT})}
\label{sec:additional_cifar_experiments}

\paragraph{Architecture, data, and metric.}
We complement the MLP experiments with additional sanity checks on ResNet-18 and a simple ViT models trained on CIFAR-10 in a distributed setting with $N=8$ clients. CIFAR-10 is split uniformly across clients, and all curves are averaged over $5$ seeds. As in the previous section, these experiments are not intended as competitive benchmarks. Their goal is to test whether the qualitative mechanisms predicted by the theory persist beyond the synthetic and MLP settings. At iteration $t$, we monitor the squared norm of the gradient,
\begin{equation}
\|g_t\|_2^2,
\end{equation}
and report its averaged trajectory across independent runs.

\paragraph{Noise injection and compression.}
For the DCSGD experiments, we use the same noisy compressed-gradient protocol as in the MLP setting, with Gaussian noise level $\sigma=0.1$ and learning rate $\eta=0.1$. At each iteration $t$, we inject additive Gaussian gradient noise with affine variance,
\begin{equation}
    Z_t \sim \mathcal N(0,\sigma^2\|g_t\|_2^2 I_d),
\end{equation}
and then apply unbiased random sparsification at compression levels $\omega\in\{4,8,16\}$, where larger $\omega$ corresponds to more aggressive sparsification. We compare DCSGD without any scheduler/normalization against DCSGD with the adaptive normalization suggested by Thm.~\ref{thm:DCSGD_2} through Eq.~\ref{eq:StepDCSGD_UV}. We also report a baseline that applies plain Normalized SGD under the same noise and compression.

\paragraph{DCSGD.}
Without any scheduler or normalization, DCSGD becomes unstable on both ResNet-18 and ViT, and the divergence worsens as the compression level $\omega$ increases. In contrast, using the adaptive normalization prescribed by Eq.~\ref{eq:StepDCSGD_UV} stabilizes training and yields convergence for all tested values of $\omega$. The plain Normalized SGD baseline is also stabilized by normalization, but exhibits a less stable profile in these experiments.

\begin{figure}[ht!]
   \hspace{.2cm}
    {\includegraphics[width=0.48\linewidth]{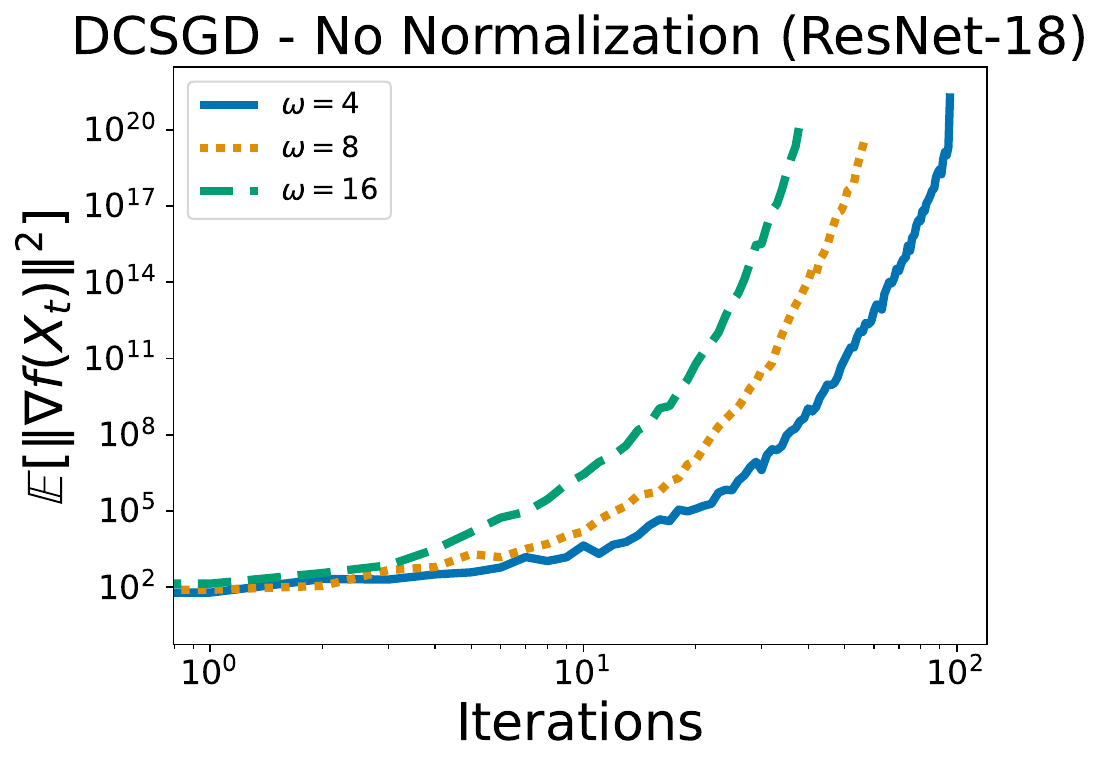} }%
    {\includegraphics[width=0.48\linewidth]{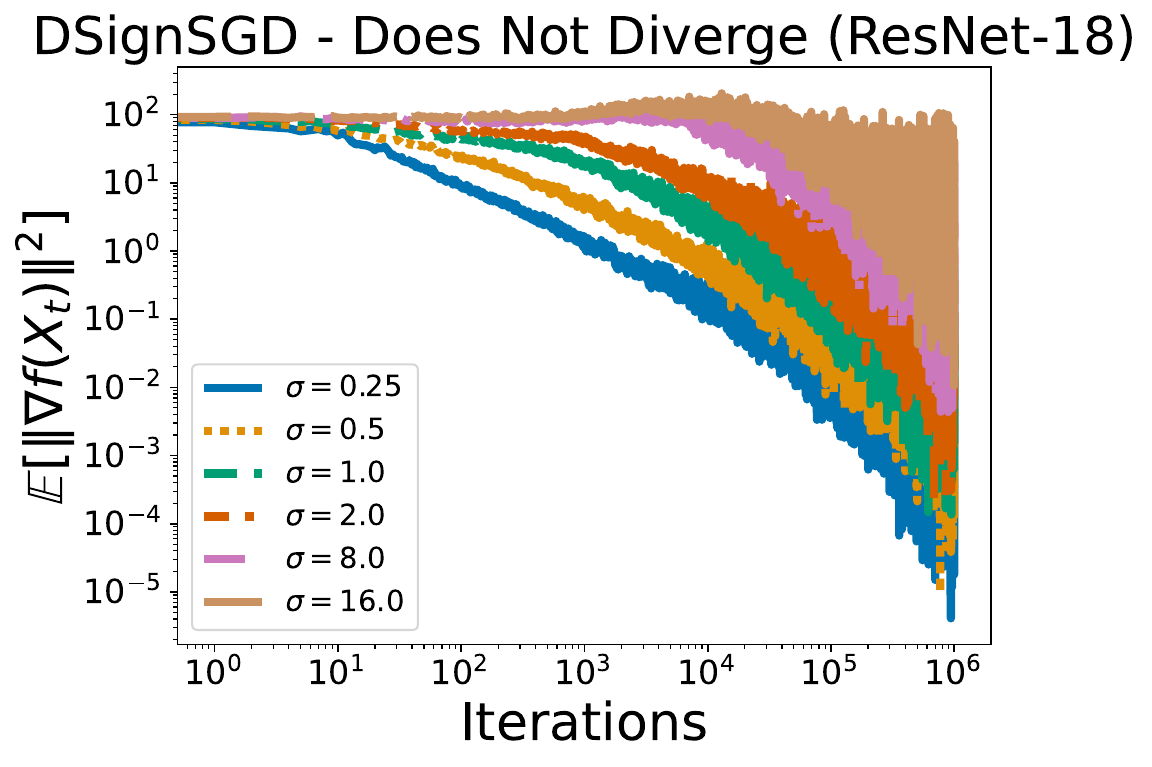} } \\
    {\includegraphics[width=0.48\linewidth]{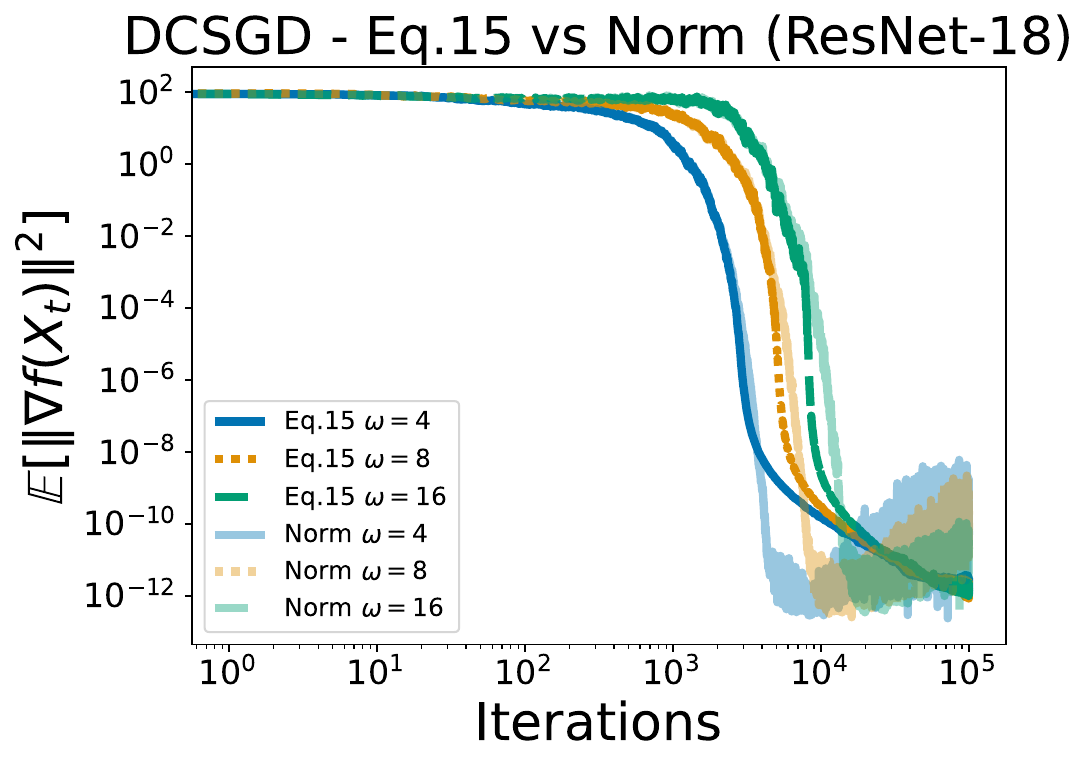} }%
    {\includegraphics[width=0.48\linewidth]{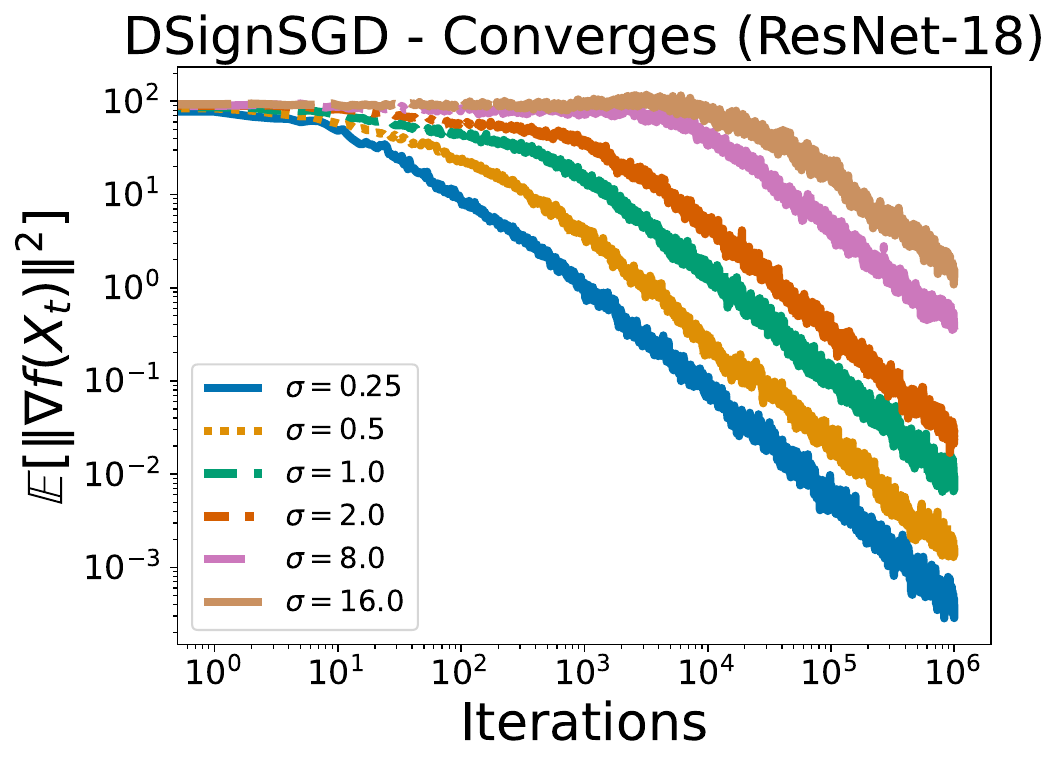} }
    \centering
    \caption{Sanity checks for our stability prescriptions on noisy, compressed ResNet-18 training on CIFAR-10 with $N=8$ clients. For DCSGD with unbiased sparsification, we inject additive Gaussian gradient noise with affine variance $Z_t\sim\mathcal{N}(0,\sigma^2\|g_t\|_2^2 I)$ and then apply random sparsification at compression levels $\omega$. Without any scheduler or normalization, DCSGD becomes unstable and the divergence worsens as $\omega$ increases (Top-Left). Using the adaptive normalization suggested by Thm.~\ref{thm:DCSGD_2} through Eq.~\ref{eq:StepDCSGD_UV} stabilizes training and yields convergence for all tested $\omega$ (Bottom-Left). We also report a baseline that applies plain Normalized SGD under the same noise and compression, which exhibits a less stable profile. For DSignSGD under heavy-tailed noise, we inject Student's t gradient noise with $\nu=1$ and different scale values $\sigma$. With a constant stepsize, DSignSGD remains stable but does not converge to zero (Top-Right), whereas the diminishing schedule prescribed by Thm.~\ref{thm:sign}, here $\eta_t=1/\sqrt{t+1}$, yields convergence across noise scales (Bottom-Right).}
    \label{fig:InsightValidation_ResNet18}
\end{figure}

\paragraph{DSignSGD.}
For the DSignSGD experiments, we use learning rate $\eta=0.001$ and inject Student's t gradient noise with $\nu=1$ and scale parameters $\sigma\in\{0.25,0.5,1,2,8,16\}$. With a constant stepsize, DSignSGD remains stable but does not converge to zero. With the diminishing schedule prescribed by Thm.~\ref{thm:sign}, namely $\eta_t = 1/\sqrt{t+1}$, DSignSGD converges across the tested noise scales. These results are consistent with the prediction that DSignSGD inherits an implicit normalization from the sign map, whereas DCSGD requires an explicit noise- and compression-dependent normalization to remain stable.

\begin{figure}[ht!]
   \hspace{.2cm}
    {\includegraphics[width=0.48\linewidth]{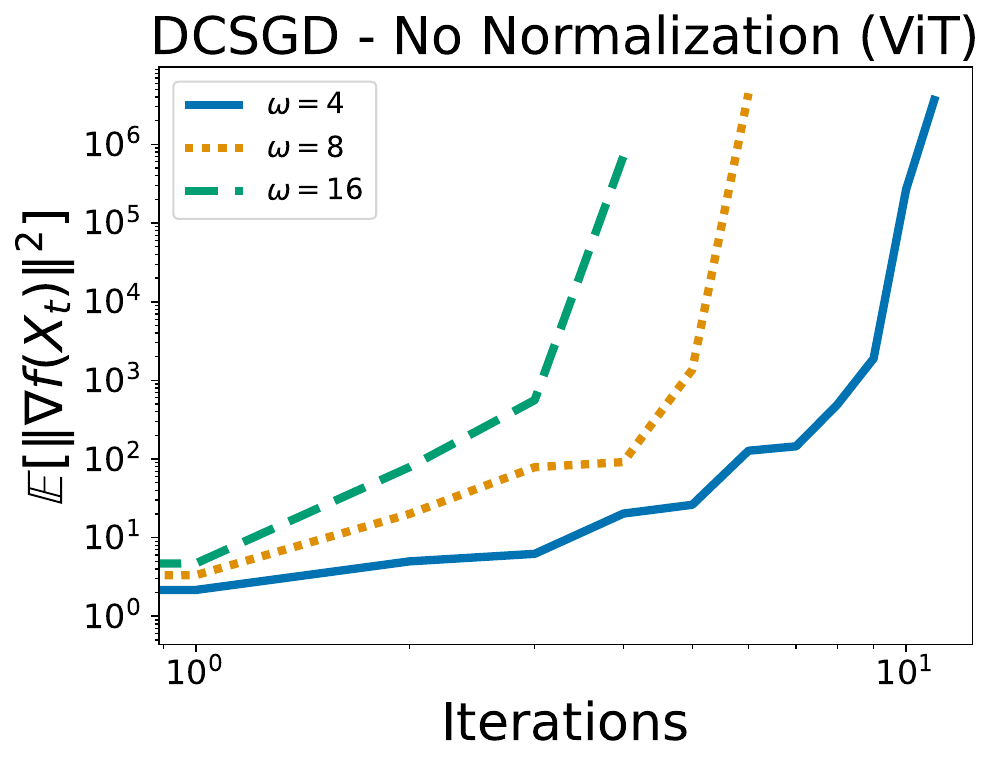} }%
    {\includegraphics[width=0.48\linewidth]{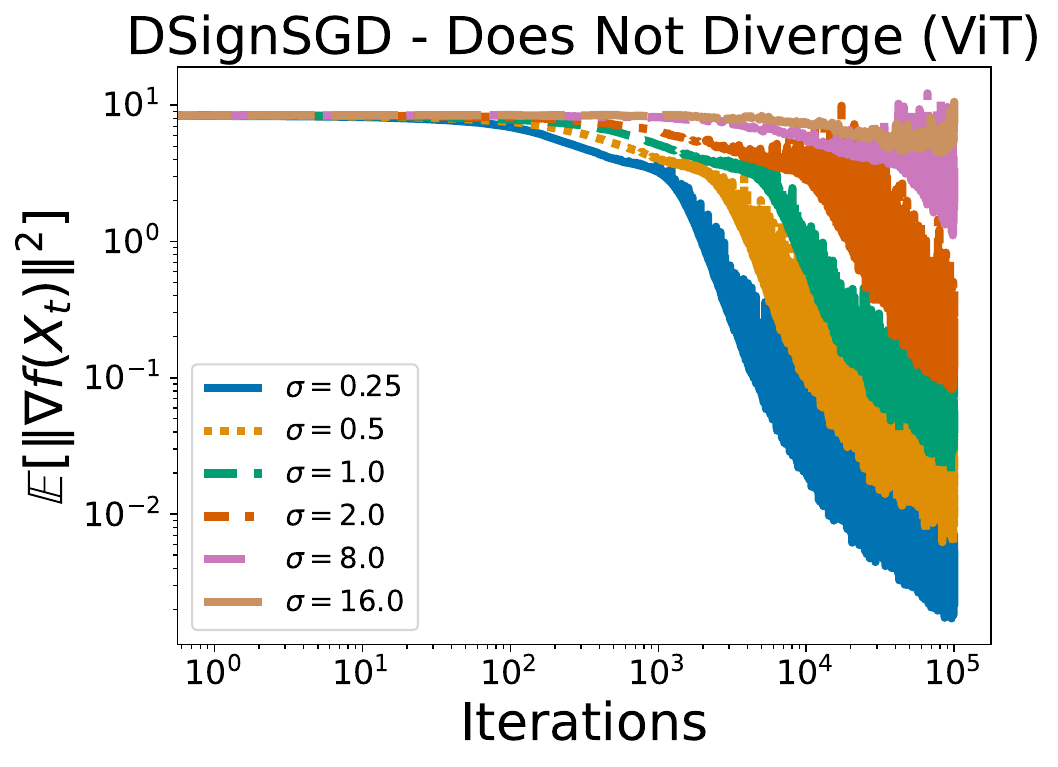} } \\
    {\includegraphics[width=0.48\linewidth]{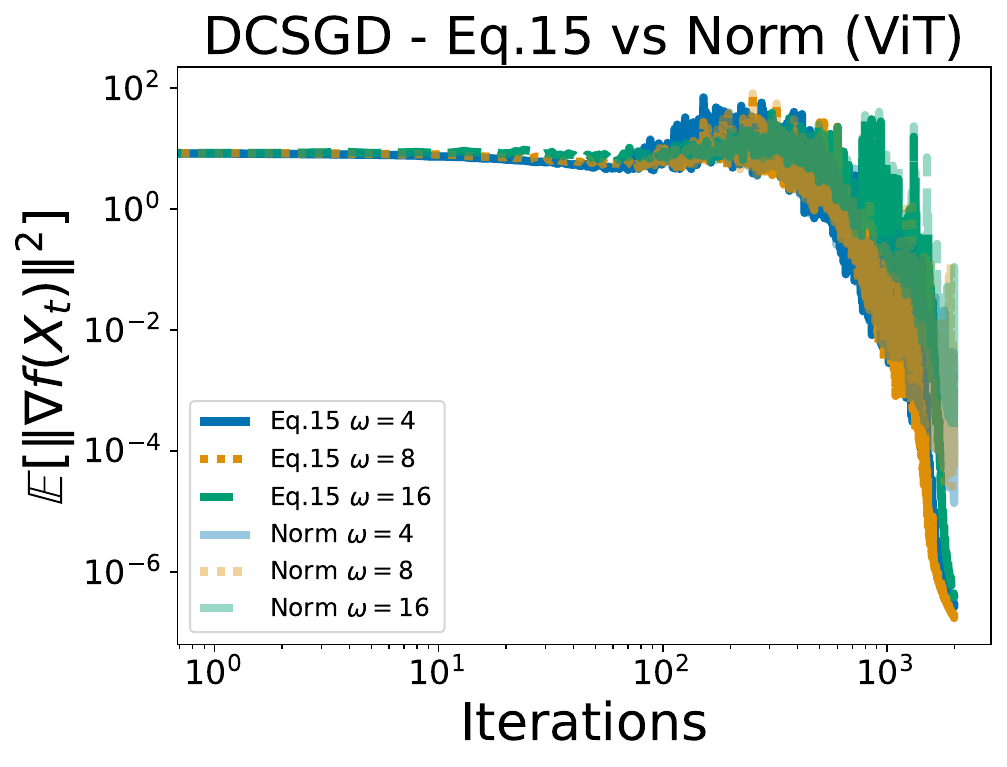} }%
    {\includegraphics[width=0.48\linewidth]{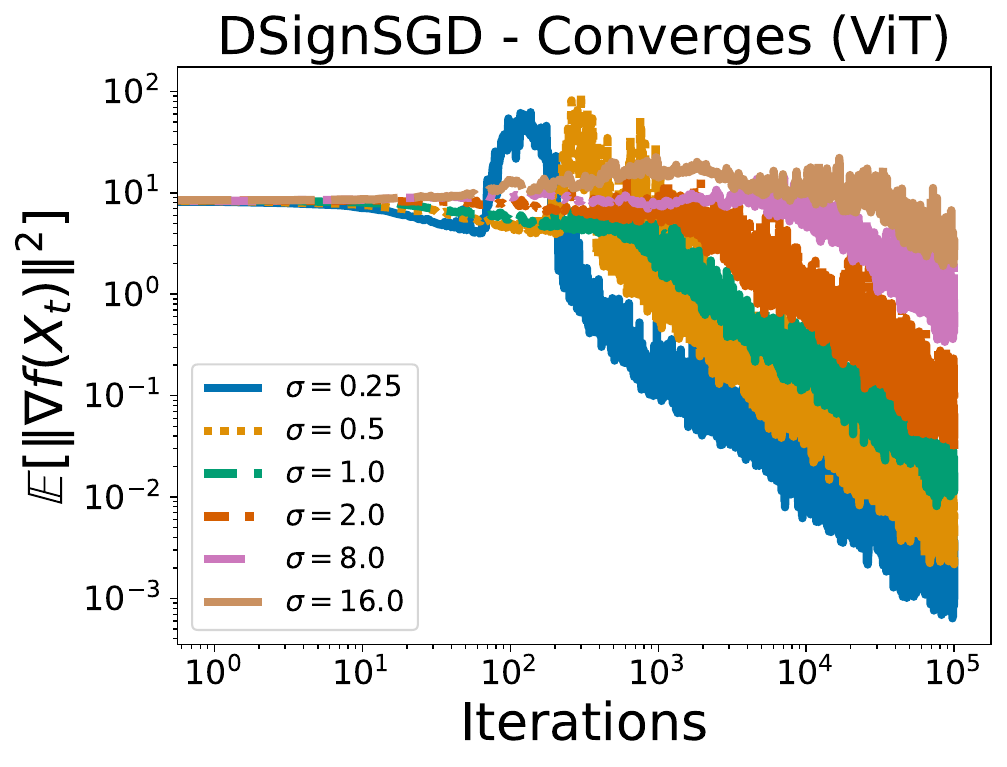} }
    \centering
    \caption{Sanity checks for our stability prescriptions on noisy, compressed ViT training on CIFAR-10 with $N=8$ clients. For DCSGD with unbiased sparsification, we inject additive Gaussian gradient noise with affine variance $Z_t\sim\mathcal{N}(0,\sigma^2\|g_t\|_2^2 I)$ and then apply random sparsification at compression levels $\omega$. Without any scheduler or normalization, DCSGD becomes unstable and the divergence worsens as $\omega$ increases (Top-Left). Using the adaptive normalization suggested by Thm.~\ref{thm:DCSGD_2} through Eq.~\ref{eq:StepDCSGD_UV} stabilizes training and yields convergence for all tested $\omega$ (Bottom-Left). We also report a baseline that applies plain Normalized SGD under the same noise and compression, which exhibits a less stable profile. For DSignSGD under heavy-tailed noise, we inject Student's t gradient noise with $\nu=1$ and different scale values $\sigma$. With a constant stepsize, DSignSGD remains stable but does not converge to zero (Top-Right), whereas the diminishing schedule prescribed by Thm.~\ref{thm:sign}, here $\eta_t=1/\sqrt{t+1}$, yields convergence across noise scales (Bottom-Right).}
    \label{fig:InsightValidation_ViT}
\end{figure}

\subsection{Constructive Form of the Normalization Condition}
\label{subsec:constructive_adaptivity}

The sufficient conditions for convergence of DCSGD (see Eq.~\ref{eq:StepDCSGD_UV}) indicate that the learning rate schedule $\eta \eta_t$ should scale inversely with $\E \lVert \nabla f(X_t)\rVert$. While this may appear abstract, it admits a natural and practical implementation in the distributed setting.

\paragraph{Client-side estimation.} 
At iteration $t$, each client $i$ already computes a stochastic gradient $\nabla f_{i,\gamma_i}(X_t)$ on a local mini-batch $\gamma_i$. We define the local norm estimate as
\begin{equation}
    \hat{g}_i^t = \lVert\nabla f_{i,\gamma_i}(X_t)\rVert.
\end{equation}
This requires no additional computation beyond what is standard for mini-batch gradient methods.

\paragraph{Server-side aggregation.} 
The server maintains an estimate of the global gradient norm by averaging the client-side estimates as
\begin{equation}
    \hat{G}_t = \frac{1}{N}\sum_{i=1}^N \hat{g}_i^t,
\end{equation}
which provides a consistent approximation of $\E \lVert \nabla f(X_t)\rVert$.

This construction is intended for the standard server-aggregated distributed setting studied in this paper, where a central server already averages the client updates. It is not meant as a fully decentralized protocol: in peer-to-peer or gossip-style topologies, estimating the same global norm would require additional communication or a different decentralized estimator.

\paragraph{Normalized learning rate.} 
A learning rate of the form
\begin{equation}
\eta \eta_t \sim \frac{\eta_0}{1+\hat{G}_t}
\end{equation}
satisfies the normalization condition in our bounds up to stochastic error. This adjustment can be implemented with negligible communication overhead, requiring each client to transmit only a single scalar per iteration.

\end{document}